\documentclass{article}

\PassOptionsToPackage{numbers, compress}{natbib}
\usepackage[main, final]{neurips_2026}

\usepackage[utf8]{inputenc} % allow utf-8 input
\usepackage[T1]{fontenc}    % use 8-bit T1 fonts
\usepackage{hyperref}       % hyperlinks
\usepackage{url}            % simple URL typesetting
\usepackage{booktabs}       % professional-quality tables
\usepackage{amsfonts}       % blackboard math symbols
\usepackage{nicefrac}       % compact symbols for 1/2, etc.
\usepackage{microtype}      % microtypography
\usepackage{xcolor}         % colors

\usepackage{amsmath}
\usepackage{amssymb}
\usepackage{mathtools}
\usepackage{amsthm}

\usepackage{algorithm}
\usepackage{algorithmic}
\usepackage{subcaption}
\usepackage{tabularx}

\theoremstyle{plain}
\newtheorem{theorem}{Theorem}[section]

\newtheorem{lemma}[theorem]{Lemma}
\newtheorem{corollary}[theorem]{Corollary}
\theoremstyle{definition}

\newtheorem{assumption}[theorem]{Assumption}
\theoremstyle{remark}

\usepackage{multirow}
\usepackage[table]{xcolor}
\usepackage{bm}
\usepackage{pifont}
\usepackage{adjustbox}

\usepackage{placeins}

\usepackage[textwidth=0.7in, textsize=small]{todonotes}
\setuptodonotes{inline}

\definecolor{lightgray}{gray}{0.92}
\definecolor{antiquewhite}{rgb}{0.98, 0.92, 0.84} 
\definecolor{blizzardblue}{rgb}{0.67, 0.9, 0.93}

\DeclareMathOperator*{\argmin}{arg\,min}

\newcommand{\roracle}{\text{\sffamily{RO}}}
\newcommand{\aoracle}{\text{\sffamily{AO}}}
\newcommand{\rsolver}{\text{\sffamily{RS}}}

\usepackage{enumitem}

\title{Correlation-Aware Contextual Bandits with\\ Surrogate Rewards for LLM Routing}

\author{%
  Ajay Narayanan Sridhar\\
   Georgia Institute of Technology\\\texttt{ajay.narayanan@gatech.edu} \\
  \And
  Ronak Singh \\
  The Pennsylvania State University \\
  \texttt{rjs7006@psu.edu} \\
  \And
  Mehrdad Mahdavi \\
  The Pennsylvania State University \\
  \texttt{mzm616@psu.edu} \\
  \And
  Vijaykrishnan Narayanan \\
  Georgia Institute of Technology \\
  \texttt{vnarayanan67@gatech.edu} \\
}

\begin{document}

\maketitle

\begin{abstract}
    We study contextual bandit problems with correlated arms and access to surrogate reward signals produced by a machine learning model, motivated by applications such as large language model (LLM) routing. Unlike classical contextual bandits that rely solely on bandit feedback and assume conditional independence across arms, our setting allows context-dependent inter-arm correlations and auxiliary reward information that may be noisy or misspecified. We propose algorithms that leverage such surrogate rewards through two complementary designs. A coupled reward-mixing approach pools true and surrogate rewards to accelerate learning when surrogate signals are reliable, while a decoupled prediction-mixing approach maintains separate estimators for bandit feedback and surrogate rewards and adaptively combines their predictions. This decoupling yields robustness to surrogate misspecification, recovering regret guarantees comparable to reward-only bandit methods in the worst case, while achieving improved regret when surrogate predictions are sufficiently informative. We provide theoretical regret analyses for both approaches and evaluate them on LLM routing benchmarks under varying accuracy versus cost trade-offs. The results demonstrate improved sample efficiency and consistently better accuracy-cost trade-offs compared to standard contextual bandit baselines and strong static routing methods.
\end{abstract}

%%%%%%%%%%%%%%%%%%%%%%%%%%%%%%%%%%%%%%%%%%%%%%%%%%%%%%%%%%%%%
\section{Introduction}

The rapid proliferation of Large Language Models (LLMs) has made it increasingly challenging for end-users to keep track of advancements and optimally select models for their specific needs. Presently, numerous proprietary LLM providers exist~\cite{chatgpt2025, gemini15pro, grok1_mla}, and thousands more are openly available through repositories such as Hugging Face~\cite{huggingface_models}. This abundance creates a practical systems problem: for a given query, which model should be called? The answer is rarely the most accurate model, since inference incurs nontrivial cost (e.g., pricing, latency, compute). In many applications, the objective is an \emph{accuracy--cost trade-off}: a slightly less accurate model may be preferable if it is significantly cheaper or faster, and the optimal choice varies with context.

Existing unified interfaces address this challenge through model cascading~\cite{chen2023frugalgptuselargelanguage, huang2025thriftllm}, ensembling~\cite{jiang2023llm}, and routing~\cite{nguyen2024metallm, ong2024routellm, zhao2024eagle, stripelis2024tensoropera, li2025llmbanditcostefficientllm}. We focus on routing: selecting one model per query to maximize user-defined utility. A common reduction models LLM routing as a \emph{contextual multi-armed bandit}, where each LLM is treated as a conditionally independent arm given the incoming query context, and only the selected arm's reward is observed, requiring costly exploration. In the $d$-dimensional linear setting, where the context is represented by a $d$-dimensional embedding and each arm's expected reward is assumed to be a linear function of this embedding, standard methods achieve regret $R_S=~O(\sqrt{dKT\log(T/d)})$~\cite{foster2020beyond} given $K$ arms (LLMs) for $T$ stream of queries, where regret is the cumulative loss relative to the best context-dependent action. However, this conditional-independence abstraction is often too coarse: models exhibit query-dependent correlations due to shared pretraining data, alignment pipelines, and architectures. Exploiting such correlations can reduce uncertainty about unplayed arms and lower the effective exploration burden.

LLM routing also provides a source of side information largely absent from standard bandit settings: offline performance data. Benchmarks or historical logs can train machine learning (ML) predictors that map query contexts to per-arm reward estimates. During online routing, these estimates can serve as \emph{surrogate rewards} for unplayed arms. Although surrogate rewards may be biased or noisy, they can still reveal relative arm quality. The key challenge is to exploit surrogate rewards when helpful without over-trusting them when misspecified. This motivates our central research question:\vspace{-1mm}
\begin{quote}
\emph{Can we design contextual bandit algorithms that systematically exploit both context-dependent inter-model correlations and auxiliary ML-predicted surrogate rewards to accelerate learning, while remaining robust to surrogate misspecification?}
\end{quote}\vspace{-1mm}

To address this question, we study contextual bandits with \emph{correlation-aware} graph feedback and auxiliary \emph{surrogate rewards}. Each round, the learner observes the query context and context-dependent side information about inter-arm relationships, which are used to form a feedback graph that identifies a small set of additional arms whose feedback is most informative given the chosen arm. After selecting an arm and observing its realized reward, the learner additionally receives ML-predicted surrogate rewards for the arms in this graph-selected set, providing partial multi-arm side information beyond the chosen arm. We first propose a coupled \emph{reward-mixing} approach, \textit{Correlation-Aware Bandits with Surrogates Coupled} (CABS-C), which builds on graph-feedback variants of SquareCB~\cite{zhang2023practical,zhang2024efficient} by pooling de-biased surrogate rewards with true rewards to fit a single contextual model. This coupling effectively increases the number of informative observations per round, improving the exploration factor from $K$ to $K/(m+1)$, where $m$ is the number of additional arms for which surrogate rewards are revealed each round. The price of coupling these observations is an added sensitivity to surrogate error. For linear bandits, we show the regret scales as $R_C=\widetilde O\!\left(\sqrt{dKT\log(T/d)/(m+1)}+\varepsilon_n \sqrt{TK^2/(m+1)}\right)$, where $\varepsilon_n$ quantifies the worst-case magnitude of surrogate noise/misspecification (see Assumption~\ref{asm:ml-side-obs-model}). The preceding regret bound makes the trade-off explicit: accurate surrogates reduce the effective exploration burden, while large surrogate error can dominate and make tight coupling brittle.

Motivated by this robustness issue, we propose a decoupled \emph{prediction-mixing} strategy, \textit{Correlation-Aware Bandits with Surrogates Decoupled} (CABS-D), that treats (i) a reward-only contextual bandit; and (ii) a correlation-aware (surrogates induced graph feedback) bandit (CABS-C) as two experts and combines them using an adaptive expert-aggregation master, in the spirit of adaptive Hedge methods~\cite{agarwal2017corralling,erven2011adaptive}. The resulting meta-bandit promises a best-of-both-worlds guarantee,
$R_{D} = O\!\left(\min \left\{R_{S}, R_{C}\right\}\right)$,
allowing us to match the standard contextual bandit rate (plus mild overhead) in the worst case while inheriting correlation-driven gains when surrogate feedback is reliable. We evaluate our LLM routing approaches, observing improved sample efficiency and better accuracy–cost trade-offs compared to online bandit and static baselines.

\noindent\textbf{Contributions.}~In summary, our contributions are:
\begin{itemize}\vspace{-1mm}
    \item We introduce \emph{correlation-aware contextual bandits with surrogate rewards}, where context-dependent inter-arm relationships induce a graph-feedback structure and an auxiliary predictor provides surrogate rewards for graph-specified arms beyond the observed true reward.\vspace{-0.25mm}
    \item We propose CABS-C (coupled \emph{reward mixing}) and CABS-D (decoupled \emph{prediction mixing}), clarifying when tight coupling is beneficial and when decoupling is necessary for robustness.\vspace{-0.25mm}
    \item We establish regret bounds that explicitly separate correlation-driven gains (reduced effective exploration) from surrogate noise/misspecification, including a best-of-both-worlds guarantee for CABS-D that matches standard contextual bandits in the worst case, while simultaneously achieving an improved regret that scales inversely with the number of surrogate rewards under mild conditions.\vspace{-0.25mm}
    \item We evaluate our methods on LLM routing benchmarks and demonstrate improved sample efficiency and stronger accuracy--cost trade-offs relative to online bandit baselines and static routing policies.
\end{itemize}

%%%%%%%%%%%%%%%%%%%%%%%%%%%%%%%%%%%%%%%%%%%%%%%%%%%%%%%%%%%%%
\section{Related Work}
\noindent\textbf{Contextual Bandits, Graph Feedback, and Expert Aggregation.}~Our work builds on contextual bandits such as SquareCB~\cite{foster2020beyond}, LinearUCB~\cite{chu2011contextual}, and NeuralUCB~\cite{zhou2020neural}, and is closest to contextual bandits with graph feedback~\cite{zhang2024efficient,zhang2023practical}. Unlike SquareCB-G~\cite{zhang2023practical}, which observes \emph{true} rewards for graph neighbors, and SquareCB-UG~\cite{zhang2024efficient}, which learns an unknown graph after action selection, we observe true reward only for the selected arm and receive \emph{surrogate} rewards from a context-dependent graph known before action selection. CABS-D is related to Hedge/AdaHedge-style expert aggregation~\cite{freund1997decision,erven2011adaptive} and Corral-style bandit masters~\cite{agarwal2017corralling}, but adapts specifically between standard bandit feedback and surrogate graph feedback. Auxiliary-feedback methods~\cite{cheung2024leveraging,verma2023exploiting,ji2025multi} use external signals mainly to improve estimates for the \emph{played} arm, whereas we propagate surrogate feedback across multiple arms. See~\cite{lattimore2020bandit} and Appendix~\ref{app:sec:related} for a broader overview and detailed comparisons.

\noindent\textbf{LLM routing.}~LLM routing selects a model for each query under quality, cost, or latency constraints. Prior work includes offline routers trained from preference or performance data~\cite{ong2024routellm,zhao2024eagle,stripelis2024tensoropera, chen2024routerdc,feng2024graphrouter,somerstep2025carrot}, benchmarks for cost--performance routing~\cite{hu2024routerbenchbenchmarkmultillmrouting, somerstep2025carrot}, and online routing methods based on contextual bandits or dueling bandits~\cite{nguyen2024metallm,li2025llmbanditcostefficientllm, wang2025mixllm,wei2025learning,dai2024cost,chiang2025llm}. Existing online routing strategies largely do not \emph{explicitly} model intrinsic correlations among LLMs (beyond implicit representation sharing), which can limit sample-efficiency when model behaviors are strongly related. Addressing this gap, our correlation-aware routing framework exploits context-dependent inter-arm relationships and integrates auxiliary surrogate signals while remaining robust to surrogate misspecification, yielding improved routing utility across cost regimes.

%%%%%%%%%%%%%%%%%%%%%%%%%%%%%%%%%%%%%%%%%%%%%%%%%%%%%%%%%%%%%
\section{Problem Setting}
%%%%%%%%%%%%%%%%%%%%%%%%%%%%%%%%%%

We consider an online LLM routing problem setting operating over a finite horizon of $T$ discrete rounds, indexed by $t \in[T]:=\{1,2, \ldots, T\}$. The router (learner) has access to a fixed pool of $K$ distinct LLMs (arms), denoted by the set $\mathcal{K}=[K]:=\{1,2, \ldots, K\}$. At each round $t\in[T]$, the learner observes a prompt and its context embedding vector  $\boldsymbol{x}_t\in\mathcal{X}\subseteq\mathbb{R}^d$, where $\mathcal{X}$ is a bounded context space, and must select an LLM $i_t\in[K]$ to process the prompt. The selected model, $i_t$, produces an outcome $\boldsymbol{o}_t$ and incurs a cost-aware reward $r_{t,i_t}(\boldsymbol{x}_t)$; when convenient, we write $r_{t,i_t}$ and define the corresponding loss $\ell_{t,i_t}=1-r_{t,i_t}$. The reward typically captures  generation quality (e.g., accuracy, helpfulness) and trades off it against operational constraints (e.g., latency, token cost). We explicitly bound the reward and loss to the $[0,1]$ interval, which is standard  for regret analysis. 

We provide a schematic overview of the routing and feedback pipeline in Appendix~\ref{appendix:implementation} (Figure~\ref{fig:task_overview}).

To exploit query-dependent correlations among LLMs and auxiliary reward predictions from offline data, we formulate LLM routing as a correlation-aware contextual bandit with surrogate rewards. At each round $t$, in addition to $\boldsymbol{x}_t$, the learner observes inter-arm side information from an affinity oracle (\aoracle), represented either as a correlation matrix $\boldsymbol{R}_t\in\mathbb{R}^{K\times K}$ or as a time-varying, context-dependent  directed feedback graph $G_t$ over $[K]$ with adjacency matrix $\boldsymbol{A}_t\in\{0,1\}^{K\times K}$, where $A_{t,(i,j)}=1$ indicates an edge from arm $i$ to arm $j$. By default, $A_{t,(i,i)} = 1$ for all $i \in [K]$. The side information among LLMs quantifies how performance scales across different models for a specific context. Upon selecting  arm $i_t$, the learner observes the true realized reward only for the chosen arm $i_t$. However, for other arms $j \neq i_t$ specified by the feedback graph, the learner receives surrogate rewards $s_{t,j}$ from a reward oracle (\roracle). The surrogate rewards predicted by the {\sffamily{RO}} act as a proxy estimate derived from offline data or cross-model evaluation functions. The conditional bias or variance of the surrogate reward $s_{t,j}$ relative to the true underlying reward $r_{t,j}$ is assumed to be bounded. 

Similar to~\cite{foster2020beyond,zhang2023practical}, we assume that the learner  has access to a regression solver (\rsolver) that is invoked at each round $t$ to update arm models from historical input--output data. The  $\mathsf{RS}$ serves as an abstract optimization primitive that handles the algorithmic complexity of empirical risk minimization and can efficiently compute a predictor from a predefined function class  $\mathcal{F}$ (such as linear models or deep neural networks).  Formally,  given a sequence of context-dependent importance weights $w_{\tau, j} \ge 0$, the $\mathsf{RS}$  processes the accumulated feedback history up to the previous round to solve a weighted empirical risk minimization problem for each arm $i \in [K]$ to update its model:
$$\hat{f}_{t,i} = \arg\min_{f \in \mathcal{F}} \sum_{\tau=1}^{t-1}  w_{\tau, i}  \left( f(\boldsymbol{x}_\tau, i) - \hat{r}_{\tau,i} \right)^2  + \mathcal{R}(f)$$

where $\hat{r}_{\tau,i}$ represents the importance-weighted estimate at round $\tau$ computed based on true or surrogate rewards (equal to $\hat r_{\tau,i_\tau}=r_{\tau,i_\tau}/p_{\tau,i_\tau}$ if $i = i_\tau$, and $\hat r_{\tau,j}=(s_{\tau,j} - \hat{b}_{\tau,j})/(\boldsymbol{A}_{\tau,(:,j)}^\top\boldsymbol{p}_t)$ if $i$ is revealed by the feedback graph $G_\tau$), with $r_{\tau,i_\tau}$ being the true observed reward, $s_{\tau, i}$ being the proxy surrogate reward with estimated bias $\hat{b}_{\tau,j}$, and $p_{\tau,i_\tau}$ being the probability of selecting arm $i$ at round $\tau$. Here, $\mathcal{R}(f)$ is an optional regularization term (e.g., $\ell_2$ penalty) used to prevent overfitting in high-dimensional context spaces.

The learner's goal is to minimize the  regret over $T$ rounds. The optimal arm at round $t$ is defined as the arm that maximizes the true  reward: $i_t^\star = \arg\max_{i \in [K]} \mathbb{E}\left[r_{t,i}(\boldsymbol{x}_t)\right]$. The cumulative regret $\mathrm{Reg}(T)$ is given by $\mathrm{Reg}(T) = \sum_{t=1}^T \mathbb{E}\left[ r_{t,i_t^\star}(\boldsymbol{x}_t) \right]- \mathbb{E}\left[r_{t,i_t}(\boldsymbol{x}_t) \right]$, which we use to measure the performance of an online routing algorithm.

%%%%%%%%%%%%%%%%%%%%%%%%%%%%%%%%%%
\section{The Proposed Algorithms}

In this section, we propose two algorithms for contextual bandits with surrogate rewards and correlated arms. We first introduce a reward-mixing method, CABS-C, that learns from both true rewards and de-biased surrogate feedback. We then introduce a prediction-mixing method, CABS-D, that dynamically integrates predictions from CABS-C with those from a standard bandit-feedback method (SquareCB~\cite{foster2020beyond} or LinUCB~\cite{li2010contextual}) to achieve best-of-both-worlds regret guarantees. 

\begin{algorithm*}[t!]
\caption{Correlation-Aware Bandits with Surrogates Coupled (CABS-C)}
\label{alg:CABS-C}
\begin{algorithmic}[1]
    \STATE \textbf{Input:} Number of rounds $T$, regularization parameter $\gamma$.
    \STATE \textbf{Oracles:} Reward oracle $\roracle: \mathcal{X} \times [K] \times \mathbb{R}_{+} \rightarrow \mathbb{R}_+^{K}$, affinity oracle $\aoracle: \mathcal{X} \rightarrow \mathbb{R}^{K \times K}$, regression solver $\rsolver: \mathcal{X} \rightarrow \mathbb{R}^{K \times d}$ 

    \STATE \textbf{Initialize:} $n_{0,i} = 0$ and $\hat b_{0,j}=0$ for all $i \in [K]$
    \FOR{$t = 1$ to $T$}
        \STATE Receive context  $\boldsymbol{x}_t$ and corresponding correlation matrix $\boldsymbol{R}_t$ from $\aoracle(\boldsymbol{x}_t)$. \label{line:cabsc-predict-start}
        \STATE Generate affinity graph $G_t$ with adjacency matrix $\boldsymbol{A}_t$: $A_{t,(i,j)}=1$ w.p. $\sigma((1-R_{t,(i,j)})/2)$

        \STATE Obtain an estimator $\hat{\boldsymbol{\theta}}_{t,i}: \mathbb{R}^d$ for all $i \in [K]$ from $\rsolver(\boldsymbol{x}_t)$

        \STATE Set $\boldsymbol{p}_t = \argmin_{\boldsymbol{p} \in \Delta([K])} \phi(\boldsymbol{p}; \hat{\boldsymbol{\theta}}_{t}, \boldsymbol{x}_t, \boldsymbol{A}_t)$, where $\phi(\boldsymbol{p}; \hat{\boldsymbol{\theta}}_{t}, \boldsymbol{x}_t, \boldsymbol{A}_t)$ is defined as follows,

        \begin{align*}
             &:= \sup_{\substack{ i^* \in [K] \\ \boldsymbol{\theta}^* \in \mathbb{R}^{K \times d}}} \hspace{0.5em} \mathbb{E}_{i \sim \boldsymbol{p}} \left[ \boldsymbol{x}_t^{\top} \boldsymbol{\theta}^*_{i^*} - \boldsymbol{x}_t^{\top} \boldsymbol{\theta}^*_{i} - \frac{\gamma}{4} \sum_{k: A_{t,(i,k)} = 1} \mathbb{E}[A_{t,(i,k)}]( \boldsymbol{x}_t^{\top}\boldsymbol{\theta}^*_{k} - \boldsymbol{x}_t^{\top} \hat{\boldsymbol{\theta}}_{t,k})^2 \right]
        \end{align*}  \label{line:cabsc-predict-end}

        \STATE Sample action $i_t\sim \boldsymbol{p}_t$ and receive true reward $r_{t,i_t}$.

        \STATE Observe surrogate rewards $s_{t,j}$ from $\roracle(\boldsymbol{x}_t,i_t,r_{t,i_t})$ for all $j \in \mathcal{N}_{t}(i_t)$, where $\mathcal{N}_{t}(i_t)$ is all arms in the neighborhood of arm $i_t$ (including $i_t$) according to $G_t$ 

        \STATE Construct de-biased surrogate rewards $\tilde s_{t,j} = s_{t,j} - \hat b_{t,j}$ for all $j \in \mathcal{N}_{t}(i_t)$ \label{line:cabsc-update-start}

        \STATE Update the bias estimator for all $j \in \mathcal{N}_{t}(i_t)$ by setting
        $n_{t+1,j} \leftarrow n_{t,j}+1$ and
        \[
        \hat b_{t+1,j} \leftarrow \left(1-\frac{1}{n_{t+1,j}}\right)\hat b_{t,j}
        + \frac{1}{n_{t+1,j}}\bigl(s_{t,j}-\boldsymbol{x}_t^\top \hat{\boldsymbol{\theta}}_{t,j}\bigr).
        \]
        \STATE Keep parameters the same for unobserved arms: $n_{t+1, k} = n_{t, k}, \hat{b}_{t+1, k} = \hat{b}_{t, k}$ for $k \notin \mathcal{N}_t(i_t)$.
        
        \STATE Set importance-weighted estimate for chosen arm: 
        $\hat r_{t,i_t}=r_{t,i_t}/p_{t,i_t}$.
        \STATE Set importance-weighted estimates for surrogates: 
        $\hat r_{t,j}=\tilde s_{t,j}/(\boldsymbol{A}_{t,(:,j)}^\top\boldsymbol{p}_t)$  $\forall j\in\mathcal N_t(i_t) \setminus \{i_t\}.$

        \STATE Feed the tuples $\{(\boldsymbol{x}_t, \hat r_{t,j})\}_{j \in \mathcal{N}_t(i_t)}$ to the oracle $\rsolver$  \label{line:cabsc-update-end}
        
    \ENDFOR
\end{algorithmic}
\end{algorithm*}

%%%%%%%%%%%%%%%%%%%%%%%%%%%%%%%%%%%%%%%%%%%%%%%%%%%%%
\subsection{A coupled algorithm with provable regret bound}
\label{sec:coupled-reg}

We first propose \emph{Correlation-Aware Bandits with Surrogates Coupled} (CABS-C), detailed in Algorithm~\ref{alg:CABS-C}. CABS-C builds on graph-feedback variants~\cite{zhang2023practical, zhang2024efficient} of SquareCB by using the feedback graph to guide exploration, but adapts this framework to the surrogate-feedback setting. Specifically, CABS-C de-biases surrogate rewards from correlated arms and pools them with the selected arm's true reward to train a single reward model. Each arm is modeled as a function $f$ which could be linearly parametrized by $\hat{\boldsymbol{\theta}}_i, i \in[K]$ or implemented as a neural net. At round $t$, the algorithm observes context $\boldsymbol{x}_t$ and an affinity matrix $\boldsymbol{R}_t$, samples a stochastic feedback graph $G_t$ with adjacency matrix $\boldsymbol{A}_t$, and obtains reward estimates from the regression solver {\sffamily{RS}}. The action distribution $\boldsymbol{p}_t \in \Delta([K])$, where $\Delta([K])$ is $(K-1)$-dimensional simplex,  is then computed by solving a optimization problem (Line~\ref{line:cabsc-predict-end}), which chooses an arm by balancing low immediate regret with high information gain from the feedback graph. Unlike a standard contextual bandit rule, which only trades off reward and uncertainty for the selected arm, this step explicitly accounts for additional side observations revealed through graph neighbors; thus, it may select an arm that is slightly suboptimal in immediate reward if it yields more informative feedback overall. The parameter $\gamma$ controls this trade-off.

After sampling an arm $i_t$ according to $\boldsymbol{p}_t$, the learner observes its true reward and receives surrogate rewards $\{s_{t,j}\}$ for graph neighbors $\mathcal{N}_t(i_t):=\{j\in[K]:A_{t,(i_t,j)}=1\}$, where each $s_{t,j}\in[0,1]$ is produced by the reward oracle \textsf{RO}. CABS-C then de-biases the surrogate feedback, constructs importance-weighted reward estimates, and updates \textsf{RS}. Thus, CABS-C interpolates between bandit and full-information learning by using context-dependent graph structure to propagate imperfect surrogate feedback beyond the selected arm. While graph-based side observations and surrogate feedback have each been studied independently in prior work, CABS-C jointly addresses both in a single correlation-aware bandit framework. 

We now turn to establishing the regret bound for the algorithm. Before, we state the assumptions.

\begin{assumption}[Realizability]\label{assumption:realizability}
Following~\cite{foster2020beyond}, we assume the expected \emph{true rewards} are realizable, i.e. there exists $f^\star\in\mathcal F$ such that, 
$
\mathbb{E}[r_{t,i}\mid \boldsymbol{x}_t]=f^\star(\boldsymbol{x}_t,i),
\;\; \forall t\in[T],\, i\in[K].
$

When we work with losses, we use the convention $\ell_{t,i}=1-r_{t,i}$, in which case the same realizability condition holds for losses as well (up to an affine transformation of the function class). In our work, we consider parametric function classes of the form $f_\theta(\boldsymbol{x},i)$, where $\theta$ denotes model parameters. In the linear contextual bandit special case, $f_\theta(\boldsymbol{x},i)=\boldsymbol{x}^\top \boldsymbol{\theta}_i$, whereas more generally $f_\theta$ may be a neural network or other nonlinear predictor.
\end{assumption}

\begin{assumption}[Regression Oracle]
\label{assumption:regression_oracle}
Following \cite{zhang2023practical}, we assume the regression oracle \rsolver\, guarantees that for any sequence $\{(\boldsymbol{x}_t, \hat r_{t,j})\}_{j \in \mathcal{O}_t, t \in [T]}$ in which $\mathcal O_t \subseteq [K]$, the oracle produces a function $\hat f_t$ such that
% \[
$
\sum_{t=1}^T\sum_{i\in \mathcal O_t}
\big(
f^\star(\boldsymbol{x}_t,i) - \hat f_t(\boldsymbol{x}_t,i)
\big)^2
\le \mathrm{Reg}_{\mathrm{Sq}}(T)
$.
% \]
\end{assumption}
Note that this assumption holds for any function class. In our analysis, we stick to linear functions, i.e, $\hat f_t(\boldsymbol{x}_t,i) = \boldsymbol{x}_t^{\top} \hat{\boldsymbol{\theta}}_{t,i}$ and $f^\star(\boldsymbol{x}_t,i) = \boldsymbol{x}_t^{\top}\boldsymbol{\theta}^*_{i}$.

\begin{assumption}[Strong Observability]
\label{asm:strong-observability}
For a directed graph $G_t = ([K], E)$, a node $i$ is observable if $\{ j \in [K] ;\;  {A}_{t,(j,i)} = 1 \} \neq \emptyset$. An observable node, $i$ is strongly observable if either $i \in \{ j \in [K] ;\; {A}_{t,(j,i)} = 1 \}$ or $\{ j \in [K] ;\; {A}_{t,(j,i)} = 1 \} = [K]\setminus \{i\}$. If all nodes $i \in [K]$ are strongly observable, then the graph $G_t$ is strongly observable and has independence number $\alpha$.  
\end{assumption}
Following \cite{alon2014nonstochasticmultiarmedbanditsgraphstructured}, since we assume that pulling an arm always reveals its own loss, we satisfy $i \in \{ j \in [K] ;\; {A}_{t,(j,i)} = 1 \}$, making the graph strongly observable. 

\begin{assumption}[ML Side Observation Model]\label{asm:ml-side-obs-model}
For every observed arm $i\in \mathcal{N}_t(i_t)$, the surrogate reward is
$s_{t,i}=f^\star(\boldsymbol{x}_t,i)+b_i+\xi_{t,i}$, where $b_i\in\mathbb{R}$
is a fixed arm-dependent bias and $\xi_{t,i}$ is a zero-mean noise term
satisfying $\mathbb{E}[\xi_{t,i}\mid\mathcal F_{t-1},\boldsymbol{x}_t]=0$
and $|\xi_{t,i}|\le\varepsilon_n$. 
\end{assumption}
Here filtration $\mathcal{F}_t$ denotes the information
available up to time $t$, and $\varepsilon_n$ bounds the surrogate noise
 across rounds and observed arms. Intuitively, $b_i$ captures
systematic surrogate over- or under-estimation for arm $i$, while larger
$\varepsilon_n$ corresponds to noisier and less informative side observations. We do not require the affinity and reward oracles to be statistically independent; the regret analysis only requires the stated graph and surrogate-noise conditions to hold.

\begin{assumption}[Minimum Probability]
\label{asm:min-probability}
Following \cite{alon2015online}, we make the assumption that $p_{t,i}\ge \epsilon$ for all $(t,i)$, where $p_{t,i}$ indicates the probability of selecting arm $i$ at round $t$. This can be achieved by mixing $\boldsymbol p_t$ with a uniform mass. Note that this $\epsilon$ is distinct from the $\varepsilon_n$ quantity from Assumption~\ref{asm:ml-side-obs-model}

\end{assumption}

Using the above assumptions, we introduce a lemma for concentrating the noise term of surrogate rewards, and then we introduce a theorem and accompanying corollary for bounding the expected regret of Algorithm~\ref{alg:CABS-C}. The proofs are deferred to Appendix~\ref{appendix:CABS-C}.

\begin{lemma}\label{lem:noise_conc}
Suppose the feedback graph $G_t$ is deterministic with independence number no more than $\alpha$. Then, for all $\delta > 0$, with probability $1-\delta$ Algorithm~\ref{alg:CABS-C} guarantees that
\begin{equation}
\sum_{t=1}^T \sum_{j=1}^K
\left|
\frac{w_{t,j}}{n_{t,j}}
\sum_{\tau \in \mathcal T_t(j)} \xi_{\tau,j}
\right|
\;\le\; \varepsilon_n\, \sqrt{
2 \alpha KT \ln\left(\frac{2KT}{\delta}\right)
\ln\!\left(\frac{4K}{\alpha\epsilon}\right)
(1+\ln T)} .
\end{equation}
\end{lemma}
%%%%%%%%%%%%%%%%%%%%%%%%%%%%%%%%%%%%%%%%%%%%%%
\begin{theorem}\label{theorem:CABS-C}
Suppose that the feedback graph $G_t$ is strongly observable, and has independence number no greater than $\alpha$. Under Assumptions~\ref{assumption:realizability}--\ref{assumption:regression_oracle}, with probability at least $1-\delta$, Algorithm~\ref{alg:CABS-C} guarantees that, $\mathbb{E}[\mathrm{Reg}_T]
\le
\widetilde O\!\left(
\sqrt{\alpha T\,\mathrm{Reg}_{\mathrm{Sq}}(T)}
+
\varepsilon_n\,\sqrt{\alpha K T}\
\right).$
\end{theorem}

 For an $m$-feedback graph, where each round reveals $m$ surrogate rewards, Theorem~\ref{theorem:CABS-C} yields the following linear-oracle specialization.

\begin{corollary}[Linear oracle with $m$-feedback]
\label{cor:CABS-C_linear_mfeedback}
Assume that we have a linear oracle and we observe surrogate rewards for $m$ arms per round. Then with probability at least $1-\delta$ CABS-C guarantees
\begin{align*}
    \mathbb{E}[\mathrm{Reg}_T]
\le 
\widetilde O\!\left(
\sqrt{d\frac{K}{m+1}T\log \left(\frac{T}{d}\right)}
\;+\;
\varepsilon_n \sqrt{\frac{K^2}{m+1}T}
\right)
\end{align*}

\end{corollary}
\noindent \textbf{Brief discussion of bias and noise.} If we ensure that \(\varepsilon_n = O(1/\sqrt{K})\), we recover the same asymptotic regret (up to polylogs) as the standard contextual bandit with graph feedback setting, $\widetilde O\!\left(\sqrt{\alpha T \, \mathrm{Reg_{Sq}}(T)}\right)$. If the bias and noise are well-behaved (each arm's respective bias $b_{i}$ remains relatively constant or changes gradually and noise is small and truly subgaussian), then CABS-C can utilize surrogate rewards to significantly improve the rate of learning (the additional information from the surrogate rewards reduces variance). However, if bias is difficult to estimate or if the magnitude of the subgaussian noise is large, then the surrogate rewards available to CABS-C may hinder its ability to learn (inconsistent surrogate rewards can potentially increase variance). 

\noindent \textbf{On bounding $\boldsymbol{\varepsilon_n}$.} 
We note that the ML reward predictor (\roracle), trained offline from a fixed mixture of LLMs, enjoys standard ERM generalization guarantees only with respect to the training distribution~\cite{vapnik2013nature, shalev2014understanding}. Generalization to unseen contexts reduces to a covariate shift problem~\cite{shimodaira2000improving, sugiyama2012machine}, and fundamentally depends on support overlap or structural invariance assumptions on the underlying data-generating process. In the next section, we propose an algorithm that is robust to such noise, while retaining the aforementioned regret bound when $\varepsilon_n$ is small.

\noindent\textbf{Cost-informed utility.} To capture the accuracy--cost trade-off inherent in practical LLM deployment, we define a cost-aware utility: $r_{t,i_t} = y_{t, i_t} - \rho\, c_{t, i_t}$, where $y_{t, i_t} \in [0,1]$ denotes the accuracy score obtained by using LLM $i_t$ to answer query $\boldsymbol{x}_t$, and $c_{t, i_t} \ge 0$ is the associated inference cost. The parameter $\rho \ge 0$ controls cost sensitivity: larger values favor cheaper models, while smaller values prioritize accuracy. We additionally clip the utility to $[0,1]$ to enforce bounded rewards.

\noindent\textbf{From linear to neural bandits.} While we focus on linear models for clarity of exposition and analysis, our approach extends naturally to nonlinear reward models. In particular, following the machinery (Neural Tangent Kernel matrix analysis) adapted in NeuralUCB~\cite{zhou2020neural}, the linear predictor $\boldsymbol{x}_t^\top  \boldsymbol{\theta^*}_{i}$ can be replaced by a neural network with uncertainty estimated via its last-layer features. At this point, we empirically adapt our method to neural bandits based estimators and show their performance in Section~\ref{sec:experiments}, but note that our regret analysis can be extend to this setting be utilizing the machinery developed in~\cite{zhou2020neural}.

\subsection{A decoupled method with optimal regret}

Theorem~\ref{theorem:CABS-C} shows that surrogate rewards can substantially improve regret for CABS-C algorithm, but can also degrade performance when their bias or noise is poorly controlled. This is a consequence of \emph{reward mixing}: CABS-C combines true and surrogate rewards at the observation level and trains a single model on the pooled feedback. To obtain robustness to unreliable surrogate rewards, we instead consider a \emph{prediction-mixing} strategy, where different feedback mechanisms are handled by separate learners and combined only at the decision level.

We introduce \emph{Correlation-Aware Bandits with Surrogates Decoupled} (CABS-D), shown in Algorithm~\ref{alg:CABS-D}. CABS-D maintains two stateful experts: a standard bandit-feedback expert and the surrogate-aware graph-feedback expert CABS-C from Algorithm~\ref{alg:CABS-C}. In our theoretical analysis, the standard expert is SquareCB, described in Appendix~\ref{appendix:standard-contextual-bandit}. At each round, the experts output policies $\boldsymbol{p}_{t,1}$ and $\boldsymbol{p}_{t,2}$, which are mixed into $\boldsymbol{q}_t$ using adaptive exponential weights over meta-copies $(m,g)$, in the spirit of Hedge/AdaHedge-style aggregation~\cite{erven2011adaptive}. After sampling $i_t\sim\boldsymbol{q}_t$, SquareCB receives only the selected arm's true reward, while CABS-C also receives surrogate rewards from the sampled feedback graph. The meta-weights are updated using importance-weighted loss estimates and an adaptive second-order learning rate, allowing CABS-D to track the better feedback mechanism without knowing whether surrogates are reliable. Thus, CABS-D retains standard bandit robustness while exploiting informative surrogate feedback. We now state its regret bound.

\begin{algorithm}[t]
\caption{Correlation-Aware Bandits with Surrogates Decoupled (\sffamily{CABS-D})}
\label{alg:CABS-D}
\begin{algorithmic}[1]
\STATE \textbf{Input:} $T$ Rounds; geometric grid 
$\Gamma=\{\gamma^{(1)},\dots,\gamma^{(L)}\}$ with $\gamma^{(g+1)}=2\gamma^{(g)}$; set $M=2L$

\STATE \textbf{Oracles:} Reward oracle $\roracle: \mathcal{X} \times [K] \times \mathbb{R}_{+} \rightarrow \mathbb{R}_+^{K}$, affinity oracle $\aoracle: \mathcal{X} \rightarrow \mathbb{R}^{K \times K}$
\STATE \textbf{Experts:} $m\in\{1,2\}$: SquareCB bandit-feedback $(m=1)$; CABS-C graph-feedback $(m=2)$
\STATE \textbf{Initialize:}  For every meta-copy $(m,g)$ set weight $w_{1,(m,g)} = 1$, and $\Delta_0 = 0$

\FOR{$t=1$ \textbf{to} $T$}
    \STATE Receive context  $\boldsymbol{x}_t$ and corresponding correlation matrix $\boldsymbol{R}_t$ from $\aoracle(\boldsymbol{x}_t)$. 

    \STATE Generate affinity graph $G_t$ with adjacency matrix $\boldsymbol{A}_t$: $A_{t,(i,j)}=1$ w.p. $\sigma((1-R_{t,(i,j)})/2)$

    \STATE Receive policy $\boldsymbol{p}_{t,1}\leftarrow\mathrm{SquareCB}(\boldsymbol{x}_t)$, policy $\boldsymbol{p}_{t,2}\leftarrow\mathrm{CABS\text{-}C}(\boldsymbol{x}_t)$ 
    
    \STATE Compute meta-probabilities and the global mixture policy:
    \[
    \mu_{t,(m,g)}
    =
    \frac{w_{t,(m,g)}}{\sum_{m'=1}^{2}\sum_{g'=1}^{L} w_{t,(m',g')}},
    \qquad
    q_{t,i}
    =
    \sum_{m=1}^{2}\sum_{g=1}^{L}
    \mu_{t,(m,g)}p_{t,m,i},
    \quad i\in[K].
    \]
    
    \STATE Sample action $i_t\sim \boldsymbol{q}_t$ and receive true reward $r_{t,i_t}$.

    \STATE Observe surrogate rewards, $\{s_{t,i}\}$ $\leftarrow \roracle(\boldsymbol{x}_t,i_t,r_{t,i_t})$   $\forall i \in \mathcal{N}_t(i_t)= \{j:A_{t,(i_t,j)}=1\}$.

    \STATE Send $(\boldsymbol{x}_t,i_t,r_{t,i_t})$ to SquareCB and $\left(\boldsymbol{x}_t,i_t,r_{t,i_t}, \{s_{t,i}\}\right)$ to CABS-C for policy updates.
    
    \STATE Define losses
    $
    \ell_{t,i_t}=1-r_{t,i_t}
    $
    and
    $
    \ell_{t,i}=1-s_{t,i}
    $
    for all $i\in \mathcal{N}_t(i_t)$

    \STATE For each copy $(m,g)$ with parameter $\gamma^{(g)}$, compute
    \[
    \hat\ell_{t,(1,g)}
    =
    \frac{\ell_{t,i_t}p_{t,1,i_t}}
    {q_{t,i_t}+\gamma^{(g)}},
    \qquad
    \hat\ell_{t,(2,g)}
    =
    \sum_{i\in \mathcal{N}_t(i_t)}
    \frac{\ell_{t,i}p_{t,2,i}}
    {\sum_{j:A_{t,(j,i)}=1}q_{t,j}+\gamma^{(g)}} .
    \]

    \STATE Set $ 
    \gamma_{\min}=\min_{g\in[L]}\gamma^{(g)},
    \qquad
    \eta_t
    =
    \min\left\{
    \gamma_{\min},
    \sqrt{\frac{\ln M}{1+\Delta_{t-1}}}
    \right\}.
    $

    \STATE  Compute the second-order term $V_t$ and update $\Delta_t$:
    \[
    V_t
    =
    \sum\nolimits_{m=1}^{2}\sum\nolimits_{g=1}^{L}
    \mu_{t,(m,g)}
    \left(\hat\ell_{t,(m,g)}\right)^2,
    \qquad
    \Delta_t=\Delta_{t-1}+V_t .
    \]

    \STATE Update each meta-copy weight:
    $
    w_{t+1,(m,g)}
    =
    w_{t,(m,g)}
    \exp\!\left(-\eta_t\hat\ell_{t,(m,g)}\right).
    $
\ENDFOR
\end{algorithmic}
\end{algorithm}

\begin{theorem}[Meta guarantee, two-expert; constants explicit]
\label{theorem:CABS-D}
Let there be $M=2$ experts: one bandit-feedback expert, SquareCB and one graph-feedback expert, CABS-C (tuned for known independence number $\alpha$). Assume losses lie in $[0,1]$. Then Algorithm~\ref{alg:CABS-D} satisfies, 
$
\mathbb{E}\!\left[\mathrm{Reg}_T\right] \le O\!\left(\min\left\{R_1 + \Gamma_1,\; R_2 + \Gamma_2\right\}\right),
$
where $\Gamma_1 = \ln\!\left(KT\ln T\right)\sqrt{KT\ln M \ln T}$ and 
$\Gamma_2 = \ln\!\left( \alpha T \ln(T/\alpha\epsilon)\ln T\right)\sqrt{\alpha T \ln(T/\alpha\epsilon) \ln M \ln T}$.
Here $R_1$ is the regret of the bandit expert and $R_2$ the regret of the graph expert. Using $R_1 = \widetilde O(\sqrt{KT \, \mathrm{Reg_{Sq}}(T)})$ and $R_2 = \widetilde O(\sqrt{\alpha T \, \mathrm{Reg_{Sq}}(T)} + \sqrt{\alpha KT}\,\varepsilon_n)$, we obtain the regret bound (absorbing polylogs)
\[
\mathbb{E}\!\left[\mathrm{Reg}_T\right] = \widetilde O\Big( \min \{\sqrt{KT \, \mathrm{Reg_{Sq}}(T)},\; \sqrt{\alpha T \, \mathrm{Reg_{Sq}}(T)} + \varepsilon_n\sqrt{\alpha KT}\,\}\Big).
\]
\end{theorem}
The proof can be found in Appendix~\ref{appendix:CABS-D}. While the theorem establishes a general regret guarantee for the  algorithm, the following corollary shows that surrogate rewards strictly improve the regret bound. 
%%%%%%%%%%%%%%%%%%%%%%%%%%%%%%%%%%%%
\begin{corollary}[Linear oracle with $m$-feedback]
\label{cor:CABS-D_linear_mfeedback}
Assume that we have a linear oracle and we observe surrogate rewards for $m$ arms per round. Then CABS-D guarantees
\[
\mathbb{E}[\mathrm{Reg}_T] \le \widetilde O\!\left(\min\left\{\sqrt{d\frac{K}{m+1}T\log\!\left(\frac{T}{d}\right)}
\;+\;\varepsilon_n\allowbreak\sqrt{\frac{K^2}{m+1}T},\;  \sqrt{dKT \log\!\left(\frac{T}{d}\right)}\right\}\right).
\]
\end{corollary}
In a standard linear contextual bandit setting (which corresponds to having no extra surrogate feedback, or $m = 0$), the typical  expected regret bound over $T$ rounds is $\mathbb{E}[\text{Reg}_T] \le \widetilde{O}\left(\sqrt{dKT }\right)$. This baseline matches the second term inside the $\min\{\cdot\}$ operator of the regret in corollary. Because the overall bound is a minimum between two strategies, the algorithm guarantees performance that is at least as good as the vanilla baseline, but can be significantly better depending on the value of $m$ and the noise $\varepsilon_n$. Specifically, for small enough $\varepsilon_n$ and as long as $m > 0$,  the first term inside the minimum takes effect  where the second term in the sum vanishes. We note that the first term in summand remains the dominant factor, as long as  $\varepsilon_n$  satisfies the condition $\varepsilon_n < \tilde{O}(\sqrt{{d}/{K}})$, which can be easily satisfied by training the offline surrogate model with enough samples. In this case,  the regret is entirely dominated by $\widetilde{O}\left(\sqrt{d\left(\frac{K}{m+1}\right)T }\right)$, yielding a significant improvement over the vanilla $\widetilde{O}\left(\sqrt{dKT}\right)$ bound  as the effective action-space dependence drops from $K$ to ${K}/({m+1})$. This significantly shrinks both the  exploration and the noise propagation terms, allowing the algorithm to learn the underlying linear reward structure much faster than standard bandit feedback would allow.

\section{Experiments}
\label{sec:experiments}

We evaluate four aspects of our proposed CABS: the benefit of decoupled versus coupled surrogate integration, utility across cost regimes, accuracy vs cost trade-offs, and online sample efficiency.

\begin{table}[t]
\centering
\caption{
Average utility (higher is better) and average regret (lower is better) across cost regimes on the SPROUT and RouterBench datasets, repeated over 10 trials; 95\% confidence intervals are provided in Appendix~\ref{appendix:rq1}.
}
\label{tab:cabs_ablation}
\renewcommand{\arraystretch}{1.12}
\setlength{\tabcolsep}{7pt}
\footnotesize
\begin{tabular}{l c c c c c c c c}
\toprule
& \multicolumn{4}{c}{\textbf{SPROUT}}
& \multicolumn{4}{c}{\textbf{RouterBench}} \\
\cmidrule(lr){2-5} \cmidrule(lr){6-9}
\textbf{Regime (Metric)}
& \multicolumn{2}{c}{CABS-C}
& \multicolumn{2}{c}{CABS-D}
& \multicolumn{2}{c}{CABS-C}
& \multicolumn{2}{c}{CABS-D} \\
\cmidrule(lr){2-3} \cmidrule(lr){4-5}
\cmidrule(lr){6-7} \cmidrule(lr){8-9}
& Linear & Neural & Linear & Neural
& Linear & Neural & Linear & Neural \\
\midrule

Low (Utility)
& 0.7233 & 0.7558 & \textbf{0.7715} & 0.7432
& 0.6350 & 0.6421 & \textbf{0.6754} & 0.6638 \\

Medium (Utility)
& 0.6015 & 0.6439 & \textbf{0.6632} & 0.6467
& 0.5308 & 0.5551 & \textbf{0.5817} & 0.5759 \\

High (Utility)
& 0.5407 & 0.4007 & \textbf{0.6265} & 0.5876
& 0.4962 & 0.5110 & \textbf{0.5466} & 0.5372 \\

\midrule

Low (Regret)
& 0.2264 & 0.1939 & \textbf{0.1782} & 0.2065
& 0.2540 & 0.2469 & \textbf{0.2136} & 0.2252 \\

Medium (Regret)
& 0.2979 & 0.2555 & \textbf{0.2362} & 0.2527
& 0.3021 & 0.2778 & \textbf{0.2512} & 0.2570 \\

High (Regret)
& 0.3249 & 0.4649 & \textbf{0.2391} & 0.2780
& 0.2946 & 0.2798 & \textbf{0.2442} & 0.2536 \\

\bottomrule
\end{tabular}
\end{table}

\begin{figure*}[t]
    \centering
    \includegraphics[width=\linewidth]{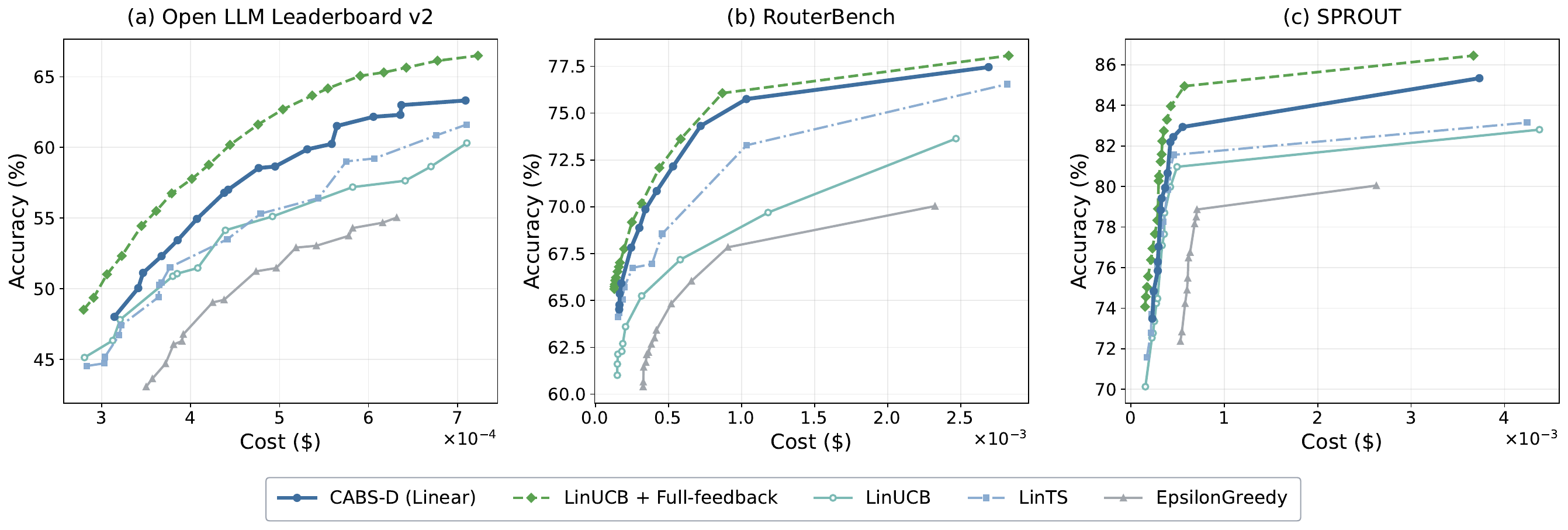}
    \caption{
    Accuracy vs cost Pareto Frontier curves across datasets on Linear Online models. 
    }
    \vspace{-2mm}
\label{fig:acc_vs_cost_linear_results}
\end{figure*}

\subsection{Experimental Setup}

\noindent\textbf{Datasets.}~We evaluate on three LLM routing benchmarks: RouterBench~\cite{hu2024routerbenchbenchmarkmultillmrouting}, SPROUT~\cite{somerstep2025carrot}, and Open LLM Leaderboard v2~\cite{open-llm-leaderboard-v2}. Dataset details and splits are deferred to Appendix~\ref{app:datasets}.

\noindent\textbf{Baselines.}~We compare against standard online contextual bandit algorithms, including LinUCB~\cite{chu2011contextual}, LinTS~\cite{agrawal2013thompson}, and $\epsilon$-greedy. In the LLM routing setting, observing rewards for unplayed arms is generally unrealistic, as only the selected model is executed. Accordingly, we primarily evaluate methods under bandit feedback (single observed reward per round). For reference, we additionally report hypothetical full-feedback variants (LinUCB/NeuralUCB with full feedback) as information upper bounds. We also include strong supervised routing baselines, such as RoBERTa/CARROT~\cite{somerstep2025carrot} and kNN-based routers~\cite{hu2024routerbenchbenchmarkmultillmrouting}, as well as random selection and single-model baselines. 

\noindent\textbf{Metrics.}~We report relative cumulative regret and average utility over time. Specifically, the relative cumulative regret of an online learner `$\textit{alg}$' is defined with respect to `$olg$' as $R_{\textit{alg}}(t) - R_{\text{olg}}(t)$, where $R_{\textit{alg}}(t)$ denotes the cumulative regret at time $t$. In addition, we summarize test accuracy and average per-query cost, and characterize accuracy–cost trade-offs via the accuracy–cost Pareto frontier by sweeping the cost-sensitivity parameter for each method.

\noindent\textbf{Setup.}~We study LLM routing under varying accuracy and cost trade-offs by sweeping the cost-sensitivity parameter over $\rho\in\{0,50,\ldots,1000\}$, which we group into three cost regimes: low cost sensitivity corresponds to small values of $\rho$ ($\le 300$) where accuracy dominates, medium cost sensitivity ($300 < \rho < 800$) represents a balanced trade-off, and high cost sensitivity ($\rho \ge 800$) corresponds to large values of $\rho$ where cost dominates. For correlation-aware methods, the Reward Oracle (\roracle) imputes surrogate utilities for graph-neighbor arms, while the Affinity Oracle (\aoracle) constructs the context-dependent feedback graph from predicted inter-arm similarity. Both are lightweight MLP heads on top of mDeBERTaV3-base query embeddings~\cite{he2021deberta}. Refer to Appendix~\ref{appendix:implementation} for implementation details.
\vspace{-2mm}
\subsection{Reward Mixing vs.\ Prediction Mixing}
\label{sec:reward_mixing_vs_prediction_mixing}\vspace{-2mm}
Table~\ref{tab:cabs_ablation} compares the coupled (CABS-C) and decoupled (CABS-D) variants by reporting average utility (higher is better) and average regret (lower is better) across three cost regimes on the SPROUT and RouterBench datasets (see Appendix~\ref{appendix:rq1} for corresponding confidence intervals and additional results on Open LLM Leaderboard V2). Across all regimes, CABS-D Linear, i.e, model mixing, improves utility and reduces regret relative to CABS-C, indicating that \emph{decoupling} reward learning from surrogate learning yields a more reliable use of surrogate signals. This is aligned with our theory: while coupling can amplify surrogate noise or bias, the decoupled construction can exploit informative surrogate predictions early while falling back to reward-based learning as data accumulates. Based on these results, we use CABS-D as the default instantiation in the remainder of our experiments.

\begin{figure*}[h]
    \centering
\vspace{-3mm}     \includegraphics[width=\linewidth]{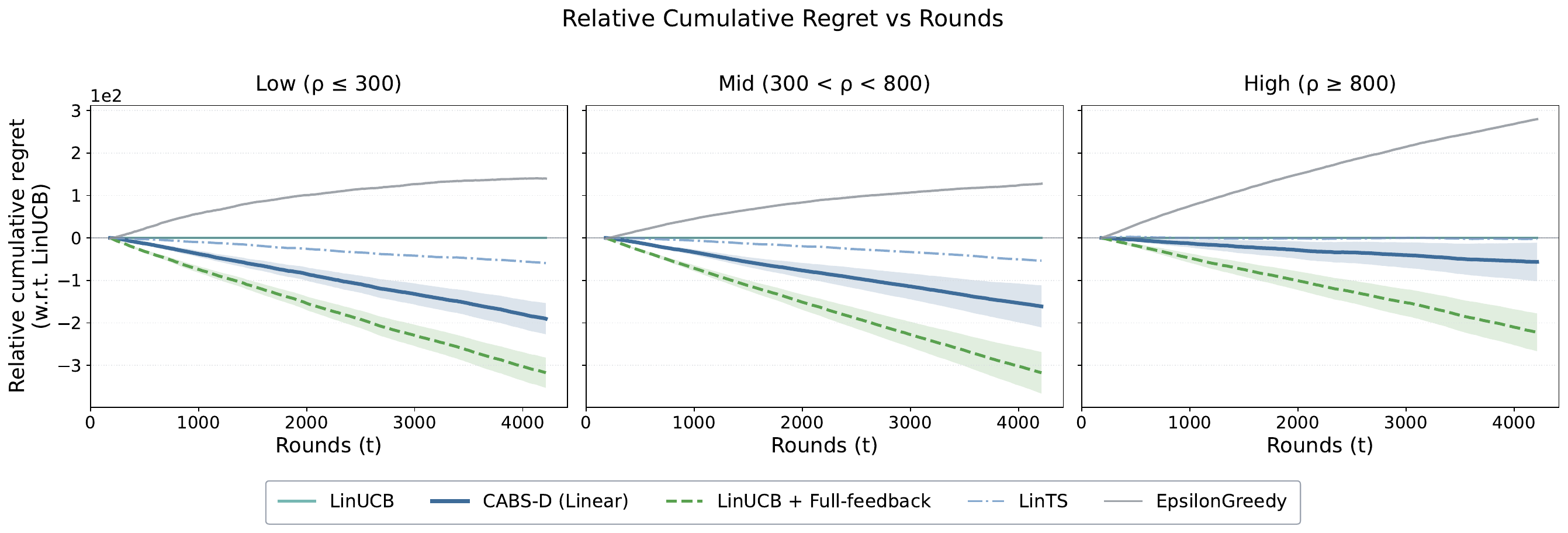}
  \vspace{-2mm}  \caption{
        Relative Cumulative regrets with respect to LinUCB for other linear online learning methods on Open LLM Leaderboard V2 dataset across low, medium and high cost regimes.
    }
    \label{fig:regret_curves}\vspace{-4mm}
\end{figure*}

\begin{table*}[ht]
\centering
\caption{Utilities across cost regimes and datasets. For online learning methods, we report mean utility over 10 independent trials; 95\% confidence intervals are reported in Appendix~\ref{app:cis_and_encoder_robustness}. CARROT and RouterBench are offline routing baselines and are reported as single-run utilities. }\vspace{-1mm}
\label{tab:utility_results}
\renewcommand{\arraystretch}{1.12}
\setlength{\tabcolsep}{6pt}

\resizebox{\textwidth}{!}{%
\begin{tabular}{ l | c c c | c c c | c c c }
\hline
& \multicolumn{3}{c|}{\textbf{Open LLM LB v2}} 
& \multicolumn{3}{c|}{\textbf{RouterBench}} 
& \multicolumn{3}{c}{\textbf{SPROUT}} \\
\cline{2-4} \cline{5-7} \cline{8-10}
\textbf{Method}
& Low & Medium & High
& Low & Medium & High
& Low & Medium & High \\
\hline

CARROT~\cite{somerstep2025carrot}
& 0.5118 & 0.3101 & 0.1948
& 0.6457 & 0.4820 & 0.3960
& 0.7615 & 0.5692 & 0.4560 \\

RouterBench~\cite{hu2024routerbenchbenchmarkmultillmrouting}
&  0.4045 & 0.2192 & 0.1268
& 0.5920 & 0.5103 & 0.4704 
& 0.7196 & 0.6034 & 0.5017 \\

\hline

Random Selection
& 0.3822 & 0.2926 & 0.1889	
& 0.4509 & 0.3431 & 0.2899
& 0.5872 & 0.4472 & 0.3863 \\

Epsilon-Greedy
& 0.4544 & 0.2543 & 0.1153
& 0.6057 & 0.5169 & 0.4751
& 0.7194 & 0.6138 & 0.5655 \\

LinTS~\cite{agrawal2013thompson}
& 0.5017 & 0.2973 & 0.1822
& 0.6514 & 0.5725 & 0.5381
& 0.7628 & 0.6548 & 0.6200 \\

LinUCB~\cite{chu2011contextual}
& 0.4876 & 0.2846 & 0.1817
& 0.6276 & 0.5437 & 0.5082
& 0.7579 & 0.6561 & 0.6110 \\

NeuralUCB~\cite{zhou2020neural}
& 0.4493 & 0.2471 & 0.1558 
& 0.6344 & 0.5418 & 0.4950
 & 0.7086 & 0.6266 & 0.5501 \\

MixLLM~\cite{wang2025mixllm}
& 0.4334 
& 0.2897 
& 0.1842 
& 0.6364 
& 0.5588 
& 0.4991 
& 0.7265 
& 0.6433 
& 0.5968  \\

BARP~\cite{wei2025learning}
& 0.4819 
& 0.2932 
& 0.1875 
& 0.6350 
& 0.5683 
& 0.4251  
& 0.7378 
& 0.6370 
& 0.6246  \\

\rowcolor{lightgray}
CABS-D (Linear)
& \textbf{0.5328} & 0.3229 & 0.1951
& \textbf{0.6754} & \textbf{0.5817} & \textbf{0.5466}
& \textbf{0.7715} & \textbf{0.6632} & \textbf{0.6265}  \\

\rowcolor{lightgray}
CABS-D (Neural)
& 0.5284 & \textbf{0.3289} & \textbf{0.1956} 
& 0.6638 & 0.5759 & 0.5372
& 0.7432 & 0.6467 & 0.5876 \\

\hline
LinUCB + Full-Feedback
& 0.5630 & 0.3600 & 0.2345
& 0.6838 & 0.6024 & 0.5660
& 0.7957 & 0.7070 & 0.6643 \\

NeuralUCB + Full-Feedback
& 0.5454 & 0.3620 & 0.2388
& 0.6704 & 0.6050 & 0.5588
& 0.7732 & 0.6842 & 0.6246 \\

\hline
\end{tabular}%
}\vspace{-5mm}
\end{table*}

%%%%%%%%%%%%%%%%%%%%%%%%%%%%%%%%%%

\vspace{-1mm}
\subsection{Utility across cost regimes}\vspace{-1mm}
\label{sec:utility_cost_regimes}

Table~\ref{tab:utility_results} reports average utility (higher is better) across three routing benchmarks under Low/Mid/High cost-sensitivity regimes. Across all datasets and regimes, CABS-D is the strongest online method: either its linear or neural instantiation attains the best non-oracle utility across cost regimes. Relative to the strongest learned online baseline in each regime, including BARP and MixLLM, CABS-D improves utility by an average of approximately $7.0\%$ on Open LLM LB v2, $2.3\%$ on RouterBench, and $0.8\%$ on SPROUT. The gains are largest on Open LLM LB v2, where CABS-D improves over the strongest online baseline by approximately $6.2\%$, $10.6\%$, and $4.3\%$ in the low-, medium-, and high-cost regimes, respectively. On RouterBench and SPROUT, CABS-D also achieves higher utility than the online baselines and remains competitive with strong static routers (CARROT and RouterBench), showing that correlation-aware online learning can match or exceed fixed routers while adapting to the target cost regime. Full-feedback variants provide an information-rich oracle reference, while CABS-D closes a substantial portion of this gap under realistic bandit feedback. Confidence intervals and results with an alternative text encoder are reported in Appendix~\ref{appendix:experiments}.

\vspace{-2mm}
\subsection{Accuracy vs Cost Trade-off}
\vspace{-2mm}

Figure~\ref{fig:acc_vs_cost_linear_results} shows the trade-off between accuracy and average per-query cost on Open LLM Leaderboard v2, RouterBench, and SPROUT. Each curve traces a routing policy as the cost-sensitivity parameter $\rho$ varies $ \in [0,1000]$. Across all datasets, CABS-D consistently improve the accuracy versus cost Pareto frontier compared to standard online bandit baselines. In particular, CABS-D achieves higher accuracy at comparable cost levels (or equivalently lower cost for a given accuracy) relative to LinUCB, LinTS, and $\epsilon$-greedy. While full-feedback variants provide an optimistic upper bound, CABS-D substantially narrows this gap under realistic bandit feedback.

\vspace{-2mm}
\subsection{Learning dynamics}
\label{sec:learning_dynamics}
\vspace{-2mm}

Figure~\ref{fig:regret_curves} reports relative cumulative regret over time relative to LinUCB, across low, mid, and high cost-sensitivity regimes for Open LLM Leaderboard V2. Across all regimes, CABS-D consistently achieves lower cumulative regret than standard bandit baselines including LinUCB, LinTS, and $\epsilon$-greedy. The improvement is most pronounced in the low- and medium-cost-sensitivity regimes, where CABS-D achieves the largest reductions in cumulative regret relative to LinUCB. Notably, the regret gap emerges early and persists throughout the horizon, suggesting that surrogate feedback provides useful additional learning signals that accelerate online learning and improve sample efficiency.

%%%%%%%%%%%%%%%%%%%%%%%
\vspace{-1mm}\section{Conclusion and Limitations}\vspace{-1mm}
\label{sec:conclusion}

We introduced correlation-aware contextual bandits with surrogate rewards and proposed two algorithms: CABS-C, a coupled surrogate-reward method, and CABS-D, a decoupled method with a best-of-both-worlds regret guarantee. Across three LLM-routing benchmarks, CABS-D improves online sample efficiency and accuracy vs cost trade-offs relative to the evaluated baselines. Our work has several limitations. First, our guarantees rely on an idealized surrogate model with fixed arm-specific bias and conditionally zero-mean bounded residual noise; in practice, our residual analysis reveals additional context-dependent error structure. Second, the Affinity Oracle uses predicted outcome similarity as a proxy for inter-arm affinity and may degrade under distribution shift. Third, the current formulation assumes one true reward per round and incurs additional computation that scales with the number of candidate models. We discuss these limitations in greater detail, along with oracle dependence, empirical diagnostics, sparse or delayed feedback, computational overheads, and practical extensions, in Appendix~\ref{sec:discussion}.

\newpage
\begin{ack}
This material is based upon work supported in part by the National Science Foundation (NSF) under Grant Number 2318101 and in part by the PRISM Center under the JUMP 2.0 Program sponsored by the Semiconductor Research Corporation (SRC). Any opinions, findings, conclusions, or recommendations expressed in this material are those of the author(s) and do not necessarily reflect the views of the NSF, the PRISM Center, or the SRC.
\end{ack}

\bibliographystyle{plainnat}
\bibliography{references}

@article{chiang2025llm,
  title={LLM Routing with Dueling Feedback},
  author={Chiang, Chao-Kai and Ishida, Takashi and Sugiyama, Masashi},
  journal={arXiv preprint arXiv:2510.00841},
  year={2025}
}

@book{lattimore2020bandit,
  title={Bandit algorithms},
  author={Lattimore, Tor and Szepesv{\'a}ri, Csaba},
  year={2020},
  publisher={Cambridge University Press}
}

@inproceedings{zhou2020neural,
  title={Neural contextual bandits with ucb-based exploration},
  author={Zhou, Dongruo and Li, Lihong and Gu, Quanquan},
  booktitle={International conference on machine learning},
  pages={11492--11502},
  year={2020},
  organization={PMLR}
}

@inproceedings{foster2020beyond,
  title={Beyond ucb: Optimal and efficient contextual bandits with regression oracles},
  author={Foster, Dylan and Rakhlin, Alexander},
  booktitle={International conference on machine learning},
  pages={3199--3210},
  year={2020},
  organization={PMLR}
}

@article{auer2002nonstochastic,
  title={The nonstochastic multiarmed bandit problem},
  author={Auer, Peter and Cesa-Bianchi, Nicolo and Freund, Yoav and Schapire, Robert E},
  journal={SIAM journal on computing},
  volume={32},
  number={1},
  pages={48--77},
  year={2002},
  publisher={SIAM}
}

@misc{nguyen2024metallm,
      title={MetaLLM: A High-performant and Cost-efficient Dynamic Framework for Wrapping LLMs}, 
      author={Quang H. Nguyen and Thinh Dao and Duy C. Hoang and Juliette Decugis and Saurav Manchanda and Nitesh V. Chawla and Khoa D. Doan},
      year={2025},
      eprint={2407.10834},
      archivePrefix={arXiv},
      primaryClass={cs.LG},
      url={https://arxiv.org/abs/2407.10834}, 
}

@article{zhang2023practical,
  title={Practical contextual bandits with feedback graphs},
  author={Zhang, Mengxiao and Zhang, Yuheng and Vrousgou, Olga and Luo, Haipeng and Mineiro, Paul},
  journal={Advances in Neural Information Processing Systems},
  volume={36},
  pages={30592--30617},
  year={2023}
}

@article{chen2024routerdc,
  title={Routerdc: Query-based router by dual contrastive learning for assembling large language models},
  author={Chen, Shuhao and Jiang, Weisen and Lin, Baijiong and Kwok, James and Zhang, Yu},
  journal={Advances in Neural Information Processing Systems},
  volume={37},
  pages={66305--66328},
  year={2024}
}

@misc{hu2024routerbenchbenchmarkmultillmrouting,
      title={RouterBench: A Benchmark for Multi-LLM Routing System}, 
      author={Qitian Jason Hu and Jacob Bieker and Xiuyu Li and Nan Jiang and Benjamin Keigwin and Gaurav Ranganath and Kurt Keutzer and Shriyash Kaustubh Upadhyay},
      year={2024},
      eprint={2403.12031},
      archivePrefix={arXiv},
      primaryClass={cs.LG},
      url={https://arxiv.org/abs/2403.12031}, 
}

@inproceedings{chu2011contextual,
  title={Contextual bandits with linear payoff functions},
  author={Chu, Wei and Li, Lihong and Reyzin, Lev and Schapire, Robert},
  booktitle={Proceedings of the fourteenth international conference on artificial intelligence and statistics},
  pages={208--214},
  year={2011},
  organization={JMLR Workshop and Conference Proceedings}
}

@inproceedings{ong2024routellm,
  title={RouteLLM: Learning to Route LLMs from Preference Data},
  author={Ong, Isaac and Almahairi, Amjad and Wu, Vincent and Chiang, Wei-Lin and Wu, Tianhao and Gonzalez, Joseph E and Kadous, M Waleed and Stoica, Ion},
  booktitle={The Thirteenth International Conference on Learning Representations},
  year={2025}
}

@misc{huggingface_models,
  author = {{Hugging Face}},
  title = {Hugging Face models},
  year = {2025},
  howpublished = "\url{https://huggingface.co/models}",
  note = "Accessed: 2025-05-10"
}

@misc{chatgpt2025,
  author = {{OpenAI}},
  title = {ChatGPT},
  year = {2025},
  howpublished = "\url{https://chat.openai.com/chat}",
  note = "Accessed: 2025-05-10"
}

@misc{gemini15pro,
  author = {{Google AI}},
  title = {Gemini 1.5 Pro Model Overview},
  year = {2024},
  howpublished = "\url{https://ai.google.dev/gemini-api/docs/models/gemini}",
  note = "Accessed: 2025-05-10"
}

@misc{grok1_mla,
  author = {{xAI}},
  title = {Open Release of Grok-1},
  year = {2024},
  howpublished = "\url{https://x.ai/blog/grok-os}",
  note = "Accessed: 2025-05-10"
}

@article{chen2023frugalgptuselargelanguage,
  title={FrugalGPT: How to Use Large Language Models While Reducing Cost and Improving Performance},
  author={Chen, Lingjiao and Zaharia, Matei and Zou, James},
  journal={Transactions on Machine Learning Research},
  year={2024}
}

@misc{huang2025thriftllm,
      title={ThriftLLM: On Cost-Effective Selection of Large Language Models for Classification Queries}, 
      author={Keke Huang and Yimin Shi and Dujian Ding and Yifei Li and Yang Fei and Laks Lakshmanan and Xiaokui Xiao},
      year={2025},
      eprint={2501.04901},
      archivePrefix={arXiv},
      primaryClass={cs.DB},
      url={https://arxiv.org/abs/2501.04901}, 
}

@misc{zhao2024eagle,
      title={Eagle: Efficient Training-Free Router for Multi-LLM Inference}, 
      author={Zesen Zhao and Shuowei Jin and Z. Morley Mao},
      year={2024},
      eprint={2409.15518},
      archivePrefix={arXiv},
      primaryClass={cs.LG},
      url={https://arxiv.org/abs/2409.15518}, 
}

@misc{jiang2023llm,
      title={LLM-Blender: Ensembling Large Language Models with Pairwise Ranking and Generative Fusion}, 
      author={Dongfu Jiang and Xiang Ren and Bill Yuchen Lin},
      year={2023},
      eprint={2306.02561},
      archivePrefix={arXiv},
      primaryClass={cs.CL},
      url={https://arxiv.org/abs/2306.02561}, 
}

@misc{stripelis2024tensoropera,
      title={TensorOpera Router: A Multi-Model Router for Efficient LLM Inference}, 
      author={Dimitris Stripelis and Zijian Hu and Jipeng Zhang and Zhaozhuo Xu and Alay Dilipbhai Shah and Han Jin and Yuhang Yao and Salman Avestimehr and Chaoyang He},
      year={2024},
      eprint={2408.12320},
      archivePrefix={arXiv},
      primaryClass={cs.AI},
      url={https://arxiv.org/abs/2408.12320}, 
}

@misc{li2025llmbanditcostefficientllm,
      title={LLM Bandit: Cost-Efficient LLM Generation via Preference-Conditioned Dynamic Routing}, 
      author={Yang Li},
      year={2025},
      eprint={2502.02743},
      archivePrefix={arXiv},
      primaryClass={cs.LG},
      url={https://arxiv.org/abs/2502.02743}, 
}

@inproceedings{li2010contextual,
  title={A contextual-bandit approach to personalized news article recommendation},
  author={Li, Lihong and Chu, Wei and Langford, John and Schapire, Robert E},
  booktitle={Proceedings of the 19th international conference on World wide web},
  pages={661--670},
  year={2010}
}

@inproceedings{
he2021deberta,
title={DEBERTA: DECODING-ENHANCED BERT WITH DISENTANGLED ATTENTION},
author={Pengcheng He and Xiaodong Liu and Jianfeng Gao and Weizhu Chen},
booktitle={International Conference on Learning Representations},
year={2021},
url={https://openreview.net/forum?id=XPZIaotutsD}
}

@article{auer2002finite,
  title={Finite-time analysis of the multiarmed bandit problem},
  author={Auer, Peter and Cesa-Bianchi, Nicolo and Fischer, Paul},
  journal={Machine learning},
  volume={47},
  number={2},
  pages={235--256},
  year={2002},
  publisher={Springer}
}

@inproceedings{agrawal2012analysis,
  title={Analysis of thompson sampling for the multi-armed bandit problem},
  author={Agrawal, Shipra and Goyal, Navin},
  booktitle={Conference on learning theory},
  pages={39--1},
  year={2012},
  organization={JMLR Workshop and Conference Proceedings}
}

@article{valko2013finite,
  title={Finite-time analysis of kernelised contextual bandits},
  author={Valko, Michal and Korda, Nathaniel and Munos, R{\'e}mi and Flaounas, Ilias and Cristianini, Nelo},
  journal={arXiv preprint arXiv:1309.6869},
  year={2013}
}

@inproceedings{neu2020efficient,
  title={Efficient and robust algorithms for adversarial linear contextual bandits},
  author={Neu, Gergely and Olkhovskaya, Julia},
  booktitle={Conference on Learning Theory},
  pages={3049--3068},
  year={2020},
  organization={PMLR}
}

@article{foster2021statistical,
  title={The statistical complexity of interactive decision making},
  author={Foster, Dylan J and Kakade, Sham M and Qian, Jian and Rakhlin, Alexander},
  journal={arXiv preprint arXiv:2112.13487},
  year={2021}
}

@article{zhang2024efficient,
  title={Efficient contextual bandits with uninformed feedback graphs},
  author={Zhang, Mengxiao and Zhang, Yuheng and Luo, Haipeng and Mineiro, Paul},
  journal={arXiv preprint arXiv:2402.08127},
  year={2024}
}

@article{mannor2011bandits,
  title={From bandits to experts: On the value of side-observations},
  author={Mannor, Shie and Shamir, Ohad},
  journal={Advances in Neural Information Processing Systems},
  volume={24},
  year={2011}
}

@article{chen2021understanding,
  title={Understanding bandits with graph feedback},
  author={Chen, Houshuang and Li, Shuai and Zhang, Chihao and others},
  journal={Advances in Neural Information Processing Systems},
  volume={34},
  pages={24659--24669},
  year={2021}
}

@article{somerstep2025carrot,
  title={Carrot: A cost aware rate optimal router},
  author={Somerstep, Seamus and Polo, Felipe Maia and de Oliveira, Allysson Flavio Melo and Mangal, Prattyush and Silva, M{\'\i}rian and Bhardwaj, Onkar and Yurochkin, Mikhail and Maity, Subha},
  journal={arXiv preprint arXiv:2502.03261},
  year={2025}
}

@inproceedings{agrawal2013thompson,
  title={Thompson Sampling For Contextual Bandits With Linear Payoffs},
  author={Agrawal, Shipra and Goyal, Navin},
  booktitle={30th International Conference on Machine Learning (ICML)},
  year={2013}
}

@misc{open-llm-leaderboard-v2,
  author = {Clémentine Fourrier and Nathan Habib and Alina Lozovskaya and Konrad Szafer and Thomas Wolf},
  title = {Open LLM Leaderboard v2},
  year = {2024},
  publisher = {Hugging Face},
  howpublished = "\url{https://huggingface.co/spaces/open-llm-leaderboard/open_llm_leaderboard}",
}

@inproceedings{alon2015online,
  title={Online learning with feedback graphs: Beyond bandits},
  author={Alon, Noga and Cesa-Bianchi, Nicolo and Dekel, Ofer and Koren, Tomer},
  booktitle={Conference on Learning Theory},
  pages={23--35},
  year={2015},
  organization={PMLR}
}

@misc{alon2014nonstochasticmultiarmedbanditsgraphstructured,
      title={Nonstochastic Multi-Armed Bandits with Graph-Structured Feedback}, 
      author={Noga Alon and Nicolò Cesa-Bianchi and Claudio Gentile and Shie Mannor and Yishay Mansour and Ohad Shamir},
      year={2014},
      eprint={1409.8428},
      archivePrefix={arXiv},
      primaryClass={cs.LG},
      url={https://arxiv.org/abs/1409.8428}, 
}

@techreport{caro1979new,
  title={New results on the independence number},
  author={Caro, Yair},
  year={1979},
  institution={Technical Report, Tel-Aviv University}
}

@misc{wei1981lower,
  title={A lower bound on the stability number of a simple graph},
  author={Wei, Victor K},
  year={1981},
  publisher={Bell Laboratories Technical Memorandum Murray Hill, NJ, USA}
}

@article{wei2025learning,
  title={Learning to Route LLMs from Bandit Feedback: One Policy, Many Trade-offs},
  author={Wei, Wang and Yang, Tiankai and Chen, Hongjie and Zhao, Yue and Dernoncourt, Franck and Rossi, Ryan A and Eldardiry, Hoda},
  journal={arXiv preprint arXiv:2510.07429},
  year={2025}
}

@article{dai2024cost,
  title={Cost-effective online multi-llm selection with versatile reward models},
  author={Dai, Xiangxiang and Li, Jin and Liu, Xutong and Yu, Anqi and Lui, John},
  journal={arXiv preprint arXiv:2405.16587},
  year={2024}
}

@article{feng2024graphrouter,
  title={Graphrouter: A graph-based router for llm selections},
  author={Feng, Tao and Shen, Yanzhen and You, Jiaxuan},
  journal={arXiv preprint arXiv:2410.03834},
  year={2024}
}

@misc{allMiniLM_L6_v2_hf,
  title        = {{all-MiniLM-L6-v2}: Sentence Transformer Embedding Model},
  author       = {{Hugging Face}},
  year         = {2025},
  howpublished = {\url{https://huggingface.co/sentence-transformers/all-MiniLM-L6-v2}},
  note         = {Accessed: 2026-02-15},
}

@article{vovk1997competitive,
  title={Competitive on-line linear regression},
  author={Vovk, Volodya},
  journal={Advances in Neural Information Processing Systems},
  volume={10},
  year={1997}
}

@article{azoury2001relative,
  title={Relative loss bounds for on-line density estimation with the exponential family of distributions},
  author={Azoury, Katy S and Warmuth, Manfred K},
  journal={Machine learning},
  volume={43},
  number={3},
  pages={211--246},
  year={2001},
  publisher={Springer}
}

@misc{zhang2025contextuallinearbanditsdelay,
      title={Contextual Linear Bandits with Delay as Payoff}, 
      author={Mengxiao Zhang and Yingfei Wang and Haipeng Luo},
      year={2025},
      eprint={2502.12528},
      archivePrefix={arXiv},
      primaryClass={cs.LG},
      url={https://arxiv.org/abs/2502.12528}, 
}

@inproceedings{cheung2024leveraging,
  title={Leveraging (biased) information: Multi-armed bandits with offline data},
  author={Cheung, Wang Chi and Lyu, Lixing},
  booktitle={Forty-first International Conference on Machine Learning},
  year={2024}
}

@article{verma2023exploiting,
  title={Exploiting correlated auxiliary feedback in parameterized bandits},
  author={Verma, Arun and Dai, Zhongxiang and Shu, Yao and Low, Bryan Kian Hsiang},
  journal={Advances in Neural Information Processing Systems},
  volume={36},
  pages={4430--4451},
  year={2023}
}

@article{ji2025multi,
  title={Multi-armed bandits with machine learning-generated surrogate rewards},
  author={Ji, Wenlong and Pan, Yihan and Zhu, Ruihao and Lei, Lihua},
  journal={arXiv preprint arXiv:2506.16658},
  year={2025}
}

@book{sugiyama2012machine,
  title={Machine learning in non-stationary environments: Introduction to covariate shift adaptation},
  author={Sugiyama, Masashi and Kawanabe, Motoaki},
  year={2012},
  publisher={MIT press}
}

@article{shimodaira2000improving,
  title={Improving predictive inference under covariate shift by weighting the log-likelihood function},
  author={Shimodaira, Hidetoshi},
  journal={Journal of statistical planning and inference},
  volume={90},
  number={2},
  pages={227--244},
  year={2000},
  publisher={Elsevier}
}

@book{shalev2014understanding,
  title={Understanding machine learning: From theory to algorithms},
  author={Shalev-Shwartz, Shai and Ben-David, Shai},
  year={2014},
  publisher={Cambridge university press}
}

@book{vapnik2013nature,
  title={The nature of statistical learning theory},
  author={Vapnik, Vladimir},
  year={2013},
  publisher={Springer science \& business media}
}

@article{erven2011adaptive,
  title={Adaptive hedge},
  author={Erven, Tim and Koolen, Wouter M and Rooij, Steven and Gr{\"u}nwald, Peter},
  journal={Advances in Neural Information Processing Systems},
  volume={24},
  year={2011}
}

@article{freund1997decision,
  title={A decision-theoretic generalization of on-line learning and an application to boosting},
  author={Freund, Yoav and Schapire, Robert E},
  journal={Journal of computer and system sciences},
  volume={55},
  number={1},
  pages={119--139},
  year={1997},
  publisher={Elsevier}
}

@inproceedings{agarwal2017corralling,
  title={Corralling a band of bandit algorithms},
  author={Agarwal, Alekh and Luo, Haipeng and Neyshabur, Behnam and Schapire, Robert E},
  booktitle={Conference on Learning Theory},
  pages={12--38},
  year={2017},
  organization={PMLR}
}

@inproceedings{wang2025mixllm,
  title={Mixllm: Dynamic routing in mixed large language models},
  author={Wang, Xinyuan and Liu, Yanchi and Cheng, Wei and Zhao, Xujiang and Chen, Zhengzhang and Yu, Wenchao and Fu, Yanjie and Chen, Haifeng},
  booktitle={Proceedings of the 2025 Conference of the Nations of the Americas Chapter of the Association for Computational Linguistics: Human Language Technologies (Volume 1: Long Papers)},
  pages={10912--10922},
  year={2025}
}

%%%%%%%%%%%%%%%%%%%%%%%%%%%%%%%%%%%%%%%%%%%%%%%%%%%%%%%%%%%%
\newpage
\appendix
\section*{Appendix}
The appendix is organized as follows,

\begin{itemize}
    \item In Appendix~\ref{app:sec:related} we discuss additional related works
    \item In Appendix~\ref{app:baselines} we expand on the baselines that are used in our experiments.
    \item In Appendix~\ref{app:datasets}, we expand on the LLM routing datasets. 
    \item In Appendix~\ref{appendix:implementation}, implementation details are discussed. 
    \item In Appendix~\ref{appendix:experiments}, we present additional experiments and results.
    \item In Appendix~\ref{appendix:proofs}, we present the proofs for CABS-C and CABS-D algorithms.
    \item In Appendix~\ref{sec:discussion}, we discuss different  trade-offs, limitations and future works.
\end{itemize}
\section{Additional Related Works}
\label{app:sec:related}
In main body, we discussed the works most directly related to our approach and highlighted the key distinctions from our method. Here, we provide a more comprehensive overview of the multi-armed bandit and LLM routing literature, including detailed comparisons of regret bounds and theoretical guarantees (Table~\ref{tab:related_works}).

\subsection{Multi-Armed Bandits (MAB)}

Classical stochastic bandits study decision-making with partial feedback and no
context. UCB-type methods~\cite{auer2002finite} use optimism under uncertainty,
while Thompson sampling~\cite{agrawal2012analysis} samples actions from a
posterior distribution. Adversarial algorithms such as Exp3~\cite{auer2002nonstochastic}
provide guarantees without stochastic reward assumptions. These methods form the
basis for exploration--exploitation trade-offs but do not capture query-dependent
model selection.

Contextual bandits incorporate side information into the decision rule. Linear
methods such as LinUCB~\cite{li2010contextual,chu2011contextual} assume rewards
are linear in context features, while KernelUCB~\cite{valko2013finite} and
NeuralUCB~\cite{zhou2020neural} model nonlinear reward functions. More generally,
SquareCB~\cite{foster2020beyond} reduces contextual bandit learning to online
regression by using a reward predictor to construct exploration probabilities,
providing a flexible oracle-based template beyond linear models. These methods
are natural baselines for LLM routing, but typically observe feedback only for
the selected arm and do not explicitly exploit inter-arm relationships.

Bandits with graph feedback study settings where playing one arm reveals
additional observations from neighboring arms. Earlier work considers
non-contextual graph-feedback models~\cite{mannor2011bandits,chen2021understanding},
while SquareCB-G~\cite{zhang2023practical} and SquareCB-UG~\cite{zhang2024efficient}
extend this idea to contextual settings. SquareCB-G assumes graph-structured
observations of true rewards, whereas SquareCB-UG handles graphs that are not
known before action selection. Our setting differs because the graph is
context-dependent and available before action selection, while additional
observations are potentially noisy/misspecified surrogate rewards rather than true rewards.

Auxiliary-feedback bandits use external predictors or side observations to
improve learning. MIN-UCB~\cite{cheung2024leveraging}, OFUL-AF~\cite{verma2023exploiting},
and MLA-UCB~\cite{ji2025multi} are representative examples. These methods
typically use auxiliary information to improve reward estimation for the played
arm or to regularize learning. Our setting instead uses context-dependent graph
structure to decide which unplayed arms receive surrogate feedback, requiring
robustness to surrogate noise and bias.

Expert-aggregation methods provide another relevant perspective for our decoupled algorithm. Hedge and multiplicative-weight algorithms aggregate expert predictions with regret guarantees against the best fixed expert in hindsight~\cite{freund1997decision}, while adaptive variants such as AdaHedge tune the learning rate online based on observed losses~\cite{erven2011adaptive}. Related master algorithms, such as Corral~\cite{agarwal2017corralling}, combine multiple bandit algorithms and compete with the best base learner under partial feedback. CABS-D follows this meta-algorithmic view, but the experts correspond to different feedback mechanisms: standard bandit feedback and correlation-aware surrogate graph feedback.

Table~\ref{tab:related_works} summarizes representative bandit settings,
feedback models, and regret rates. We report worst-case order bounds, hiding
polylogarithmic factors in $\widetilde O(\cdot)$, under standard assumptions;
some algorithms admit sharper gap-dependent bounds, which we omit for brevity.

\begin{table*}[ht]
\centering
\caption{Representative bandit settings, feedback models, and regret bounds.}
\label{tab:related_works}
\resizebox{\textwidth}{!}{%
\begin{tabular}{|l|l|l|l|}
\hline
\textbf{Algorithm(s)} & \textbf{Classification} & \textbf{Feedback} & \textbf{Regret Bound} \\
\hline
\multicolumn{4}{|c|}{\textbf{Non-Contextual Setting}} \\
\hline
\multirow{2}{*}{UCB1~\cite{auer2002finite}} 
  & \multirow{2}{*}{Stochastic Non-Contextual} 
  & \multirow{2}{*}{Partial} 
  &  \multirow{2}{*}{$O\!\left(\sqrt{KT\log T}\right)$}  \\
& & & \\
\hline
Thompson Sampling~\cite{agrawal2012analysis} & Stochastic Non-Contextual & Partial & $O\!\left( \left(\sum_{i:\mu_i<\mu^*}\frac{1}{\Delta_i^2}\right)^2 \log T\right)$ \\
& & & *\textit{problem-dependent bound} \\
\hline
Exp3~\cite{auer2002nonstochastic} & Adversarial Non-Contextual & Partial & $O(\sqrt{KT\log K})$ \\
\hline
E2D~\cite{foster2021statistical} & Stochastic Non-Contextual & Partial & $O(\sqrt{KT\log K})$ \\
 \hline
\multirow{2}{*}{ELP~\cite{mannor2011bandits}} 
  & \multirow{1}{*}{Graph Feedback} 
  & \multirow{2}{*}{Partial/Full} 
  & \multirow{2}{*}{$\widetilde{O}(\sqrt{\log(k) \sum_{t=1}^{T} \alpha_t })$} \\
&Stochastic Non-Contextual & &  \\
 \hline
\multirow{3}{*}{OSMD~\cite{chen2021understanding}} 
  & \multirow{1}{*}{Graph Feedback} 
  & \multirow{2}{*}{Partial/Full} 
  & \multirow{2}{*}{$O\!\left( \left(\delta^* \log(k) \right)^\frac{1}{3} T^{\frac{2}{3}}\right)$} \\
&Stochastic Non-Contextual & &  \\
& & & $\delta^*$ is packing number \\
\hline
\multicolumn{4}{|c|}{\textbf{Contextual Setting}} \\
\hline
LinUCB~\cite{li2010contextual} \& SquareCB~\cite{foster2020beyond}  & Stochastic Contextual (Linear) & Partial & $\widetilde O\!\left(\sqrt{ K d T }\right)$ \\
\hline
\multirow{3}{*}{KernelUCB~\cite{valko2013finite}} & \multirow{3}{*}{Stochastic Contextual (Non-Linear)} & \multirow{3}{*}{Partial} & \multirow{2}{*}{$\widetilde{O}(\widetilde{d}\sqrt{T} )$} \\
& & &  \\
& & & ($\tilde d$ be the effective dimension) \\
\hline
\multirow{3}{*}{NeuralUCB~\cite{zhou2020neural} } 
  & \multirow{3}{*}{Stochastic Contextual (Non-Linear)} 
  & \multirow{3}{*}{Partial} 
  & \multirow{2}{*}{$\widetilde{O}(\widetilde d \sqrt{T}) $} \\
& & &  \\
& & & ($\tilde d$ be the effective dimension) \\
\hline
\multirow{2}{*}{RealLinExp3~\cite{neu2020efficient}} & \multirow{2}{*}{Adversarial Contextual (Linear)} & \multirow{2}{*}{Partial} & \multirow{2}{*}{$\widetilde{O}(\sqrt{KdT})$} \\
& & &  \\
\hline
\multirow{3}{*}{Square-CB~\cite{zhang2023practical}} 
  & \multirow{1}{*}{Informed Graph Feedback } 
  & \multirow{3}{*}{Partial/Full} 
  & $\widetilde{O}(\sqrt{\alpha T Reg_{sq}})$ \\
& Stochastic Contextual (Non-Linear) & & ($Reg_{sq}$ = bound of regression oracle) \\
& & & ($\alpha$ = independence number) \\
\hline
\multirow{4}{*}{SquareCB-UG~\cite{zhang2024efficient}} 
  & \multirow{1}{*}{Uninformed Graph Feedback  } 
  & \multirow{4}{*}{Partial/Full} 
  & $\widetilde{O}(\sqrt{\alpha T \max\left(Reg_{sq}, Reg_{log}\right)})$ \\
& Stochastic Contextual (Non-Linear) & & ($Reg_{sq}$ = bound of regression oracle) \\
& & & ($Reg_{log}$ = bound of graph oracle) \\
& & & ($\alpha$ = independence number) \\
\hline
\end{tabular}
}%
\end{table*}

\subsection{LLM Routing}

Here we give a high level overview of selected LLM routing methods, 
\begin{itemize}[leftmargin=*, itemsep=2pt, topsep=2pt]
  \item \textbf{RouterBench}~\cite{hu2024routerbenchbenchmarkmultillmrouting} introduces a standardized benchmark/dataset and evaluation protocol for multi-LLM routers under explicit cost--performance trade-offs, enabling apples-to-apples comparison across routing strategies.
    \item \textbf{RouterDC}~\cite{chen2024routerdc} proposes a query-based router trained via dual contrastive learning, aligning query embeddings with top-performing LLM embeddings (and separating from poorly performing ones) using sample–LLM and sample–sample contrastive losses to improve routing stability and robustness.

  \item \textbf{CARROT / SPROUT}~\cite{somerstep2025carrot} propose a cost-aware router based on estimating model cost/performance, and introduce the SPROUT dataset to evaluate routing over a wide query distribution and modern model pools.
  \item \textbf{RouteLLM}~\cite{ong2024routellm} trains a router from preference supervision to decide between cheaper and stronger LLMs, targeting cost savings at fixed quality levels and demonstrating transfer across model pairs.
  \item \textbf{Eagle}~\cite{zhao2024eagle} proposes a training-free router that uses ELO-style ranking (global and local) to efficiently update routing decisions, aiming at scalable online serving without heavy retraining.
  \item \textbf{TensorOpera / TO-Router}~\cite{stripelis2024tensoropera} studies multi-model orchestration as a system problem and routes queries to expert LLMs to improve the throughput--cost--quality ``trilemma'' in practical deployments.
  \item \textbf{GraphRouter}~\cite{feng2024graphrouter} models query--task--LLM interactions as a heterogeneous graph and performs edge/attribute prediction to recommend an LLM, emphasizing generalization to new LLMs via relational inductive bias.
  \item \textbf{MetaLLM}~\cite{nguyen2024metallm} frames routing as a contextual bandit and learns an online selection policy from partial feedback to improve cost-efficiency on targeted tasks (e.g., classification / MCQ).
  \item \textbf{LLM Bandit}~\cite{li2025llmbanditcostefficientllm} formulates LLM selection as a bandit problem with preference-conditioned routing, allowing users to specify trade-offs at inference while learning under partial-feedback constraints.
  \item \textbf{MixLLM}~\cite{wang2025mixllm} extends contextual-bandit routing to dynamic objectives (quality/cost/latency) with continual learning and changing model pools, using lightweight predictors plus a meta decision module.
  \item \textbf{BaRP}~\cite{wei2025learning} trains routers to align with the deployment partial-feedback setting while conditioning on a user preference vector, enabling a single learned policy to realize multiple cost--quality trade-offs at test time.
  \item \textbf{LLM Routing with Dueling Feedback}~\cite{chiang2025llm} treats routing as contextual dueling bandits and learns from pairwise preferences between two candidate model outputs, using CCFT representations and FGTS.CDB for online learning.
\end{itemize}

\section{Baselines}
\label{app:baselines}
\begin{itemize}
    \item \textbf{LinUCB~\cite{chu2011contextual} / MetaLLM~\cite{nguyen2024metallm}}: the standard linear upper-confidence bound algorithm, which selects the arm with the highest optimistic estimate of the linear reward.
    \item \textbf{LinUCB~\cite{chu2011contextual} + Full Feedback}: an oracle variant of LinUCB that observes the full reward vector for all arms at every round, rather than only the reward of the chosen arm. This baseline is not implementable in practice but provides a strong upper bound on what is achievable with linear models.
    \item  \textbf{LinTS~\cite{agrawal2013thompson}}: Linear Thompson sampling with a Gaussian posterior over the regression parameters.
    \item \textbf{Epsilon Greedy}: a linear regression model updated online together with an $\epsilon$-greedy exploration policy.
    \item \textbf{NeuralUCB~\cite{zhou2020neural}}: a neural contextual bandit algorithm that uses a neural network reward predictor and an upper-confidence bonus based on the network's linearized features (e.g., via the last-layer/NTK approximation), selecting the arm with the highest optimistic predicted reward.
    \item \textbf{NeuralUCB~\cite{zhou2020neural} + Full-Feedback}: an oracle variant of NeuralUCB that observes the full reward vector for all arms at each round and updates the neural predictor using all arms' rewards, providing an upper bound on what is achievable with neural reward models under full information.
    \item \textbf{MixLLM~\cite{wang2025mixllm}}: an online LLM routing method that first trains per-LLM predictors for response quality and response length. During online learning, refined feedback from the selected LLM is used to update that LLM's predictors and uncertainty matrix. For binary user feedback, MixLLM instead trains a shared dynamic-feedback head using a REINFORCE-style policy-gradient update and adds its output to the routing score. Because its codebase and manually constructed domain groupings are not publicly available, we implement the described binary-feedback variant without manual domain grouping and adapt its objective to our accuracy–dollar-cost utility.
    \item \textbf{BaRP~\cite{wei2025learning}}: a preference-conditioned policy-gradient (REINFORCE) method for LLM routing. At each round, it samples one LLM, observes only that LLM's quality and cost, constructs a preference-weighted scalar reward, and updates the policy using a variance-reduced, entropy-regularized policy-gradient loss. Its experiments simulate this bandit-feedback process using static benchmark logs rather than continual deployment-time feedback. We adapt its preference variable and reward definition to match our cost regimes.

\end{itemize}

All online learning baselines use exactly the same feature representation as CABS-C/CABS-D and differ only in their exploration strategy and feedback model. We additionally compare against \emph{static} routers that do not adapt during test time:

\begin{itemize}
    \item \textbf{CARROT~\cite{somerstep2025carrot} (RoBERTa)}: we adapt the current SOTA CARROT router to our utility objective. Instead of predicting the probability of success, we fine-tune a RoBERTa encoder as a regressor to predict, for each (query, model) pair, both the correctness reward and the cost. At test time, the router selects the model with the largest predicted utility,
    $\hat{r}_{t, i_t} = \hat{y}_{t,i_t} - \rho  \cdot \hat{c}_{t, i_t}$.

    \item \textbf{KNN (RouterBench~\cite{hu2024routerbenchbenchmarkmultillmrouting})}: a non-parametric router operating on text
    embeddings. For each test query, we retrieve the $k=3$ nearest training queries (in embedding space) for each candidate model and estimate reward and cost by averaging over their observed outcomes. The chosen model maximizes the estimated utility $\hat{r}_{t, i_t}$.

    \item \textbf{Random Selection}: For each query, this method randomly selects an arm from the K possible arms. 
    
    \item Finally, we report the performance of individual LLMs used by the benchmarks as static baselines that always route to a single model. This highlights the benefit of routing compared to selecting a single fixed model.

\end{itemize}

\section{Dataset Details}
\label{app:datasets}

\begin{itemize}
    \item \textbf{RouterBench}~\cite{hu2024routerbenchbenchmarkmultillmrouting} contains roughly 35k text queries spanning 8 reasoning, coding, and commonsense datasets, with responses from 11 LLMs annotated using GPT-4 LLM-as-judge binary correctness and token-based cost. We randomly split the dataset into 70\% train, 10\% validation, and 20\% test.
    \item \textbf{SPROUT}~\cite{somerstep2025carrot} contains roughly 44k queries drawn from six reasoning datasets, with responses from 15 LLMs annotated with cost and numeric correctness scores using LLama 3.1 70B LLM-as-judge. We use the default train/val/test split.
    \item \textbf{Open LLM Leaderboard v2}~\cite{open-llm-leaderboard-v2} provides responses from 18 contemporary LLMs for routing evaluation. The scores are generated using exact match-based accuracy. We randomly split the dataset into 70\% train, 10\% validation, and 20\% test.
\end{itemize}

We also show the average performance of each LLMs in the three LLM Routing datasets (Table~\ref{tab:open_llm_lb_v2_avg_cost_perf}, Table~\ref{tab:sprout_avg_cost_perf} and Table~\ref{tab:routerbench_avg_cost_perf}). 

\begin{table}[h]
\centering
\caption{Average performance and cost for models in OpenLLM Leaderboard V2 dataset}
\label{tab:open_llm_lb_v2_avg_cost_perf}
\begin{tabular}{lcc}
\toprule
Model & Performance & Cost ($\times 10^{-4}$) \\
\midrule
Yi-34B-Chat & 0.421 & 6.375 \\
Nous-Hermes-2-Mixtral-8x7B-DPO & 0.401 & 4.926 \\
QwQ-32B-Preview & 0.552 & 8.904 \\
Qwen2-72B-Instruct & 0.559 & 6.678 \\
Qwen2.5-72B-Instruct & 0.557 & 8.904 \\
Qwen2.5-7B-Instruct & 0.423 & 2.226 \\
WizardLM-2-8x22B & 0.488 & 9.852 \\
deepseek-llm-67b-chat & 0.408 & 7.049 \\
gemma-2-27b-it & 0.466 & 6.132 \\
gemma-2-9b-it & 0.420 & 2.300 \\
gemma-2b-it & 0.191 & 0.767 \\
Llama-2-13b-chat-hf & 0.228 & 2.472 \\
Meta-Llama-3.1-70B-Instruct & 0.547 & 6.442 \\
Mistral-7B-Instruct-v0.1 & 0.258 & 1.642 \\
Mistral-7B-Instruct-v0.2 & 0.308 & 1.642 \\
Mistral-7B-Instruct-v0.3 & 0.335 & 1.642 \\
Mixtral-8x7B-Instruct-v0.1 & 0.383 & 4.927 \\
Llama-3.1-Nemotron-70B-Instruct-HF & 0.500 & 6.442 \\
\bottomrule
\end{tabular}
\end{table}

\begin{table}[h]
\centering
\caption{Average performance and cost for models in SPROUT dataset}
\label{tab:sprout_avg_cost_perf}
\begin{tabular}{lcc}
\toprule
Model & Performance & Cost ($\times 10^{-4}$) \\
\midrule
claude-3-5-sonnet-v1 & 0.829 & 76.614 \\
titan-text-premier-v1 & 0.581 & 5.659 \\
openai-gpt-4o & 0.845 & 49.355 \\
openai-gpt-4o-mini & 0.808 & 3.406 \\
granite-3-2b-instruct & 0.555 & 0.848 \\
granite-3-8b-instruct & 0.660 & 1.502 \\
llama-3-1-70b-instruct & 0.810 & 7.205 \\
llama-3-1-8b-instruct & 0.690 & 2.438 \\
llama-3-2-1b-instruct & 0.462 & 0.662 \\
llama-3-2-3b-instruct & 0.632 & 0.644 \\
llama-3-3-70b-instruct & 0.806 & 5.524 \\
llama-3-405b-instruct & 0.776 & 20.180 \\
mixtral-8x7b-instruct & 0.618 & 3.743 \\
\bottomrule
\end{tabular}
\end{table}

\begin{table}[h]
\centering
\caption{Average performance and cost for models in RouterBench dataset}
\label{tab:routerbench_avg_cost_perf}
\begin{tabular}{lcc}
\toprule
Model & Performance & Cost ($\times 10^{-4}$) \\
\midrule
gpt-3.5-turbo-1106 & 0.619 & 2.420 \\
claude-instant-v1 & 0.598 & 2.311 \\
claude-v1 & 0.630 & 21.233 \\
claude-v2 & 0.636 & 24.009 \\
gpt-4-1106-preview & 0.781 & 32.495 \\
meta\_llama-2-70b-chat & 0.328 & 2.016 \\
mistralai\_mixtral-8x7b-chat & 0.547 & 1.340 \\
zero-one-ai\_Yi-34B-Chat & 0.647 & 1.846 \\
WizardLM\_WizardLM-13B-V1.2 & 0.431 & 0.726 \\
meta\_code-llama-instruct-34b-chat & 0.202 & 1.714 \\
mistralai\_mistral-7b-chat & 0.306 & 0.456 \\
\bottomrule
\end{tabular}
\end{table}

\section{Implementation Details}
\label{appendix:implementation}

\begin{figure*}[ht]
    \centering
     \includegraphics[width=\linewidth]{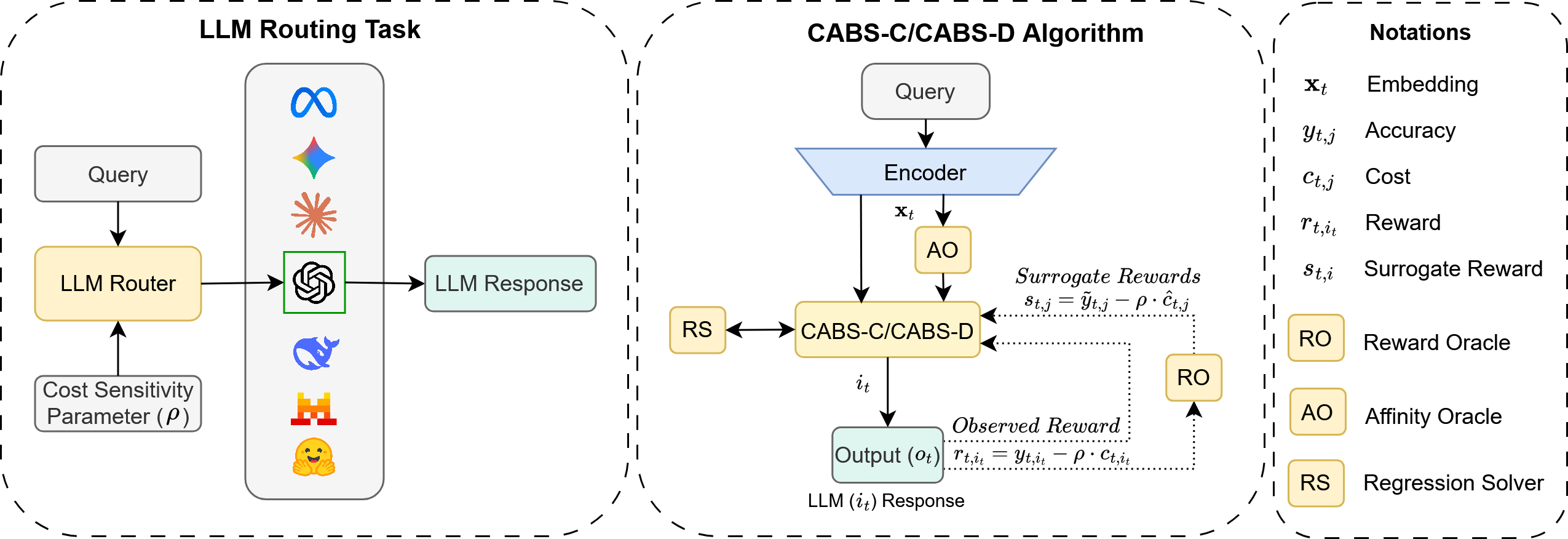}
    \caption{Overview of the LLM routing task and the proposed CABS-C/CABS-D online learning pipeline. The router selects an LLM under cost  parameter $\rho$, while \aoracle, \roracle, and \rsolver\, denote the affinity oracle, reward oracle, and regression solver used by the CABS-C/CABS-D algorithms.}
    \label{fig:task_overview}
\end{figure*}

All experiments are conducted on a machine equipped with 1xNVIDIA H100 GPU (80~GB). For all linear online learning methods, we set the regularization parameter to $\lambda = 10$ and $\gamma = 4$ whenever applicable. For neural online learning methods, we set regularization parameter to $\lambda = 0.001$ and the hidden dimension to 64, with learning rate of $0.001$. Reported results for all online learning methods are averaged over 10 independent trials to account for randomness in data ordering. We report 95\% confidence intervals using Student’s \(t\)-distribution, computed as $\bar{x} \pm (t_{0.975,n-1}*s) / \sqrt{n}$ across independent trials, where each trial’s utility is first averaged across the $p$ values within a cost regime, and $\bar{x}$ and $s$ denote the mean and sample standard deviation across the $n$ independent trials. For $\varepsilon$-greedy, we use the decay schedule, $\epsilon_t=\max\!\left(\epsilon_{\min}, \frac{\epsilon_0}{1+\text{decay}\cdot t}\right)$, tune $\epsilon_0 \in \{1.0, 0.5, 0.2, 0.1\}$ per dataset, and fix $\text{decay}=0.001$ and $\epsilon_{\min}=0.01$. For LinUCB/NeuralUCB with full feedback, we update using the ground-truth rewards of all arms at each round; since all rewards are observed, we set the exploration parameter to zero. 

Figure~\ref{fig:task_overview} provides a schematic overview of the LLM Routing task and CABS-C/CABS-D routing pipeline, including the interaction between the algorithm, Affinity Oracle (\aoracle), Reward Oracle (\roracle), and regression solver (\rsolver).

\paragraph{Affinity and Reward Oracles.}
We implement the \emph{Reward Oracle} (\textsf{RO}) and \emph{Affinity Oracle} (\textsf{AO}) using a lightweight multi-head network on top of query embeddings $\boldsymbol{x}_t$. Each query is encoded once, using mDeBERTaV3-base~\cite{he2021deberta} in our experiments, and passed to small MLP-style heads (using Tanh/GELU/ReLU nonlinearities with normalization/dropout). The framework is otherwise encoder-agnostic.

The \textsf{RO} predicts the accuracy and cost components separately and combines them into the cost-informed utility $r_{t,j}=y_{t,j}-\rho c_{t,j}$. For query embedding $\boldsymbol{x}_t$ and anchor arm $i_t$, we define the accuracy delta $\delta_{t,j}(i_t)=y_{t,j}-y_{t,i_t}$ for all $j\in[K]$. A correlation-aware propagation head predicts $\widehat{\delta}_{t,j}(i_t)=g(\boldsymbol{x}_t,i_t,y_{t,i_t})$, yielding surrogate accuracies $\tilde y_{t,j}=y_{t,i_t}+\widehat{\delta}_{t,j}(i_t)$ after observing $y_{t,i_t}$. In parallel, a cost head predicts $\widehat c_{t,j}=h(\boldsymbol{x}_t,j)$ (trained with MSE on $z$-scored costs), and the two are combined as $s_{t,j}=\tilde y_{t,j}-\rho \widehat c_{t,j}$. These surrogate utilities are then filtered to generate additional feedback for correlation-aware updates.

The \textsf{AO} is constructed from a context-only accuracy prediction head, which outputs per-arm scores $\bar y_{t,j}$ for all $j\in[K]$. Since estimating a true context-conditional arm correlation matrix would require repeated joint observations for the same query, we instead use the proxy $R_{t,(i,j)}=1-2|\bar y_{t,i}-\bar y_{t,j}|$, which is high for arms predicted to behave similarly and low, possibly negative, for arms predicted to differ sharply. This proxy defines the feedback graph used by our correlation-aware methods, although our algorithms and theory allow any context-dependent choice of $\boldsymbol{R}_t$. To enforce self loops, we set $R_{t,(i,i)} = -\infty$. 

In addition to the cost, propagation, and context-only accuracy heads, the network includes a bandit feature head that maps $\boldsymbol{x}_t$ to a 128-dimensional representation for downstream bandit algorithms. The heads are trained offline using Adam with learning rate $10^{-4}$, MSE for cost prediction, and $L_1$/Huber-style losses for accuracy prediction and propagation. Since we sweep large cost sensitivities up to $\rho=1000$, we feed the learner a bounded utility signal $\bar r_{t,i_t}=\mathrm{clip}(r_{t,i_t},0,1)$ for numerical stability. Finally, although \textsf{RO} can produce surrogate utilities for all graph-neighbor arms, we only incorporate pseudo-observations for a pruned subset selected by an IQR rule on predicted delta magnitudes, retaining arm $j$ only if $|\widehat{\delta}_{t,j}|>Q_3+\kappa\,\mathrm{IQR}$ with $\kappa=1.5$. This acts as a conservative sparsification guard without changing the underlying feedback model. 
% \vspace{-1em}

\paragraph{Efficient Computation of $\boldsymbol{p}_t$ (Line~8 of Algorithm~\ref{alg:CABS-C}):}
Line~8 of Algorithm~\ref{alg:CABS-C} requires solving for the sampling distribution $\boldsymbol{p}_t\in\Delta(K)$ that minimizes the objective $\phi(\boldsymbol{p};\hat{\boldsymbol{\theta}}_t,\boldsymbol{x}_t,\boldsymbol{A}_t)$. Rather than optimizing this objective in its original form, we follow the convex reformulation of~\cite{zhang2023practical}, which converts the problem into an equivalent convex program over $(\boldsymbol{p},z)$ that can be solved efficiently with standard solvers at each round, as stated in the following theorem.

\begin{theorem}[Theorem 3.6~\cite{zhang2023practical}]
\label{theorem:convex_method}
Solving $\arg\min_{\boldsymbol{p} \in \Delta(K)} \overline{\mathrm{dec}}_{\gamma}(\boldsymbol{p}; \hat{\boldsymbol{f}}, \boldsymbol{x}, \boldsymbol{G})$
is equivalent to solving the following convex optimization problem:
\begin{equation}
\begin{aligned}
\min_{\boldsymbol{p} \in \Delta(K),\, z} \quad & \boldsymbol{p}^\top \hat{\boldsymbol{f}} + z \\
\text{s.t.} \quad 
& \forall a \in [K] :
\frac{1}{\gamma}
\left\| \boldsymbol{p} - \boldsymbol{e}_a \right\|^2_{\mathrm{diag}(\boldsymbol{G}^\top \boldsymbol{p})^{-1}}
\le \hat{f}(\boldsymbol{x},a) + z, \\
& \boldsymbol{G}^\top \boldsymbol{p} \succ \boldsymbol{0},
\end{aligned}
\tag{5}
\end{equation}
where $\hat{\boldsymbol{f}}$ is shorthand for $\hat{\boldsymbol{f}}(\boldsymbol{x},\cdot)\in\mathbb{R}^K$,
$\boldsymbol{e}_a$ is the $a$-th standard basis vector, $\succ$ denotes element-wise strict positivity, and $z\in\mathbb{R}$ is an auxiliary slack variable.
\end{theorem}

We refer readers to~\cite{zhang2023practical} for the proof. Although Theorem~\ref{theorem:convex_method} is stated in terms of loss scores $\hat f(\boldsymbol{x}_t,j)$, our bandit models produce reward predictions $\hat r_{t,j}\in[0,1]$. We therefore apply the theorem with the transformation, $\hat f(\boldsymbol{x}_t,j)\;=\;1-\hat r_{t,j}$. In particular, in the linear case $\hat r_{t,j}=\boldsymbol{x}_t^\top \hat{\boldsymbol{\theta}}_{t,j}$, while in the neural case $\hat r_{t,j}=f_{\hat{\boldsymbol{\theta}}_t}(\boldsymbol{x}_t,j)$ denotes the prediction of the corresponding neural reward model. Thus, $\hat{\boldsymbol f}(\boldsymbol{x}_t,\cdot)\in\mathbb{R}^K$ is obtained component-wise from $\hat{\boldsymbol r}_t(\boldsymbol{x}_t,\cdot)\in\mathbb{R}^K$ via $\hat{\boldsymbol f}_t=\mathbf{1}-\hat{\boldsymbol r}_t$. Since $\boldsymbol{p}\in\Delta(K)$ implies $\boldsymbol{p}^\top \mathbf{1}=1$, this substitution only changes the objective by an additive constant (absorbed into the slack variable $z$) and leaves the optimizer $\boldsymbol{p}_t$ unchanged. We use Theorem~\ref{theorem:convex_method} to compute $\boldsymbol{p}_t=\arg\min_{\boldsymbol{p}\in\Delta(K)}\phi(\boldsymbol{p};\hat{\boldsymbol{\theta}}_{t},\boldsymbol{x}_t,\boldsymbol{A}_t)$ in Algorithm~\ref{alg:CABS-C}.

\FloatBarrier
\section{Additional Graphs and Results}
\label{appendix:experiments}

\begin{table}[ht]
\centering
\caption{
Average utility (higher is better) and average regret (lower is better) across cost regimes.
Values are reported with 95\% confidence intervals over 10 repeated trials.
}
\label{tab:cabs_ablation1}

\renewcommand{\arraystretch}{1.12}
\setlength{\tabcolsep}{6pt}
\footnotesize
\newcommand{\ci}[1]{{\scriptsize $\pm #1$}}

% ---------------- SPROUT ----------------
\begin{subtable}{\textwidth}
\centering
\caption{SPROUT}
\label{tab:cabs_ablation_ci_sprout}

\begin{tabular}{l c c c c}
\toprule
\multirow{2}{*}{\textbf{Regime (Metric)}}
& \multicolumn{2}{c}{\textbf{CABS-C}}
& \multicolumn{2}{c}{\textbf{CABS-D}} \\
\cmidrule(lr){2-3} \cmidrule(lr){4-5}
& Linear & Neural & Linear & Neural \\
\midrule

Low (Utility)
& 0.7233 \ci{0.0216}
& 0.7558 \ci{0.0019}
& \textbf{0.7715} \ci{0.0013}
& 0.7432 \ci{0.0015} \\

Medium (Utility)
& 0.6015 \ci{0.0114}
& 0.6439 \ci{0.0025}
& \textbf{0.6632} \ci{0.0022}
& 0.6467 \ci{0.0012} \\

High (Utility)
& 0.5407 \ci{0.0376}
& 0.4007 \ci{0.0037}
& \textbf{0.6265} \ci{0.0030}
& 0.5876 \ci{0.0010} \\

\midrule

Low (Regret)
& 0.2264 \ci{0.0216}
& 0.1939 \ci{0.0019}
& \textbf{0.1782} \ci{0.0013}
& 0.2065 \ci{0.0015} \\

Medium (Regret)
& 0.2979 \ci{0.0114}
& 0.2555 \ci{0.0025}
& \textbf{0.2362} \ci{0.0022}
& 0.2527 \ci{0.0012} \\

High (Regret)
& 0.3249 \ci{0.0376}
& 0.4649 \ci{0.0037}
& \textbf{0.2391} \ci{0.0030}
& 0.2780 \ci{0.0010} \\

\bottomrule
\end{tabular}
\end{subtable}

\vspace{0.8em}

% ---------------- RouterBench ----------------
\begin{subtable}{\textwidth}
\centering
\caption{RouterBench}
\label{tab:cabs_ablation_ci_routerbench}

\begin{tabular}{l c c c c}
\toprule
\multirow{2}{*}{\textbf{Regime (Metric)}}
& \multicolumn{2}{c}{\textbf{CABS-C}}
& \multicolumn{2}{c}{\textbf{CABS-D}} \\
\cmidrule(lr){2-3} \cmidrule(lr){4-5}
& Linear & Neural & Linear & Neural \\
\midrule

Low (Utility)
& 0.6350 \ci{0.0096}
& 0.6421 \ci{0.0028}
& \textbf{0.6754} \ci{0.0011}
& 0.6638 \ci{0.0023} \\

Medium (Utility)
& 0.5308 \ci{0.0061}
& 0.5551 \ci{0.0061}
& \textbf{0.5817} \ci{0.0011}
& 0.5759 \ci{0.0026} \\

High (Utility)
& 0.4962 \ci{0.0171}
& 0.5110 \ci{0.0026}
& \textbf{0.5466} \ci{0.0029}
& 0.5372 \ci{0.0013} \\

\midrule

Low (Regret)
& 0.2540 \ci{0.0096}
& 0.2469 \ci{0.0028}
& \textbf{0.2136} \ci{0.0011}
& 0.2252 \ci{0.0023} \\

Medium (Regret)
& 0.3021 \ci{0.0061}
& 0.2778 \ci{0.0061}
& \textbf{0.2512} \ci{0.0011}
& 0.2570 \ci{0.0026} \\

High (Regret)
& 0.2946 \ci{0.0171}
& 0.2798 \ci{0.0026}
& \textbf{0.2442} \ci{0.0029}
& 0.2536 \ci{0.0013} \\

\bottomrule
\end{tabular}
\end{subtable}

\vspace{0.8em}

% ---------------- Open LLM ----------------
\begin{subtable}{\textwidth}
\centering
\caption{Open LLM Leaderboard V2}
\label{tab:cabs_ablation_ci_openllm}

\begin{tabular}{l c c c c}
\toprule
\multirow{2}{*}{\textbf{Regime (Metric)}}
& \multicolumn{2}{c}{\textbf{CABS-C}}
& \multicolumn{2}{c}{\textbf{CABS-D}} \\
\cmidrule(lr){2-3} \cmidrule(lr){4-5}
& Linear & Neural & Linear & Neural \\
\midrule

Low (Utility)
& 0.5047 \ci{0.0102}
& 0.5147 \ci{0.0056}
& \textbf{0.5328} \ci{0.0023}
& 0.5284 \ci{0.0096} \\

Medium (Utility)
& 0.2886 \ci{0.0107}
& 0.3197 \ci{0.0064}
& 0.3229 \ci{0.0019}
& \textbf{0.3289} \ci{0.0039} \\

High (Utility)
& 0.1753 \ci{0.0089}
& 0.1776 \ci{0.0092}
& 0.1951 \ci{0.0023}
& \textbf{0.1956} \ci{0.0097} \\

\midrule

Low (Regret)
& 0.4272 \ci{0.0102}
& 0.4172 \ci{0.0056}
& \textbf{0.3991} \ci{0.0023}
& 0.4035 \ci{0.0096} \\

Medium (Regret)
& 0.5313 \ci{0.0107}
& 0.5002 \ci{0.0064}
& 0.4970 \ci{0.0019}
& \textbf{0.4910} \ci{0.0039} \\

High (Regret)
& 0.5471 \ci{0.0089}
& 0.5448 \ci{0.0092}
& 0.5273 \ci{0.0023}
& \textbf{0.5268} \ci{0.0097} \\

\bottomrule
\end{tabular}
\end{subtable}

\end{table}

\subsection{Reward Mixing vs Prediction Mixing (CABS-C vs CABS-D)}
\label{appendix:rq1}

Table~\ref{tab:cabs_ablation1} reports the full CABS-C vs.\ CABS-D ablation
results with confidence intervals. Panels (a) and (b) provide confidence
intervals for the SPROUT~\cite{somerstep2025carrot} and
RouterBench~\cite{hu2024routerbenchbenchmarkmultillmrouting} results reported
in Table~\ref{tab:cabs_ablation} and discussed in
Section~\ref{sec:reward_mixing_vs_prediction_mixing}.

Panel (c) additionally reports the same ablation on Open LLM Leaderboard
V2~\cite{open-llm-leaderboard-v2}. The results exhibit a similar overall
pattern to SPROUT and RouterBench: CABS-D generally achieves comparable or
higher utility than CABS-C and tends to obtain lower regret, although the
differences between the variants are smaller on this dataset. These additional
results provide further evidence that the behavior observed in the main text
extends beyond the two primary evaluation datasets.

\begin{table*}[ht]
\centering
\caption{
Utilities across cost regimes on Open LLM LB v2, RouterBench, and SPROUT using
mDeBERTaV3-base~\cite{he2021deberta} embeddings. We report mean utility with
95\% confidence intervals over 10 independent trials. FB stands for Full-Feedback.
}
\label{tab:utility_results_mdeberta_all}

\small
\renewcommand{\arraystretch}{1.08}
\setlength{\tabcolsep}{5pt}
\newcommand{\ci}[1]{{\scriptsize $\pm #1$}}

% =========================================================
% Open LLM LB v2
% =========================================================
\begin{subtable}{0.98\textwidth}
\centering
\caption{Open LLM LB v2}
\label{tab:utility_results_mdeberta_openllm}

\begin{tabular}{@{}l|ccc@{}}
\hline
\textbf{Method}
& Low & Medium & High \\
\hline

Random Selection
& 0.3822 \ci{0.0014}
& 0.2926 \ci{0.0013}
& 0.1889 \ci{0.0019} \\

Epsilon-Greedy
& 0.4544 \ci{0.0046}
& 0.2543 \ci{0.0019}
& 0.1153 \ci{0.0030} \\

LinTS~\cite{agrawal2013thompson}
& 0.5017 \ci{0.0072}
& 0.2973 \ci{0.0044}
& 0.1822 \ci{0.0058} \\

LinUCB~\cite{chu2011contextual}
& 0.4876 \ci{0.0084}
& 0.2846 \ci{0.0116}
& 0.1817 \ci{0.0105} \\

NeuralUCB~\cite{zhou2020neural}
& 0.4493 \ci{0.0104}
& 0.2471 \ci{0.0172}
& 0.1558 \ci{0.0133} \\

BARP~\cite{wei2025learning}
& 0.4819 \ci{0.0034}
& 0.2932 \ci{0.0042}
& 0.1875 \ci{0.0029} \\

MixLLM~\cite{wang2025mixllm}
& 0.4334 \ci{0.0011}
& 0.2897 \ci{0.0016}
& 0.1842 \ci{0.0015} \\

\rowcolor{lightgray}
CABS-D (Linear)
& \textbf{0.5328} \ci{0.0023}
& 0.3229 \ci{0.0019}
& 0.1951 \ci{0.0023} \\

\rowcolor{lightgray}
CABS-D (Neural)
& 0.5284 \ci{0.0096}
& \textbf{0.3289} \ci{0.0039}
& \textbf{0.1956} \ci{0.0097} \\

\hline

LinUCB + FB
& 0.5630 \ci{0.0009}
& 0.3600 \ci{0.0007}
& 0.2345 \ci{0.0011} \\

NeuralUCB + FB
& 0.5454 \ci{0.0028}
& 0.3620 \ci{0.0019}
& 0.2388 \ci{0.0014} \\

\hline
\end{tabular}
\end{subtable}

\vspace{1em}

% =========================================================
% RouterBench
% =========================================================
\begin{subtable}{0.98\textwidth}
\centering
\caption{RouterBench}
\label{tab:utility_results_mdeberta_routerbench}

\begin{tabular}{@{}l|ccc@{}}
\hline
\textbf{Method}
& Low & Medium & High \\
\hline

Random Selection
& 0.4509 \ci{0.0010}
& 0.3431 \ci{0.0012}
& 0.2899 \ci{0.0014} \\

Epsilon-Greedy
& 0.6057 \ci{0.0012}
& 0.5169 \ci{0.0010}
& 0.4751 \ci{0.0019} \\

LinTS~\cite{agrawal2013thompson}
& 0.6514 \ci{0.0048}
& 0.5725 \ci{0.0022}
& 0.5381 \ci{0.0028} \\

LinUCB~\cite{chu2011contextual}
& 0.6276 \ci{0.0118}
& 0.5437 \ci{0.0057}
& 0.5082 \ci{0.0066} \\

NeuralUCB~\cite{zhou2020neural}
& 0.6344 \ci{0.0180}
& 0.5418 \ci{0.0119}
& 0.4950 \ci{0.0070} \\

BARP~\cite{wei2025learning}
& 0.6350 \ci{0.0023}
& 0.5683 \ci{0.0028}
& 0.4251 \ci{0.0160} \\

MixLLM~\cite{wang2025mixllm}
& 0.6364 \ci{0.0015}
& 0.5588 \ci{0.0014}
& 0.4991 \ci{0.0011} \\

\rowcolor{lightgray}
CABS-D (Linear)
& \textbf{0.6754} \ci{0.0011}
& \textbf{0.5817} \ci{0.0011}
& \textbf{0.5466} \ci{0.0029} \\

\rowcolor{lightgray}
CABS-D (Neural)
& 0.6638 \ci{0.0023}
& 0.5759 \ci{0.0026}
& 0.5372 \ci{0.0013} \\

\hline

LinUCB + FB
& 0.6838 \ci{0.0002}
& 0.6024 \ci{0.0004}
& 0.5660 \ci{0.0004} \\

NeuralUCB + FB
& 0.6704 \ci{0.0011}
& 0.6050 \ci{0.0004}
& 0.5588 \ci{0.0014} \\

\hline
\end{tabular}
\end{subtable}

\vspace{1em}

% =========================================================
% SPROUT
% =========================================================
\begin{subtable}{0.98\textwidth}
\centering
\caption{SPROUT}
\label{tab:utility_results_mdeberta_sprout}

\begin{tabular}{@{}l|ccc@{}}
\hline
\textbf{Method}
& Low & Medium & High \\
\hline

Random Selection
& 0.5872 \ci{0.0013}
& 0.4472 \ci{0.0015}
& 0.3863 \ci{0.0016} \\

Epsilon-Greedy
& 0.7194 \ci{0.0013}
& 0.6138 \ci{0.0013}
& 0.5655 \ci{0.0019} \\

LinTS~\cite{agrawal2013thompson}
& 0.7628 \ci{0.0044}
& 0.6548 \ci{0.0033}
& 0.6200 \ci{0.0043} \\

LinUCB~\cite{chu2011contextual}
& 0.7579 \ci{0.0045}
& 0.6561 \ci{0.0050}
& 0.6110 \ci{0.0053} \\

NeuralUCB~\cite{zhou2020neural}
& 0.7086 \ci{0.0048}
& 0.6266 \ci{0.0069}
& 0.5501 \ci{0.0086} \\

BARP~\cite{wei2025learning}
& 0.7378 \ci{0.0028}
& 0.6370 \ci{0.0026}
& 0.6246 \ci{0.0023} \\

MixLLM~\cite{wang2025mixllm}
& 0.7265 \ci{0.0018}
& 0.6433 \ci{0.0016}
& 0.5968 \ci{0.0015} \\

\rowcolor{lightgray}
CABS-D (Linear)
& \textbf{0.7715} \ci{0.0013}
& \textbf{0.6632} \ci{0.0022}
& \textbf{0.6265} \ci{0.0030} \\

\rowcolor{lightgray}
CABS-D (Neural)
& 0.7432 \ci{0.0015}
& 0.6467 \ci{0.0012}
& 0.5876 \ci{0.0010} \\

\hline

LinUCB + FB
& 0.7957 \ci{0.0004}
& 0.7070 \ci{0.0002}
& 0.6643 \ci{0.0003} \\

NeuralUCB + FB
& 0.7732 \ci{0.0007}
& 0.6842 \ci{0.0006}
& 0.6246 \ci{0.0010} \\

\hline
\end{tabular}
\end{subtable}

\end{table*}

\FloatBarrier

\subsection{Utility across cost regimes}
\label{app:cis_and_encoder_robustness}

Table~\ref{tab:utility_results_mdeberta_all} reports the full utility results
with 95\% confidence intervals over 10 independent trials for the online learning
methods evaluated in Section~\ref{sec:utility_cost_regimes}. Panels (a), (b), and (c)
corresponds to Open LLM Leaderboard v2, RouterBench, and SPROUT, respectively.
These intervals quantify the uncertainty associated with the mean utilities
reported in the main paper. Since CARROT and RouterBench are offline routing
baselines evaluated as single runs, they are omitted from the confidence-interval
tables.

\begin{figure*}[!htb]
    \centering
    (a) RouterBench
    \includegraphics[width=\linewidth]{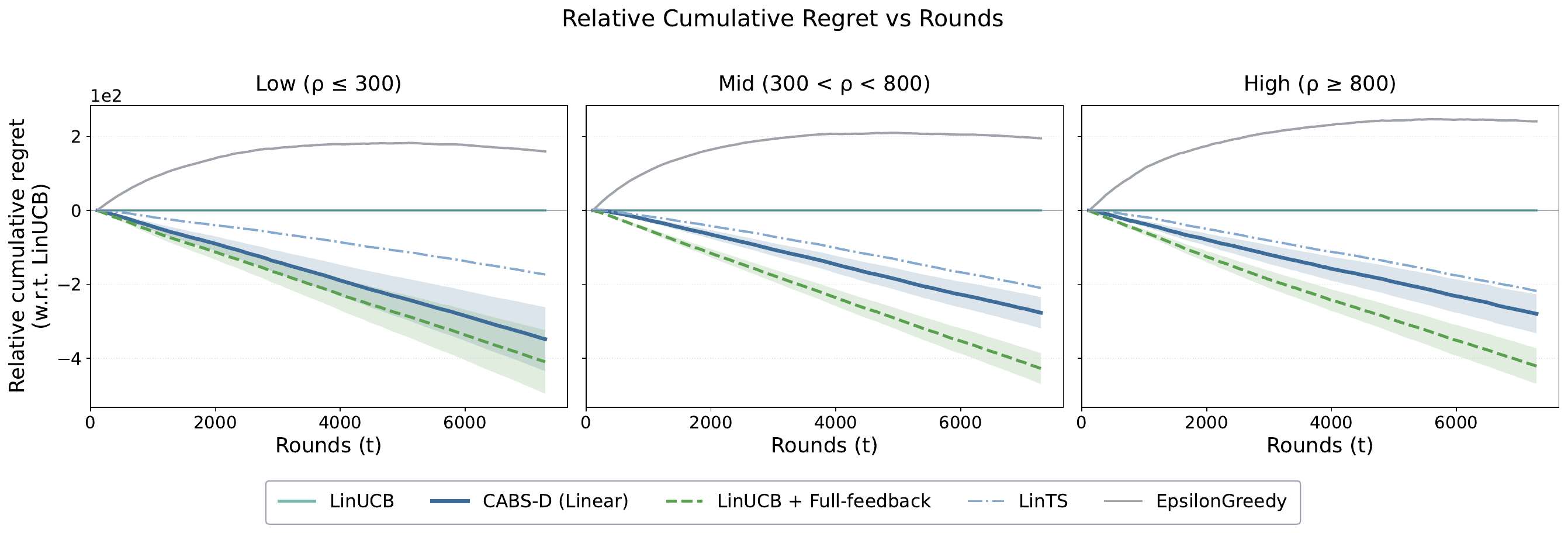}
    
    (b) SPROUT
    \includegraphics[width=\linewidth]{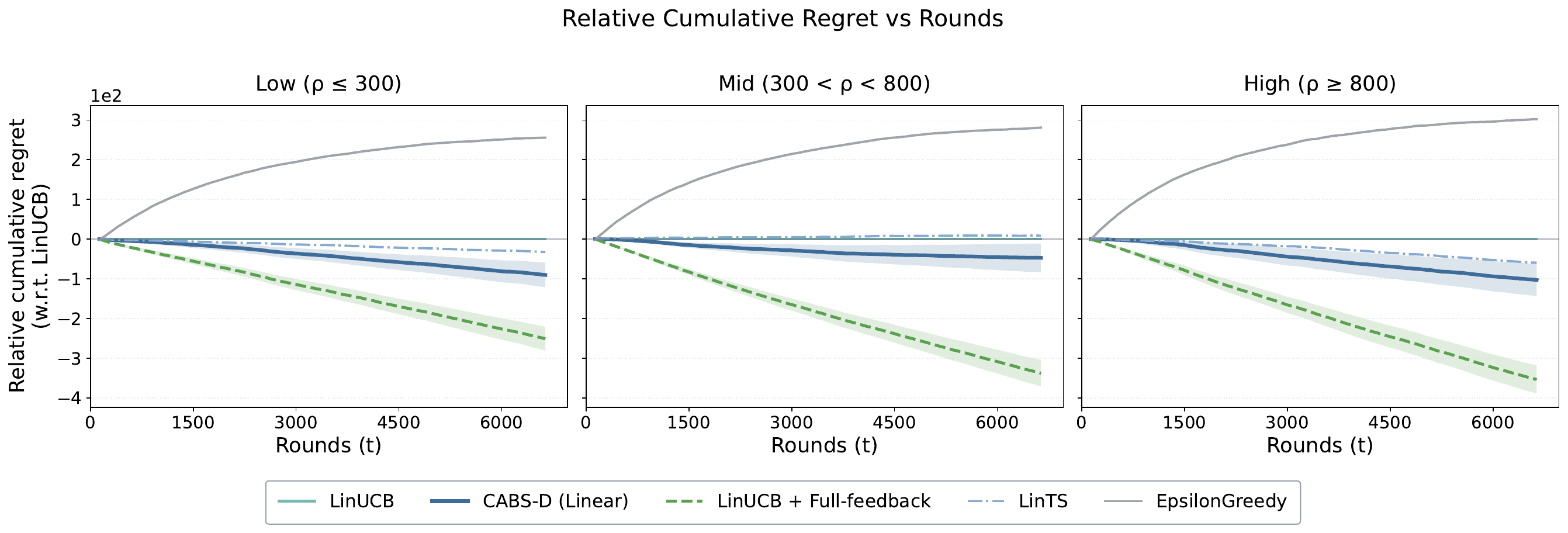}
    \caption{
        Relative Cumulative regrets with respect to LinUCB for other linear online learning methods on RouterBench (top) and SPROUT (bottom) dataset.
    }
    \label{fig:regret_curves1}
\end{figure*}

\subsection{Learning Dynamics \& Accuracy vs Cost Trade-offs}

Figure~\ref{fig:regret_curves1} extends the learning-dynamics analysis of Section~\ref{sec:learning_dynamics} to RouterBench and SPROUT, reporting relative cumulative regret normalized by LinUCB across low ($\rho \le 300$), mid ($300<\rho<800$), and high ($\rho \ge 800$) cost regimes. Consistent with the Open LLM Leaderboard v2 results in the main text, CABS-D achieves uniformly lower relative cumulative regret than standard bandit baselines in most regimes, and the separation typically appears early in the horizon, indicating faster learning from surrogate feedback rather than solely an asymptotic gain. The trends are especially stable in the mid-cost regime, where the objective most strongly couples accuracy and cost and thus benefits from exploiting inter-arm correlations. As expected, full-feedback variants remain an optimistic upper bound and provide a reference for the achievable regret reduction under complete feedback. Further, Figure~\ref{fig:regret_curves3} shows relative cumulative regret graphs for Neural variants where we observe similar trends as that of linear variants. 

Figure~\ref{fig:acc_vs_cost_neural_results} provides the analogous accuracy--cost Pareto frontiers for neural online-learning variants. The qualitative behavior mirrors the linear case: correlation-aware surrogate methods shift the frontier outward, achieving higher accuracy at comparable cost or reducing cost for a fixed accuracy, with the largest advantages occurring in regimes where cost sensitivity meaningfully constrains the policy.

\begin{figure}[h]
    \centering

    \includegraphics[width=\linewidth]{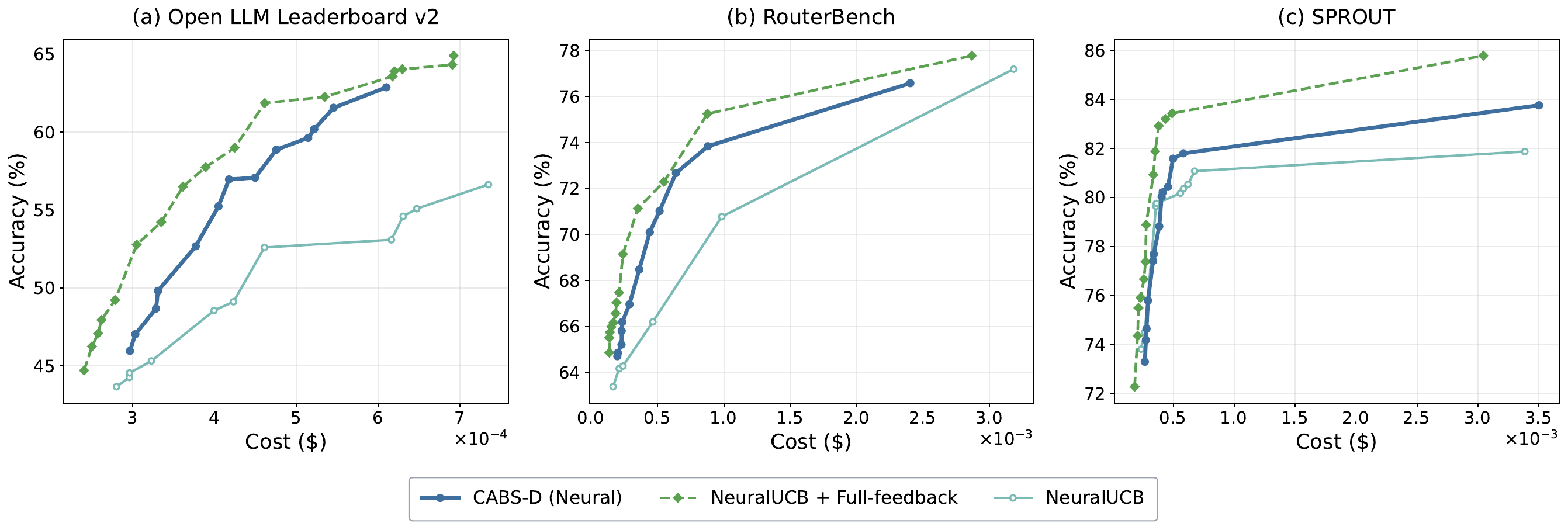}

    \caption{
    Accuracy vs cost Pareto Frontier curves across datasets on Neural Online Learning models.
    }
    \label{fig:acc_vs_cost_neural_results}
\end{figure}

\begin{figure}[h]
    \centering
    % % -------- Row 1: Linear --------
        (a) Open LLM Leaderboard v2
    \includegraphics[width=\linewidth]{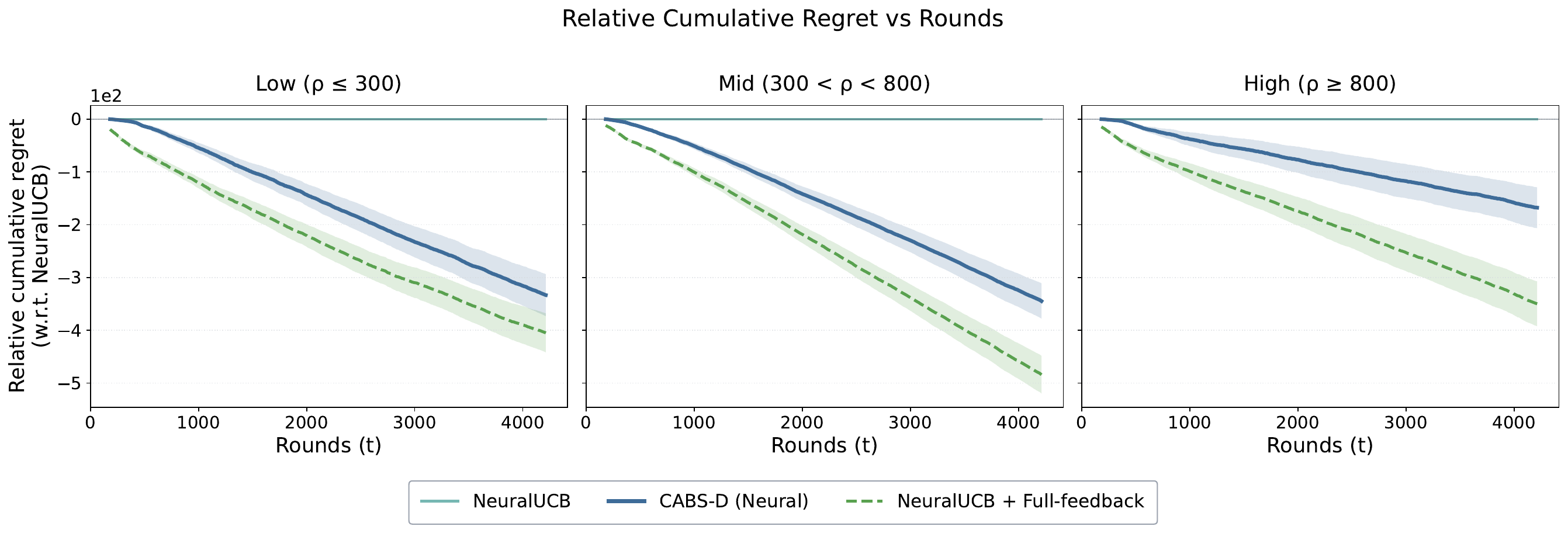}

        (b) RouterBench
    \includegraphics[width=\linewidth]{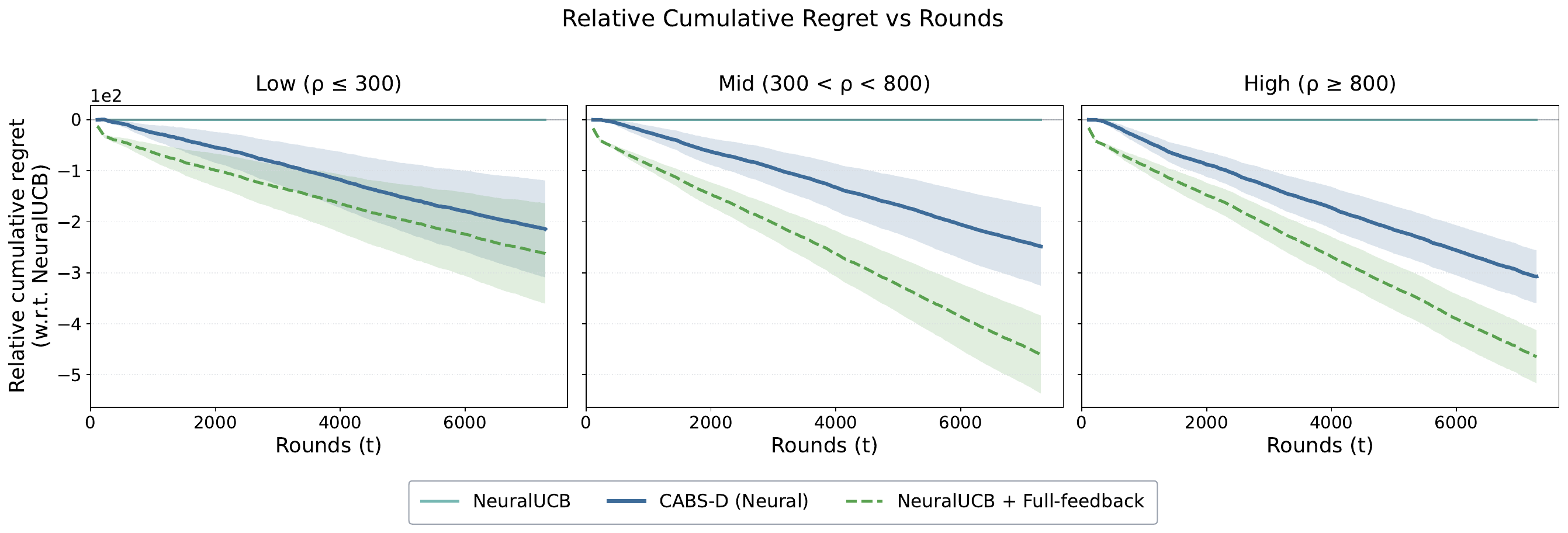}

        (c) SPROUT
    \includegraphics[width=\linewidth]{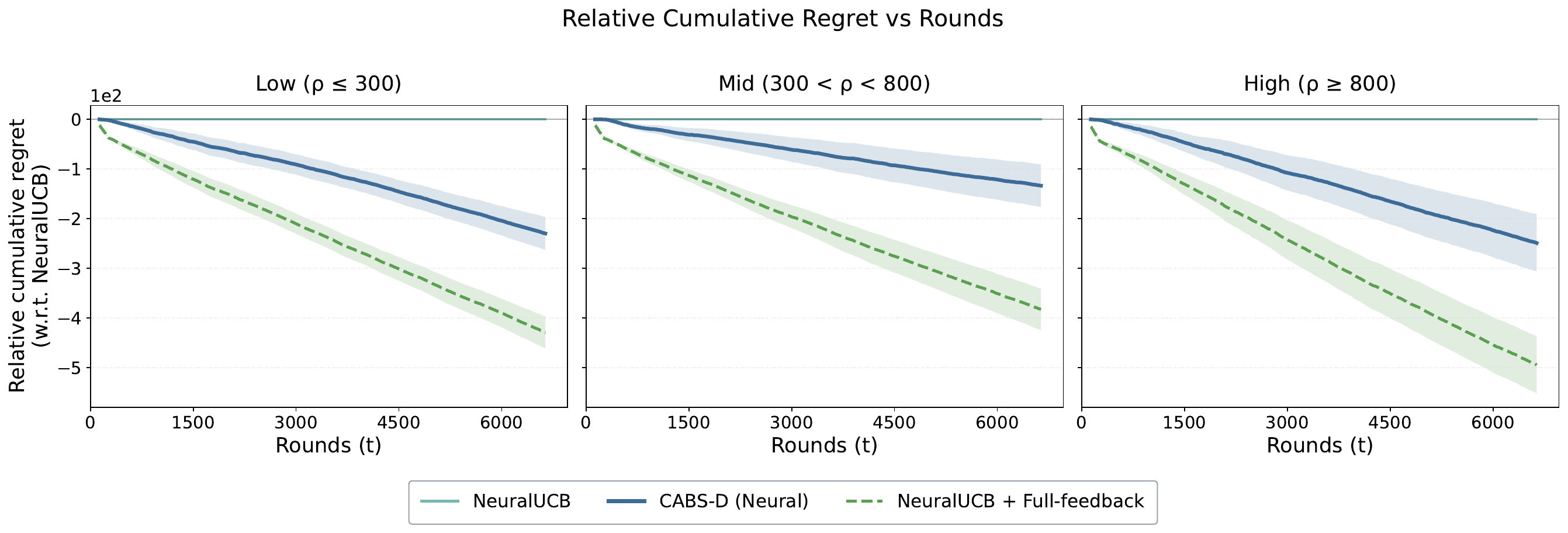}

    \caption{
        Relative Cumulative regrets with respect to NeuralUCB for other neural online learning methods on Open LLM Leaderboard V2 dataset (top), RouterBench (middle) and SPROUT (bottom).
    }
    \label{fig:regret_curves3}
\end{figure}

\FloatBarrier

\subsection{Ablations}

We investigate three aspects of our framework: which surrogate observations should be selected, how robust the proposed methods are to inaccurate surrogate rewards, and how sensitive their performance is to the choice of text encoder.

\subsubsection{Surrogate Selection Strategies}

We compare three surrogate-selection strategies for the linear variants of CABS-C and CABS-D on RouterBench: (i) the IQR-based selection rule, which selects surrogate feedbacks that are farther from the bandit feedback (see Appendix~\ref{appendix:implementation}); (ii) Random selection, which independently includes each non-chosen arm with probability $0.5$; and (iii) All Surrogates, which includes surrogate feedback from all non-chosen arms.

As shown in Table~\ref{tab:surrogate_selection}, CABS-D consistently benefits from the IQR-based selection strategy. Across the low-, medium-, and high-cost regimes, IQR selection improves utility over Random 50\% by $0.0300$, $0.0346$, and $0.0376$, respectively, and over All Surrogates by $0.0322$, $0.0371$, and $0.0280$. In contrast, CABS-C is relatively insensitive to the selection strategy, with only small variations in utility across the three alternatives. These results suggest that CABS-D benefits from selecting informative surrogate feedback rather than simply incorporating more surrogate observations. This behavior is consistent with its design: the quality of the selected surrogate observations directly affects the surrogate-aware expert, while CABS-D can adaptively balance that expert against the reward-only learner.

\begin{table}[ht]
\centering
\caption{
Utility on RouterBench under different surrogate-selection strategies. We report mean utility with 95\% confidence intervals over 10 independent trials.
}
\label{tab:surrogate_selection}

\small
\renewcommand{\arraystretch}{1.08}
\setlength{\tabcolsep}{7pt}
\newcommand{\ci}[1]{{\scriptsize $\pm #1$}}

\begin{tabular}{llccc}
\toprule
\textbf{Regime}
& \textbf{Method}
& \textbf{IQR}
& \textbf{Random 50\%}
& \textbf{All Surrogates} \\
\midrule

\multirow{2}{*}{Low}
& CABS-C
& 0.6350 \ci{0.0096}
& 0.6400 \ci{0.0010}
& 0.6357 \ci{0.0030} \\

& CABS-D
& 0.6754 \ci{0.0011}
& 0.6454 \ci{0.0014}
& 0.6432 \ci{0.0015} \\
\midrule

\multirow{2}{*}{Medium}
& CABS-C
& 0.5308 \ci{0.0061}
& 0.5368 \ci{0.0062}
& 0.5338 \ci{0.0022} \\

& CABS-D
& 0.5817 \ci{0.0011}
& 0.5471 \ci{0.0027}
& 0.5446 \ci{0.0027} \\
\midrule

\multirow{2}{*}{High}
& CABS-C
& 0.4962 \ci{0.0171}
& 0.4925 \ci{0.0030}
& 0.4889 \ci{0.0027} \\

& CABS-D
& 0.5466 \ci{0.0029}
& 0.5090 \ci{0.0027}
& 0.5186 \ci{0.0026} \\

\bottomrule
\end{tabular}
\end{table}

\subsubsection{Robustness to Surrogate Misspecification}

We conduct a controlled surrogate-degradation study by perturbing only the surrogate rewards while leaving the observed rewards unchanged. Motivated by the bias-and-noise decomposition in Assumption~4.4, we perturb a selected surrogate reward according to $\widetilde{r}_{t,i} = \widehat{r}_{t,i} + b_i + \xi_{t,i},$ where $\xi_{t,i} \sim \mathcal{N}(0,\sigma^2).$ 

We use an adversarial arm-dependent bias that decreases the surrogate rewards of the top-performing half of the arms and increases those of the bottom-performing half, thereby directly corrupting the surrogate ranking.

Table~\ref{tab:surrogate_misspecification_ci} shows that stronger surrogate corruption generally increases mean regret, with small non-monotonic variations at intermediate settings. CABS-D achieves the lowest mean regret across all tested settings and cost regimes, remaining below the reward-only LinUCB reference. CABS-C has higher mean regret than LinUCB throughout the medium- and high-cost regimes; in the low-cost regime, its initial advantage disappears at corruption levels \((0.10,0.10)\) and above. These results support the empirical advantage of the decoupled approach under the tested surrogate perturbations.

\begin{table*}[t]
\centering
\caption{
Average regret (lower is better) on RouterBench under different
surrogate-corruption settings $(b,\sigma)$ for linear models.
Values are means with 95\% confidence intervals across 10 independent trials.
}
\label{tab:surrogate_misspecification_ci}

\small
\renewcommand{\arraystretch}{1.08}
\setlength{\tabcolsep}{5pt}
\newcommand{\ci}[1]{{\tiny $\pm #1$}}

\begin{tabular}{llccccc}
\toprule
\multirow{2}{*}{\textbf{Regime}}
& \multirow{2}{*}{\textbf{Method}}
& \multicolumn{5}{c}{\textbf{Corruption setting $(b,\sigma)$}} \\
\cmidrule(lr){3-7}
& & $(0,0)$ & $(0.05,0.05)$ & $(0.10,0.10)$
& $(0.20,0.20)$ & $(0.20,0.50)$ \\
\midrule

\multirow{3}{*}{Low}
& LinUCB
& 0.2614 \ci{0.0118}
& 0.2614 \ci{0.0118}
& 0.2614 \ci{0.0118}
& 0.2614 \ci{0.0118}
& 0.2614 \ci{0.0118} \\

& CABS-C
& 0.2540 \ci{0.0096}
& 0.2597 \ci{0.0027}
& 0.2643 \ci{0.0071}
& 0.2720 \ci{0.0075}
& 0.2759 \ci{0.0027} \\

& CABS-D
& \textbf{0.2136} \ci{0.0011}
& \textbf{0.2297} \ci{0.0195}
& \textbf{0.2284} \ci{0.0145}
& \textbf{0.2327} \ci{0.0184}
& \textbf{0.2455} \ci{0.0059} \\

\midrule

\multirow{3}{*}{Medium}
& LinUCB
& 0.2892 \ci{0.0057}
& 0.2892 \ci{0.0057}
& 0.2892 \ci{0.0057}
& 0.2892 \ci{0.0057}
& 0.2892 \ci{0.0057} \\

& CABS-C
& 0.3021 \ci{0.0061}
& 0.3151 \ci{0.0069}
& 0.3143 \ci{0.0060}
& 0.3224 \ci{0.0069}
& 0.3251 \ci{0.0063} \\

& CABS-D
& \textbf{0.2512} \ci{0.0011}
& \textbf{0.2542} \ci{0.0125}
& \textbf{0.2545} \ci{0.0128}
& \textbf{0.2594} \ci{0.0176}
& \textbf{0.2664} \ci{0.0168} \\

\midrule

\multirow{3}{*}{High}
& LinUCB
& 0.2826 \ci{0.0066}
& 0.2826 \ci{0.0066}
& 0.2826 \ci{0.0066}
& 0.2826 \ci{0.0066}
& 0.2826 \ci{0.0066} \\

& CABS-C
& 0.2946 \ci{0.0171}
& 0.3131 \ci{0.0097}
& 0.3131 \ci{0.0097}
& 0.3127 \ci{0.0093}
& 0.3231 \ci{0.0097} \\

& CABS-D
& \textbf{0.2442} \ci{0.0029}
& \textbf{0.2500} \ci{0.0123}
& \textbf{0.2501} \ci{0.0124}
& \textbf{0.2551} \ci{0.0124}
& \textbf{0.2690} \ci{0.0124} \\

\bottomrule
\end{tabular}
\end{table*}

\FloatBarrier
\subsubsection{Encoder Robustness}
\label{appendix:encoder_robustness}

 To assess sensitivity to the choice of text representation, we repeat the utility experiments using all-MiniLM-L6-v2~\cite{allMiniLM_L6_v2_hf} in place of mDeBERTaV3-base~\cite{he2021deberta}. Table~\ref{tab:utility_results_all_minilm_l6_v2_all} reports the corresponding results on Open LLM Leaderboard v2, RouterBench, and SPROUT.

 The overall performance trends remain broadly consistent with those in the main experiments. CABS-D remains competitive with the online bandit baselines across datasets and cost regimes, indicating that the main conclusions are not tied to a particular encoder. Absolute utility values shift under the alternative encoder, with the largest changes appearing on RouterBench in the high-cost regime. The full-feedback baselines shift in the same direction, suggesting that part of this difference is attributable to the representations induced by the encoder rather than instability in the online learning procedure.
 
\begin{table*}[t]
\centering
\caption{
Utilities across cost regimes on Open LLM LB v2, RouterBench, and SPROUT using
all-MiniLM-L6-v2~\cite{allMiniLM_L6_v2_hf} embeddings. We report mean utility
with 95\% confidence intervals. FB stands for Full-Feedback.
}
\label{tab:utility_results_all_minilm_l6_v2_all}

\small
\renewcommand{\arraystretch}{1.08}
\setlength{\tabcolsep}{5pt}
\newcommand{\ci}[1]{{\scriptsize $\pm #1$}}

% =========================================================
% Open LLM LB v2
% =========================================================
\begin{subtable}{0.98\textwidth}
\centering
\caption{Open LLM LB v2}
\label{tab:utility_results_all_minilm_l6_v2_openllm}

\begin{tabular}{@{}l|ccc@{}}
\hline
\textbf{Method}
& Low & Medium & High \\
\hline

Random Selection
& 0.3822 \ci{0.0014}
& 0.2926 \ci{0.0013}
& 0.1889 \ci{0.0019} \\

Epsilon-Greedy
& 0.4331 \ci{0.0047}
& 0.2403 \ci{0.0038}
& 0.1189 \ci{0.0022} \\

LinTS~\cite{agrawal2013thompson}
& 0.4662 \ci{0.0131}
& 0.2718 \ci{0.0046}
& 0.1908 \ci{0.0028} \\

LinUCB~\cite{chu2011contextual}
& 0.4522 \ci{0.0126}
& 0.2510 \ci{0.0062}
& 0.1903 \ci{0.0051} \\

NeuralUCB~\cite{zhou2020neural}
& 0.4323 \ci{0.0098}
& 0.2417 \ci{0.0085}
& 0.1759 \ci{0.0107} \\

BARP~\cite{wei2025learning}
& 0.4811 \ci{0.0024}
& 0.2765 \ci{0.0019}
& 0.1825 \ci{0.0025} \\

MixLLM~\cite{wang2025mixllm}
& 0.4829 \ci{0.0021}
& 0.2688 \ci{0.0039}
& 0.1806 \ci{0.0067} \\

\rowcolor{lightgray}
CABS-D (Linear)
& \textbf{0.4990} \ci{0.0030}
& 0.2823 \ci{0.0021}
& 0.1925 \ci{0.0031} \\

\rowcolor{lightgray}
CABS-D (Neural)
& 0.4920 \ci{0.0045}
& \textbf{0.2989} \ci{0.0030}
& \textbf{0.1937} \ci{0.0025} \\

\hline

LinUCB + FB
& 0.5293 \ci{0.0007}
& 0.3280 \ci{0.0008}
& 0.2293 \ci{0.0013} \\

NeuralUCB + FB
& 0.5213 \ci{0.0014}
& 0.3421 \ci{0.0032}
& 0.2374 \ci{0.0020} \\

\hline
\end{tabular}
\end{subtable}

\vspace{1em}

% =========================================================
% RouterBench
% =========================================================
\begin{subtable}{0.98\textwidth}
\centering
\caption{RouterBench}
\label{tab:utility_results_all_minilm_l6_v2_routerbench}

\begin{tabular}{@{}l|ccc@{}}
\hline
\textbf{Method}
& Low & Medium & High \\
\hline

Random Selection
& 0.4509 \ci{0.0010}
& 0.3431 \ci{0.0012}
& 0.2899 \ci{0.0014} \\

Epsilon-Greedy
& 0.5795 \ci{0.0024}
& 0.4803 \ci{0.0032}
& 0.3431 \ci{0.0026} \\

LinTS~\cite{agrawal2013thompson}
& 0.6280 \ci{0.0033}
& 0.4994 \ci{0.0062}
& 0.3653 \ci{0.0093} \\

LinUCB~\cite{chu2011contextual}
& 0.6104 \ci{0.0089}
& 0.4991 \ci{0.0092}
& 0.3560 \ci{0.0199} \\

NeuralUCB~\cite{zhou2020neural}
& 0.6011 \ci{0.0105}
& 0.5030 \ci{0.0100}
& 0.3763 \ci{0.0071} \\

BARP~\cite{wei2025learning}
& 0.6331 \ci{0.0011}
& 0.5277 \ci{0.0024}
& 0.4004 \ci{0.0097} \\

MixLLM~\cite{wang2025mixllm}
& 0.6171 \ci{0.0029}
& \textbf{0.5296} \ci{0.0032}
& 0.4072 \ci{0.0021} \\

\rowcolor{lightgray}
CABS-D (Linear)
& \textbf{0.6539} \ci{0.0025}
& 0.5266 \ci{0.0053}
& 0.3981 \ci{0.0076} \\

\rowcolor{lightgray}
CABS-D (Neural)
& 0.6421 \ci{0.0012}
& 0.5127 \ci{0.0029}
& \textbf{0.4125} \ci{0.0062} \\

\hline

LinUCB + FB
& 0.6701 \ci{0.0005}
& 0.5412 \ci{0.0002}
& 0.4244 \ci{0.0011} \\

NeuralUCB + FB
& 0.6613 \ci{0.0007}
& 0.5527 \ci{0.0005}
& 0.4634 \ci{0.0027} \\

\hline
\end{tabular}
\end{subtable}

\vspace{1em}

% =========================================================
% SPROUT
% =========================================================
\begin{subtable}{0.98\textwidth}
\centering
\caption{SPROUT}
\label{tab:utility_results_all_minilm_l6_v2_sprout}

\begin{tabular}{@{}l|ccc@{}}
\hline
\textbf{Method}
& Low & Medium & High \\
\hline

Random Selection
& 0.5872 \ci{0.0013}
& 0.4472 \ci{0.0015}
& 0.3863 \ci{0.0016} \\

Epsilon-Greedy
& 0.7235 \ci{0.0018}
& 0.6230 \ci{0.0019}
& 0.5654 \ci{0.0038} \\

LinTS~\cite{agrawal2013thompson}
& 0.7641 \ci{0.0073}
& 0.6574 \ci{0.0035}
& 0.6167 \ci{0.0046} \\

LinUCB~\cite{chu2011contextual}
& 0.7612 \ci{0.0083}
& 0.6580 \ci{0.0034}
& 0.6031 \ci{0.0066} \\

NeuralUCB~\cite{zhou2020neural}
& 0.7131 \ci{0.0091}
& 0.5031 \ci{0.0069}
& 0.4693 \ci{0.0045} \\

BARP~\cite{wei2025learning}
& 0.7450 \ci{0.0040}
& 0.6399 \ci{0.0038}
& 0.6027 \ci{0.0015} \\

MixLLM~\cite{wang2025mixllm}
& 0.7487 \ci{0.0159}
& 0.6141 \ci{0.0047}
& 0.5596 \ci{0.0032} \\

\rowcolor{lightgray}
CABS-D (Linear)
& \textbf{0.7738} \ci{0.0017}
& \textbf{0.6730} \ci{0.0013}
& \textbf{0.6226} \ci{0.0016} \\

\rowcolor{lightgray}
CABS-D (Neural)
& 0.7494 \ci{0.0017}
& 0.5932 \ci{0.0025}
& 0.4925 \ci{0.0030} \\

\hline

LinUCB + FB
& 0.8024 \ci{0.0004}
& 0.7105 \ci{0.0003}
& 0.6672 \ci{0.0004} \\

NeuralUCB + FB
& 0.7789 \ci{0.0010}
& 0.6282 \ci{0.0011}
& 0.5467 \ci{0.0008} \\

\hline
\end{tabular}
\end{subtable}

\end{table*}

\FloatBarrier
\newpage
\section{The Technical Proofs}
\label{appendix:proofs}

For clarity, we decompose the analysis of regret of proposed algorithms into three components, which progressively build toward our final guarantee. In particular, for the ease of reading, we split our algorithmic analysis into three parts as follows:
\begin{enumerate}
    \item \textbf{Baseline: Contextual Bandit} In Subsection~\ref{appendix:standard-contextual-bandit}, we first present a classical contextual bandit with regret bound $\widetilde O\!\left(\sqrt{KT\mathrm{Reg}_{\mathrm{Sq}}(T)}\right)$. In particular for the linear case, we get $\widetilde O\!\left(\sqrt{dKT \log \left(\frac{T}{d}\right)}\right)$.
    \item \textbf{Correlation-Aware Bandits with Surrogates — Coupled (CABS-C).}  
    We next analyze our proposed graph-feedback algorithm with noisy surrogate rewards in Subsection~\ref{appendix:CABS-C}. When feedback graph $G_t$ has independence number no greater than $\alpha$, the regret satisfies
    \[
    \widetilde O\!\left(
    \sqrt{\alpha T\, \mathrm{Reg}_{\mathrm{Sq}}(T)}
    +
    \varepsilon_{n} \sqrt{\alpha K T}
    \right).
    \]
    In particular, for the linear case, when exactly $m$ surrogate rewards are observed each round, this becomes
    \[
    \widetilde O\!\left(
    \sqrt{d\frac{K}{m+1}T\log\left(\frac{T}{d}\right)}
    +
    \varepsilon_{n} \sqrt{\frac{K^2}{m+1}T}
    \right).
    \]
    demonstrating a significant improvement by scale of roughly $1/m$ as long as the noise from ML-predicted rewards $\varepsilon_n$ is well-bounded. 
    \item \textbf{Correlation-Aware Bandits with Surrogates — Decoupled (CABS-D)} Finally, in Subsection~\ref{appendix:CABS-D} we establish the regret bound for the proposed prediction-mixing algorithm, a meta-bandit algorithm with two arms ((1) and (2) above) with regret bounded by min of them, 
    \[
        \widetilde O\!\left( \min \left\{\sqrt{K T \mathrm{Reg}_{\mathrm{Sq}}(T)}, \sqrt{\alpha T\, \mathrm{Reg}_{\mathrm{Sq}}(T)}
    +
    \varepsilon_{n} \sqrt{\alpha K T} \right\}\right)
    \]
    For the linear case, when exactly $m$ surrogate rewards are observed each round, this becomes
    \[
    \widetilde O\!\left( \min \left\{\sqrt{dKT \log \left(\frac{T}{d}\right) }, \sqrt{d\frac{K}{m+1}T\log \left(\frac{T}{d}\right)}+ \varepsilon_{n}\sqrt{\frac{K^2}{m+1}T}\right\}\right)
    \]
\end{enumerate}

Ultimately, these results demonstrate that leveraging predicted surrogate rewards offers a powerful framework for LLM routing, yielding substantial regret reductions. Specifically, our approach effectively interpolates between the full-information and bandit settings, with the performance gains smoothly scaling based on the quality of the surrogate rewards and correlation among LLMs. This finding may be of independent interest beyond the specific setting considered here.

\subsection{Contextual bandits Regret Bounds}\label{appendix:standard-contextual-bandit}
\begin{algorithm}[h]
\caption{SquareCB (Reward Form; adapted from \cite{foster2020beyond})}
\label{alg:squarecb-reward}
\begin{algorithmic}[1]
\STATE \textbf{Parameters:} learning rate $\gamma>0$, exploration parameter $\mu>0$
\STATE \textbf{Input:} online regression oracle \rsolver\, producing predictors $\hat f_t$
\FOR{$t=1,\dots,T$}
    \STATE Receive context $\boldsymbol{x}_t$ \label{line:squarecb-predict-start}
    \STATE \textit{// Compute oracle's reward predictions}
    \FOR{each arm $i \in [K]$}
        \STATE $\hat r_{t,i} \gets \hat f_t(\boldsymbol{x}_t,i)$
    \ENDFOR
    \STATE $b_t \gets \arg\max_{i\in[K]} \hat r_{t,i}$
    \FOR{each arm $i \in [K]\setminus\{b_t\}$}
        \STATE $p_{t,i} \gets \frac{1}{\mu + \gamma\big(\hat r_{t,b_t}-\hat r_{t,i}\big)}$
    \ENDFOR
    \STATE $p_{t,b_t} \gets 1 - \sum_{i \neq b_t} p_{t,i}$ \label{line:squarecb-predict-end}
    \STATE Sample $i_t \sim \boldsymbol{p}_t$ and observe reward $r_{t,i_t}$
    \STATE Update \rsolver\ with example $\big\{(\boldsymbol{x}_t,r_{t,i_t})\big\}$
\ENDFOR
\end{algorithmic}
\end{algorithm}

We begin with the SquareCB~\cite{foster2020beyond} algorithm, which reduces contextual bandits to online regression. We state the modified algorithm that works on rewards instead of losses in Algorithm~\ref{alg:squarecb-reward}. We provide its regret guarantee for completeness and refer the reader to \cite{foster2020beyond} for the proof.

\begin{theorem}[Theorem 1~\cite{foster2020beyond}]
\label{theorem:squarecb}
Suppose Assumption~\ref{assumption:regression_oracle} holds. Then for any $\delta > 0$, by setting $\mu = K$ and $\gamma = \sqrt{KT / \left(\mathrm{Reg}_{\mathrm{Sq}}(T) + \log(2\delta^{-1})\right)}$, 
\textsc{SquareCB} guarantees that with probability at least $1 - \delta$,
\begin{equation}
\mathrm{Reg}_{\mathrm{CB}}(T)
\le
4 \sqrt{KT \cdot \mathrm{Reg}_{\mathrm{Sq}}(T)}
+
8 \sqrt{KT \log(2\delta^{-1})}
= \widetilde O\!\left(\sqrt{KT\mathrm{Reg}_{\mathrm{Sq}}(T)}\right).
\tag{8}
\end{equation}
\end{theorem}

\begin{corollary}[Linear Regression Oracle]
\label{cor:squarecb_linear}
Suppose the regression oracle is linear, i.e.,
$\hat{f}(\boldsymbol{x}_t,i_t)=\boldsymbol{x}_t^\top\boldsymbol{\theta}_{i_t}$, 
where $\boldsymbol{x}_t\in\mathbb{R}^d$ with $\|\boldsymbol{x}\|_2\le1$ and 
$\boldsymbol{\theta}_{i_t}\in\mathbb{R}^d$ with $\|\boldsymbol{\theta}_{i_t}\|_2\le1$. 
Choosing $\textsc{SqAlg}$ as the Vovk--Azoury--Warmuth forecaster~\cite{vovk1997competitive, azoury2001relative}, which satisfies
$\mathrm{Reg}_{\mathrm{Sq}}(T)\le d\log \left(\frac{T}{d}\right)$,
Theorem~\ref{theorem:squarecb} implies
\[
\mathrm{Reg}_{\mathrm{CB}}(T)
\le
4\sqrt{dKT\log\left(\frac{T}{d}\right)}
+
8\sqrt{KT\log(2\delta^{-1})}
=
\widetilde O\!\left(\sqrt{dKT\log\left(\frac{T}{d}\right)}\right).
\]
\end{corollary}

\subsection{Correlation Aware Bandits with Surrogates Coupled (CABS-C)}\label{appendix:CABS-C}

In this section, we present algorithmic analysis for Contextual Bandits with Noisy Graph Feedback with regret bound depending on a noise factor, $\widetilde O\!\left(\sqrt{\alpha T\, \mathrm{Reg}_{\mathrm{Sq}}(T)}+\varepsilon_{n} \sqrt{\alpha K T}\right)$. In our main paper, we present this as Correlation Aware Bandits with Surrogates Coupled (CABS-C), Algorithm~\ref{alg:CABS-C}. We include here all intermediate lemmas and derivations to ensure completeness and reproducibility of the analysis.

\subsubsection{High Level Proof Sketch}
\label{app:CABS-C-sketch}
The proof of Theorem~\ref{theorem:CABS-C} proceeds in three main steps. First, we decompose the total expected regret into an algorithmic optimization term (sometimes referred to as the \textit{non-bias term}), the cumulative regret of the regression oracle $\mathrm{Reg_{Sq}}(T)$, and an expected bias term $\mathbb{E}[\mathcal{B}_T]$ that accounts for the misspecification of the surrogate rewards. Next, we further decompose this expected bias term into two distinct components: a regression-propagation error and an observation noise term. We continue by systematically bounding each piece of this decomposition: the optimization term is bounded using a minimax argument over the action probabilities, the regression-propagation error is bounded using the Cauchy-Schwarz inequality and properties of harmonic sums, and the observation noise is bounded using the Azuma-Hoeffding inequality applied to the bounded martingale difference sequence. Combining these individual bounds and tuning the learning rate parameter $\gamma=\sqrt{\alpha T/Reg_{Sq}(T)}$ to balance the components yields the final stated regret bound of $\tilde{O}(\sqrt{\alpha T~Reg_{Sq}(T)}+\epsilon_n\sqrt{\alpha KT})$.

\paragraph{Dependence between the affinity and reward oracles.}
The affinity oracle ($\aoracle$) and reward oracle ($\roracle$) may share the same offline predictor, and our analysis does not impose a separate statistical independence assumption between them. Their effects enter the analysis through different quantities: the affinity oracle through the induced feedback graph and its independence number $\alpha$, and the reward oracle through the residual-noise bound $\varepsilon_n$. Thus, sharing an underlying predictor does not itself invalidate the guarantee; the analysis only requires that the stated feedback-graph and surrogate-noise conditions hold.

\paragraph{Regret Bound for Algorithm~\ref{alg:CABS-C} ({\sffamily{CABS-C}})}

\paragraph{Notation}
Let \(n_{t,j}\) be the number of times arm \(j\) has been observed up to (and including) round \(t\), and \(\mathcal T_t(j)=\{\tau\le t: \mathcal A_{\tau,(i_\tau, j)}=1\}\). Finally denote \(\Delta_{\tau,j}:=
|\boldsymbol{x}_\tau^{\top} \hat{\boldsymbol{\theta}}_{\tau,k} - \boldsymbol{x}_\tau^{\top}\boldsymbol{\theta}^*_{k} |
% f^\star(x_\tau,j)-\hat f_\tau(x_\tau,j)
\).\\\\
To prove Theorem~\ref{theorem:CABS-C}, we first introduce the following auxiliary lemmas regarding the regret and bias decomposition. We continue by introducing lemmas to individually bound each term of the decomposition. Finally, we combine the results of all these lemmas in the final proof of Theorem~\ref{theorem:CABS-C}.

\begin{lemma}[Regret Decomposition]\label{lem:cabsc-regret-decomp}
For any $\gamma > 0$, we may decompose the regret of Algorithm~\ref{alg:CABS-C} as the following:
\begin{equation}\label{eq:dec-decomp}
\mathbb{E}[\mathrm{Reg}_T] \;\le\;
T \cdot \phi(\boldsymbol{p}; \hat{\boldsymbol{\theta}}_{t}, \boldsymbol{x}_t, \boldsymbol{A}_t)
\;+\; \frac{\gamma}{4}\,\mathrm{Reg}_{\mathrm{Sq}}(T)
\;+\; \mathbb{E}[\mathcal B_T ].
\end{equation}
\begin{proof}
This follows from the decomposition in Theorem 3.3 of \cite{foster2021statistical}, plus an additive bias term to account for potential misspecification of surrogate rewards
\end{proof}
\end{lemma}

\begin{lemma}[Bias Decomposition]\label{lem:cabsc-bias-decomp}
Given the regret decomposition provided in Equation~\ref{eq:dec-decomp},  we may decompose the expected bias term $\mathbb{E}[\mathcal{B}_T]$ of Algorithm~\ref{alg:CABS-C} as
\begin{equation*}
\mathbb{E}[\mathcal{B}_T]
=
\sum_{t=1}^T \sum_{j=1}^K \left[
\frac{w_{t,j}}{n_{t,j}}
\sum_{\tau \in \mathcal T_t(j)} \Delta_{\tau,j}
\right]
+
\sum_{t=1}^T \sum_{j=1}^K
\left|
\frac{w_{t,j}}{n_{t,j}}
\sum_{\tau \in \mathcal T_t(j)} \xi_{\tau,j}
\right|,
\end{equation*}
where $w_{t,j} = \frac{p_{t,j}}{\boldsymbol{A}_{t,(:,j)}^\top \boldsymbol{p}_t}.$ The first term corresponds to what we term the regression-propagation error, and the second term captures the contribution of observation noise.
\begin{proof}
The term $\mathbb{E}[\mathcal B_T]$ captures the expected cumulative regret induced by bias in the surrogate reward estimates. Let us begin by defining $e_{t,j} := |b_j - \hat b_{t,j}|$, where $b_j$ is the true bias of arm $j$, and $\hat b_{t,j}$ is the estimated bias of arm $j$ at round $t$. We may state the cumulative regret as
\begin{equation*}
\mathbb{E}[\mathcal{B}_T]
= \sum_{t=1}^T \, \mathbb{E} \left[ \sum_{j: A_{t,(i_t,j)}=1} e_{t,j}\right]
\end{equation*}
At round $t$, the instantaneous contribution of bias to the regret is
\begin{equation*}
\left[ \sum_{j: A_{t,(i_t,j)}=1} e_{t,j}\right] = \sum_{j=1}^K \frac{p_{t,j}}{\boldsymbol{A}_{t,(:,j)}^\top \boldsymbol{p}_t}\, e_{t,j}
\end{equation*}
We use the definition of $w_{t,j}$ to rewrite this as
\begin{equation*}
\mathbb{E}[\mathcal B_T]
\;:=\;
\sum_{t=1}^T \sum_{j=1}^K w_{t,j}\, e_{t,j}.
\end{equation*}
We now make explicit the structure of $e_{t,j}$.
Recall that the bias estimate $\hat b_{t,j}$ is maintained as a running average over past observations of arm $j$.
By definition of the running average from Algorithm~\ref{alg:CABS-C},
\[
\hat b_{t,j}
\;=\;
\frac{1}{n_{t,j}}
\sum_{\tau \in \mathcal T_t(j)}
\left( s_{\tau,j} - \boldsymbol{x}_\tau^{\top} \hat{\boldsymbol{\theta}}_{\tau,j} 
% - \hat f_\tau(x_\tau,j)
\right).
\]
Substituting the reward model
\[
s_{\tau,j}
\;=\;
% f^\star(x_\tau,j)
\boldsymbol{x}_\tau^{\top}\boldsymbol{\theta}^*_{j} + b_j + \xi_{\tau,j},
\]
we obtain
\begin{equation*}
\hat b_{t,j}
=
b_j
+
\frac{1}{n_{t,j}}
\sum_{\tau \in \mathcal T_t(j)}
\left(
\boldsymbol{x}_t^{\top} \hat{\boldsymbol{\theta}}_{\tau,j} - \boldsymbol{x}_\tau^{\top}\boldsymbol{\theta}^*_{j}
\right)
+
\frac{1}{n_{t,j}}
\sum_{\tau \in \mathcal T_t(j)} \xi_{\tau,j}.
\end{equation*}
Plugging this into the definition of $e_{t,j}$ yields
\begin{align*}
e_{t,j}
&=
\left|
\frac{1}{n_{t,j}}
\sum_{\tau \in \mathcal T_t(j)}
\left(
\boldsymbol{x}_t^{\top} \hat{\boldsymbol{\theta}}_{\tau,j} - \boldsymbol{x}_\tau^{\top}\boldsymbol{\theta}^*_{j}
\right)
+\frac{1}{n_{t,j}}
\sum_{\tau \in \mathcal T_t(j)} \xi_{\tau,j}
\right|\\
&\le
\left|
\frac{1}{n_{t,j}}
\sum_{\tau \in \mathcal T_t(j)}
\left(
\boldsymbol{x}_t^{\top} \hat{\boldsymbol{\theta}}_{\tau,j} - \boldsymbol{x}_\tau^{\top}\boldsymbol{\theta}^*_{j}
\right)\right|
+ \left|\frac{1}{n_{t,j}}
\sum_{\tau \in \mathcal T_t(j)} \xi_{\tau,j}
\right|
\end{align*}
Recall,
\[
\Delta_{\tau,j}
:= 
|\boldsymbol{x}_t^{\top} \hat{\boldsymbol{\theta}}_{\tau,j} - \boldsymbol{x}_\tau^{\top}\boldsymbol{\theta}^*_{j}|,
\]
After plugging this into our previous equation and summing over all rounds $T$, this can be written compactly as
\begin{equation*}
\mathbb{E}[\mathcal{B}_T]
\le
\sum_{t=1}^T \sum_{j=1}^K \left[
\frac{w_{t,j}}{n_{t,j}}
\sum_{\tau \in \mathcal T_t(j)} \Delta_{\tau,j}
\right]
+
\sum_{t=1}^T \sum_{j=1}^K
\left|
\frac{w_{t,j}}{n_{t,j}}
\sum_{\tau \in \mathcal T_t(j)} \xi_{\tau,j}
\right|,
\end{equation*}

\end{proof}
\end{lemma}

\begin{lemma}[Bounding the Non-Bias Terms]\label{lem:cabsc-dec-bound}
Suppose the feedback graph $G_t$ is deterministic with independence number no more than $\alpha$. Then Algorithm~\ref{alg:CABS-C} guarantees that
\begin{equation*}
\phi(\boldsymbol{p}; \hat{\boldsymbol{\theta}}_{t}, \boldsymbol{x}_t, \boldsymbol{A}_t) 
\le
O\!\left(\frac{\alpha \log(K\gamma)}{\gamma}\right)
\end{equation*}
\begin{proof}
The following proof primarily follows from \cite{zhang2023practical} (Theorem 3.2). We include a summarized version below for completeness. Let us first rewrite $\phi(\boldsymbol{p}; \hat{\boldsymbol{\theta}}_{t}, \boldsymbol{x}_t, \boldsymbol{A}_t)$ in terms of losses:

\[
\phi(\boldsymbol{p}; \hat{\boldsymbol{\theta}}_{t}, \boldsymbol{x}_t, \boldsymbol{A}_t)
             := \sup_{\substack{ i^* \in [K] \\ \boldsymbol{\theta}^* \in \mathbb{R}_+^{K \times d}}} \hspace{0.5em} \mathbb{E}_{i \sim p} \left[ \boldsymbol{x}_t^{\top} \boldsymbol{\theta}^*_{i} - \boldsymbol{x}_t^{\top} \boldsymbol{\theta}^*_{i^*} - \frac{\gamma}{4} \sum_{k: A_{t,(i,k)} = 1} \mathbb{E}[A_{t,(i,k)}](\boldsymbol{x}_t^{\top} \hat{\boldsymbol{\theta}}_{t,k} - \boldsymbol{x}_t^{\top}\boldsymbol{\theta}^*_{k})^2 \right]
\]
Direct calculation shows that for all $i^* \in [K]$,
\newcommand{\decW}{\sum_{k: A_{t,(k,i)} = 1}\,p_{t,k}\,}
\begin{align*}
& \mathbb{E}_{i \sim p} \left[ \boldsymbol{x}_t^{\top} \boldsymbol{\theta}^*_{i} - \boldsymbol{x}_t^{\top} \boldsymbol{\theta}^*_{i^*} - \frac{\gamma}{4} \sum_{k: A_{t,(i,k)} = 1} \mathbb{E}[A_{t,(i,k)}](\boldsymbol{x}_t^{\top} \hat{\boldsymbol{\theta}}_{t,k} - \boldsymbol{x}_t^{\top}\boldsymbol{\theta}^*_{k})^2 \right]\\
&= \sum_{i=1}^K p_{t,i}\boldsymbol{x}_t^{\top} \boldsymbol{\theta}^*_{i} - \boldsymbol{x}_t^{\top} \boldsymbol{\theta}^*_{i^*} - \frac{\gamma}{4} \sum_{i=1}^K \decW (\boldsymbol{x}_t^{\top} \hat{\boldsymbol{\theta}}_{t,i} - \boldsymbol{x}_t^{\top}\boldsymbol{\theta}^*_{i})^2
\end{align*}

We can take the gradient over $f^*(\boldsymbol{x}_t,\cdot)= \boldsymbol{x}_t^{\top}\boldsymbol{\theta}^*_{\cdot}$ to get
\begin{align*}
& \sup_{f^* \in (\mathcal{X} \times [K] \rightarrow \mathbb{R})} \left[
\sum_{i=1}^K p_{t,i}\boldsymbol{x}_t^{\top} \boldsymbol{\theta}^*_{i} - \boldsymbol{x}_t^{\top} \boldsymbol{\theta}^*_{i^*} - \frac{\gamma}{4} \sum_{i=1}^K \decW (\boldsymbol{x}_t^{\top} \hat{\boldsymbol{\theta}}_{t,i} - \boldsymbol{x}_t^{\top}\boldsymbol{\theta}^*_{i})^2\right]\\
& = \sum_{i=1}^K p_{t,i} \boldsymbol{x}_t^{\top} \hat{\boldsymbol{\theta}}_{t,i} - \boldsymbol{x}_t^{\top} \hat{\boldsymbol{\theta}}_{t,i^*} + \frac{1}{\gamma}\|\boldsymbol{p}_t-\boldsymbol{e}_{i^*}\|_{\mathrm{diag}(\boldsymbol{W}_t)^{-1}}^2
\end{align*}
Where $W_{t,i} = \sum_{k: A_{t,(k,i)} = 1}\,p_{t,k}$. Next, let us consider the following minimax form over $\boldsymbol{p}_t \in \Delta([K])$ and $i^* \in [K]$:
\begin{align}
& \inf_{\boldsymbol{p}_t \in \Delta([K])} \, \sup_{i^* \in [K]} \, \left\{ \sum_{i=1}^K p_{t,i} \boldsymbol{x}_t^{\top} \hat{\boldsymbol{\theta}}_{t,i} - \boldsymbol{x}_t^{\top} \hat{\boldsymbol{\theta}}_{t,i^*} + \frac{1}{\gamma}\|\boldsymbol{p}_t-\boldsymbol{e}_{i^*}\|_{\mathrm{diag}(\boldsymbol{W}_t)^{-1}}^2 \right\}\notag \\
& = \min_{\boldsymbol{p}_t \in \Delta([K])} \, \max_{i^* \in [K]} \, \left\{ \sum_{i=1}^K p_{t,i} \boldsymbol{x}_t^{\top} \hat{\boldsymbol{\theta}}_{t,i} - \boldsymbol{x}_t^{\top} \hat{\boldsymbol{\theta}}_{t,i^*} + \frac{1}{\gamma}\sum_{i \neq i^*} \frac{p_{t,i}^2}{W_{t,i}} + \frac{1}{\gamma}\frac{(1-p_{t,i^*})^2}{\gamma \, W_{t,i}} \right\} \label{eq:dec-zhang-convex}\\
& = \max_{\boldsymbol{q}_t \in \Delta_K} \, \min_{\boldsymbol{p}_t \in \Delta_K} \, \left\{ \sum_{i=1}^K p_{t,i} \boldsymbol{x}_t^{\top} \hat{\boldsymbol{\theta}}_{t,i} - \boldsymbol{x}_t^{\top} \hat{\boldsymbol{\theta}}_{t,i^*} + \frac{1}{\gamma}\sum_{i=1}^K \frac{p_{t,i}^2(1-q_{t,i})}{W_{t,i}} + \sum_{i=1}^K \frac{q_{t,i}(1-p_{t,i})^2}{\gamma \, W_{t,i}} \right\} \notag
\end{align}
where the last inequality is due to the fact that Equation \eqref{eq:dec-zhang-convex} is convex in $p \in \Delta([K])$ and Sion's minimax theorem. Choose $p_{t,i} = (1-\frac{1}{\gamma})q_{t,i} + \frac{1}{\gamma K}$ for all $i \in [K]$. Due to Assumption~\ref{asm:strong-observability}, all nodes in the feedback graph $G_t$ have a self-loop. This allows us to upper bound the quantity as
\begin{align*}
& = \max_{\boldsymbol{q}_t \in \Delta([K])} \, \min_{\boldsymbol{p}_t \in \Delta_K} \, \left\{ \sum_{i=1}^K p_{t,i} \boldsymbol{x}_t^{\top} \hat{\boldsymbol{\theta}}_{t,i} - \boldsymbol{x}_t^{\top} \hat{\boldsymbol{\theta}}_{t,i^*} + \frac{1}{\gamma}\sum_{i=1}^K \frac{p_{t,i}^2(1-q_{t,i})}{W_{t,i}} + \sum_{i=1}^K \frac{q_{t,i}(1-p_{t,i})^2}{\gamma \, W_{t,i}} \right\}\\
& \le \max_{\boldsymbol{q}_t \in \Delta([K])} \, \left\{  \frac{2}{\gamma}+\frac{1}{\gamma} \sum_{i=1}^K \frac{\left((1-\frac{1}{\gamma}q_{t,i} + \frac{1}{\gamma K}\right)^2 (1-q_{t,i})+q_{t,i}\left(1-(1-\frac{1}{\gamma})q_{t,i} - \frac{1}{\gamma K}\right)^2}{W_{t,i}} \right\}\\
& \le \max_{\boldsymbol{q}_t \in \Delta([K])} \, \left\{  \frac{2}{\gamma}+\frac{1}{\gamma} \sum_{i=1}^K \frac{2\left((1-\frac{1}{\gamma})^2 q_{t,i}^2 + \frac{1}{\gamma^2 K^2}\right) (1-q_{t,i})+q_{t,i}\left(1-(1-\frac{1}{\gamma})q_{t,i}\right)^2}{W_{t,i}} \right\}\\
& \le \max_{\boldsymbol{q}_t \in \Delta([K])} \, \left\{  \frac{2}{\gamma}+\frac{2}{\gamma^2} + \frac{1}{\gamma} \sum_{i=1}^K \frac{2q_{t,i}^2(1-q_{t,i})+2q_{t,i}(1-q_{t,i})^2 + \frac{2q_{t,i}^3}{\gamma^2}}{W_{t,i}} \right\}
\end{align*}
Since $W_{t,i} = \decW \ge \frac{1}{\gamma K}$ for all $i \in [K]$ We can further bound this by saying
\begin{align*}
& \le \max_{\boldsymbol{q}_t \in \Delta([K])} \, \left\{  \frac{2}{\gamma}+\frac{2}{\gamma^2} + \frac{1}{\gamma} \sum_{i=1}^K \frac{2q_{t,i}^2(1-q_{t,i})+2q_{t,i}(1-q_{t,i})^2 + \frac{2q_{t,i}^3}{\gamma^2}}{W_{t,i}} \right\}\\
& \le \max_{\boldsymbol{q}_t \in \Delta([K])} \, \left\{  \frac{2}{\gamma}+\frac{2}{\gamma^2} + \frac{2}{\gamma} \sum_{i=1}^K \frac{q_{t,i}(1-q_{t,i})}{W_{t,i}} + \frac{2}{\gamma^3} \sum_{i=1}^K \frac{q_{t,i}^3}{W_{t,i}} \right\}\\
& \le \max_{\boldsymbol{q}_t \in \Delta([K])} \, \left\{  \frac{2}{\gamma}+\frac{2}{\gamma^2} + \frac{2}{\gamma} \sum_{i=1}^K \frac{q_{t,i}(1-q_{t,i})}{W_{t,i}} + \frac{2}{\gamma^3} \sum_{i=1}^K q_{t,i}^2 \right\}\\
& \le \max_{\boldsymbol{q}_t \in \Delta_K} \, \left\{  \frac{8}{\gamma} + \frac{2}{\gamma} \sum_{i=1}^K \frac{q_{t,i}(1-q_{t,i})}{W_{t,i}} \right\}\\
\end{align*}
Next, we will bound $\frac{2q_{t,i}(1-q_{t,i})}{W_{t,i}}$ for each $i \in [K]$. Since we know that all nodes have self-loops in our feedback graph,we know that
\begin{align*}
\sum_{i=1}^K \frac{2q_{t,i}(1-q_{t,i})}{W_{t,i}} &\le \sum_{i=1}^K \frac{2q_{t,i}(1-q_{t,i})}{\sum_{j:A_{t,(j,i)}}((1-\frac{1}{\gamma})q_{t,j} + \frac{1}{\gamma K})}\\
&\le \frac{\gamma}{\gamma-1}\sum_{i=1}^K \frac{2((1-\frac{1}{\gamma})q_{t,i}+\frac{1}{\gamma K})(1-q_{t,i})}{\sum_{j:A_{t,(j,i)}}((1-\frac{1}{\gamma})q_{t,j} + \frac{1}{\gamma K})}\\
&\le 4 \sum_{i=1}^K \frac{((1-\frac{1}{\gamma})q_{t,i}+\frac{1}{\gamma K})}{\sum_{j:A_{t,(j,i)}}((1-\frac{1}{\gamma})q_{t,j} + \frac{1}{\gamma K})} \le O\!\left(\alpha \log(K\gamma)\right)
\end{align*}
The last inequality is due to Lemma 5 from~\cite{alon2015online}. Combining all of the above inequalities, we get
\begin{align*}
& \inf_{\boldsymbol{p}_t \in \Delta([K])} \, \sup_{i^* \in [K]} \, \left\{ \sum_{i=1}^K p_{t,i} \boldsymbol{x}_t^{\top} \hat{\boldsymbol{\theta}}_{t,i} - \boldsymbol{x}_t^{\top} \hat{\boldsymbol{\theta}}_{t,i^*} + \frac{1}{\gamma}\|\boldsymbol{p}_t-\boldsymbol{e}_{i^*}\|_{\mathrm{diag}(\boldsymbol{W}_t)^{-1}}^2 \right\}\\
& \le \max_{\boldsymbol{q}_t \in \Delta_K} \, \left\{  \frac{8}{\gamma} + \frac{2}{\gamma} \sum_{i=1}^K \frac{q_{t,i}(1-q_{t,i})}{W_{t,i}} \right\} \le O\!\left(\frac{\alpha \log(K\gamma)}{\gamma}\right)
\end{align*}
\end{proof}
\end{lemma}

Before introducing the lemmas to bound the noise terms of our regret decomposition, we introduce the following useful graph-theoretic result, which follows from ~\cite{alon2015online}.

\begin{lemma}[A Useful Graph-Theoretic Property]\label{lem:alon-graph-alpha} 
Let $G_t$ be a directed graph with vertices $i \in [K]$, in which $A_{t,(i,i)}=1$ for all vertices. Assign each $i \in [K]$ a positive weight $p_{t,i}$ such that $\sum_{i \in [K]} p_{t,i} \le 1$ and $p_{t,i} \ge \epsilon$ for all $i \in [K]$. Then
\[
\sum_{i=1}^K \frac{p_{t,i}}{
\boldsymbol{A}_{t,(:,i)}^\top\boldsymbol p_t}
\;\le\; \alpha \cdot \ln\!\left(\frac{4K}{\alpha\epsilon}\right),
\]
where $\alpha$ is the independence number of $G_t$.
\begin{proof}
This follows from Lemma 5 of ~\cite{alon2015online}.
\end{proof}
\end{lemma}

\begin{lemma}[Bounding the Regression-Propagation Error]\label{lem:cabsc-reg-prop-bound}
Suppose the feedback graph $G_t$ is deterministic with independence number no more than $\alpha$. Then Algorithm~\ref{alg:CABS-C} guarantees that the regression-propagation error term of the bias may be bounded as
\begin{equation*}
\sum_{t=1}^T \sum_{j=1}^K \left[
\frac{w_{t,j}}{n_{t,j}}
\sum_{\tau \in \mathcal T_t(j)} \Delta_{\tau,j}
\right]
\le
\sqrt{\alpha \,\ln(4K/\alpha\epsilon)\,(1+\ln T)\, T\,\mathrm{Reg}_{\mathrm{Sq}}(T)}
\end{equation*}
\begin{proof}
We begin by applying the Cauchy--Schwarz inequality to the regression-propagation error term and performing algebraic manipulation:
\begin{align}
    \sum_{t=1}^T \sum_{j=1}^K \left[
\frac{w_{t,j}}{n_{t,j}}
\sum_{\tau \in \mathcal T_t(j)} \Delta_{\tau,j}
\right]
&\le
\sqrt{\sum_{t=1}^T\sum_{j=1}^K w_{t,j}}
\cdot
\sqrt{\sum_{t=1}^T\sum_{j=1}^K\frac{w_{t,j}}{n_{t,j}}\sum_{\tau\in\mathcal T_t(j)}\Delta_{\tau,j}^2} \notag\\
&\le
\sqrt{\alpha\,T\;\ln\!\left(\frac{4K}{\alpha\epsilon}\right)}
\cdot
\sqrt{\sum_{t=1}^T\sum_{j=1}^K\frac{w_{t,j}}{n_{t,j}}\sum_{\tau\in\mathcal T_t(j)}\Delta_{\tau,j}^2} \label{eq:reg-prop-err-cs1}
\end{align}
where Equation~\ref{eq:reg-prop-err-cs1} is due to applying Lemma~\ref{lem:alon-graph-alpha} to the first term and assuming \(p_{t,i}\ge \epsilon\) for all \(i\) (Assumption~\ref{asm:min-probability}).\\\\
Next we swap sums in the second term to group by the regression error 
\(\Delta_{\tau,j}^2\). To do so, we must re-index the summations. We temporarily define $t_{j,m}$ as the round in which the learner observes arm $j$ for the $m$-th time. After re-indexing we get
\begin{align}
&\sqrt{\alpha\,T\;\ln\!\left(\frac{4K}{\alpha\epsilon}\right)}
\cdot
\sqrt{\sum_{j=1}^K\sum_{m=2}^{n_{T,j}}\frac{w_{t_{j,m},j}}{m-1}\sum_{u=1}^{m-1}\Delta_{t_{j,u},j}^2} \notag\\
&=
\sqrt{\alpha\,T\;\ln\!\left(\frac{4K}{\alpha\epsilon}\right)}
\cdot
\sqrt{\sum_{j=1}^K \sum_{u=1}^{n_{T,j}-1}\Delta_{t_{j,u},j}^2 \sum_{m=u+1}^{n_{T,j}}\frac{w_{t_{j,m},j}}{m-1}} \label{eq:reg-prop-err-swapsum}\\
&\le
\sqrt{\alpha\,T\;\ln\!\left(\frac{4K}{\alpha\epsilon}\right)}
\cdot
\sqrt{\sum_{j=1}^K \sum_{u=1}^{n_{T,j}-1}\Delta_{t_{j,u},j}^2 \sum_{m=1}^{n_{T,j}}\frac{1}{m}} \notag\\
&\le
\sqrt{\alpha\,T\;\ln\!\left(\frac{4K}{\alpha\epsilon}\right)}
\cdot
\sqrt{\sum_{j=1}^K \sum_{u=1}^{n_{T,j}-1}\Delta_{t_{j,u},j}^2 \cdot (1 + \ln T)} \label{eq:reg-prop-err-harmonic}\\
&\le
\sqrt{\alpha\,T\;\ln\!\left(\frac{4K}{\alpha\epsilon}\right)}
\cdot
\sqrt{\sum_{j=1}^K \sum_{t=1}^{T}\Delta_{t,j}^2 \cdot (1 + \ln T)} \notag\\
&\le
\sqrt{\alpha\,T\;\ln\!\left(\frac{4K}{\alpha\epsilon}\right)}
\cdot
\sqrt{\mathrm{Reg_{Sq}}(T) \cdot (1 + \ln T)}, \label{eq:reg-prop-err-mult}
\end{align}
where Equation~\ref{eq:reg-prop-err-swapsum} is due to swapping sums, and Equation~\ref{eq:reg-prop-err-harmonic} is due to the following fact about harmonic sums: For any arm $j,$ the sum, over its observation indices $m=1,\dots,n_{T,j}$, of $1/m$ satisfies
\begin{equation*}
\sum_{m=1}^{n_{T,j}} \frac{1}{m} \le 1 + \ln n_{T,j} \le 1 + \ln T.
\end{equation*}
We can combine the terms in Equation~\ref{eq:reg-prop-err-mult} to finally get
\begin{equation*}
\sum_{t=1}^T \sum_{j=1}^K \left[
\frac{w_{t,j}}{n_{t,j}}
\sum_{\tau \in \mathcal T_t(j)} \Delta_{\tau,j}
\right]
\le
\sqrt{\alpha T\,\ln(4K/\alpha\epsilon)\,(1+\ln T)\,\mathrm{Reg}_{\mathrm{Sq}}(T)}
\end{equation*}
\end{proof}
\end{lemma}

\begin{lemma}[Bounding the Observation Noise]\label{lem:cabsc-obs-noise-bound}
Suppose the feedback graph $G_t$ is deterministic with independence number no more than $\alpha$. Then, for all $\delta > 0$, with probability $1-\delta$ Algorithm~\ref{alg:CABS-C} guarantees that the observation noise term of the bias may be bounded as
\begin{equation*}
\sum_{t=1}^T \sum_{j=1}^K
\left|
\frac{w_{t,j}}{n_{t,j}}
\sum_{\tau \in \mathcal T_t(j)} \xi_{\tau,j}
\right|
\le
\varepsilon_n \sqrt{{2 \alpha KT \ln\!\left(\frac{2KT}{\delta}\right)}
\ln\!\left(\frac{4K}{\alpha\epsilon}\right)
(1+\ln T)}
\end{equation*}
\begin{proof}
We begin by bounding the term 
\begin{equation*}
\left|\frac{1}{n_{t,j}}\sum_{\tau\in\mathcal T_t(j)}\xi_{\tau,j}\right|
\end{equation*}
We do so by fixing a specific arm $j$ and considering the sequence of noise terms observed for this arm. Let $k$ index the observations of arm $j$, such that $\zeta_k$ is the noise of the $k$-th observation. The sequence of partial sums $S_m = \sum_{k=1}^m \zeta_k$ forms a martingale with respect to the filtration of observation events, satisfying the bounded difference property:
\[
|S_m - S_{m-1}| = |\zeta_m| \le \varepsilon_n.
\]
We apply the {Azuma-Hoeffding inequality}. For any $m \ge 1$ and any $\lambda > 0$:
\[
\mathbb{P}\left( |S_m| \ge \lambda \right) \le 2 \exp\!\left( -\frac{\lambda^2}{2 \sum_{k=1}^m (\varepsilon_n)^2} \right) = 2 \exp\!\left( -\frac{\lambda^2}{2 m \varepsilon_n^2} \right).
\]
We seek a bound that holds with high probability $1 - \delta'$. Setting the right-hand side to $\delta'$ and solving for $\lambda$:
\begin{align*}
    2 \exp\!\left( -\frac{\lambda^2}{2 m \varepsilon_n^2} \right) &= \delta' \\
    -\frac{\lambda^2}{2 m \varepsilon_n^2} &= \ln(\delta'/2) \\
    \lambda &= \sqrt{2 m \varepsilon_n^2 \ln(2/\delta')}.
\end{align*}
Substituting $S_m = n_{t,j} \cdot \frac{1}{n_{t,j}}\sum_{\tau\in\mathcal T_t(j)}\xi_{\tau,j}$ (where $m=n_{t,j}$) and dividing by $n_{t,j}$:
\[
\mathbb{P}\left( \left|\frac{1}{n_{t,j}} \sum_{\tau \in \mathcal{T}_t(j)} \xi_{\tau,j}\right| \ge \varepsilon_n \sqrt{\frac{2 \ln(2/\delta')}{n_{t,j}}} \right) \le \delta'.
\]
To ensure this bound holds simultaneously for all rounds $t \in [T]$ and all arms $j \in [K]$, we apply a {Union Bound} over the $KT$ possible events. We set the individual failure probability to $\delta' = \delta / (KT)$.
Substituting $\delta'$ into the bound:
\[
\left|\frac{1}{n_{t,j}} \sum_{\tau \in \mathcal{T}_t(j)} \xi_{\tau,j}\right| \le \varepsilon_n \sqrt{\frac{2 \ln(2KT/\delta)}{n_{t,j}}}.
\]
We may plug this quantity back into the observation noise term to get
\begin{align}
\sum_{t=1}^T \sum_{j=1}^K
\left|
\frac{w_{t,j}}{n_{t,j}}
\sum_{\tau \in \mathcal T_t(j)} \xi_{\tau,j}
\right|
&\le
\varepsilon_n \sqrt{{2 \ln\!\left(\frac{2KT}{\delta}\right)}}
\sum_{t=1}^T\sum_{j=1}^K\frac{w_{t,j}}{\sqrt{n_{t,j}}}\notag\\
&\le
\varepsilon_n \sqrt{{2 \ln\!\left(\frac{2KT}{\delta}\right)}}
\sqrt{\left(\sum_{t=1}^T\sum_{j=1}^K w_{t,j}\right)}
\sqrt{\left(\sum_{t=1}^T\sum_{j=1}^K\frac{w_{t,j}}{n_{t,j}}\right)}\label{eq:obs-noise-cs}\\
&\le
\varepsilon_n \sqrt{{2 \ln\!\left(\frac{2KT}{\delta}\right)}}
\sqrt{\alpha T \ln\!\left(\frac{4K}{\alpha\epsilon}\right)}
\sqrt{\left(\sum_{t=1}^T\sum_{j=1}^K\frac{w_{t,j}}{n_{t,j}}\right)}\label{eq:obs-noise-wbound}\\
&\le
\varepsilon_n \sqrt{{2 \ln\!\left(\frac{2KT}{\delta}\right)}}
\sqrt{\alpha T \ln\!\left(\frac{4K}{\alpha\epsilon}\right)}
\sqrt{\sum_{j=1}^K(1+\ln T)}\label{eq:obs-noise-harmonic}\\
&\le
\varepsilon_n \sqrt{{2 \ln\!\left(\frac{2KT}{\delta}\right)}}
\sqrt{\alpha T \ln\!\left(\frac{4K}{\alpha\epsilon}\right)}
\sqrt{K(1+\ln T)}\notag,
\end{align}
where Equation~\ref{eq:obs-noise-cs} is due to applying the Cauchy-Shwarz inequality, Equation~\ref{eq:obs-noise-wbound} is due to the same application of Lemma~\ref{lem:alon-graph-alpha} as Equation~\ref{eq:reg-prop-err-cs1}, and Equation~\ref{eq:obs-noise-harmonic} is due to applying the same fact about harmonic sums used in Equation~\ref{eq:reg-prop-err-harmonic}. Combining these terms yields (with probability at least $1-\delta$):
\begin{equation*}
\sum_{t=1}^T \sum_{j=1}^K
\left|
\frac{w_{t,j}}{n_{t,j}}
\sum_{\tau \in \mathcal T_t(j)} \xi_{\tau,j}
\right|
\le
\varepsilon_n \sqrt{{2 \alpha KT \ln\!\left(\frac{2KT}{\delta}\right)}
\ln\!\left(\frac{4K}{\alpha\epsilon}\right)
(1+\ln T)}.
\end{equation*}
\end{proof}
\end{lemma}

\paragraph{Proof of Theorem~\ref{theorem:CABS-C}.}
We now utilize the preceding lemmas to prove that under the assumptions outlined in Section~\ref{sec:coupled-reg}, if the feedback graph $G_t$ is deterministic with independence number no more than $\alpha$. Then, with probability $1-\delta$, Algorithm~\ref{alg:CABS-C} with choice $\gamma = \max\left\{4, \sqrt{\alpha T / \mathrm{Reg_{Sq}}(T)}\right\}$ guarantees that
\begin{equation*}
\mathbb{E}[\mathrm{Reg}_T]
\le
\widetilde O\!\left(
\sqrt{\alpha T\,\mathrm{Reg}_{\mathrm{Sq}}(T)}
+
\varepsilon_n\,\sqrt{\alpha K T}\
\right).
\end{equation*}
\begin{proof}
Combining the results of Lemmas~\ref{lem:cabsc-regret-decomp} and \ref{lem:cabsc-bias-decomp}, we get the full regret decomposition of Algorithm~\ref{alg:CABS-C}:
\begin{align*}
\mathbb{E}[\mathrm{Reg}_T]
\le
& \; T \cdot \phi(\boldsymbol{p}; \hat{\boldsymbol{\theta}}_{t}, \boldsymbol{x}_t, \boldsymbol{A}_t)
\;+\; 
\frac{\gamma}{4}\,\mathrm{Reg}_{\mathrm{Sq}}(T)\\
&\;+\; 
\sum_{t=1}^T \sum_{j=1}^K \left[
\frac{w_{t,j}}{n_{t,j}}
\sum_{\tau \in \mathcal T_t(j)} \Delta_{\tau,j}
\right]
\;+\; 
\sum_{t=1}^T \sum_{j=1}^K
\left|
\frac{w_{t,j}}{n_{t,j}}
\sum_{\tau \in \mathcal T_t(j)} \xi_{\tau,j}
\right|
\end{align*}
We may plug in the result of Lemma~\ref{lem:cabsc-dec-bound} into $\phi(\boldsymbol{p}; \hat{\boldsymbol{\theta}}_{t}, \boldsymbol{x}_t, \boldsymbol{A}_t)$ to get
\begin{align*}
\mathbb{E}[\mathrm{Reg}_T]
\le
&\;O\!\left(\frac{\alpha T \log(K\gamma)}{\gamma}\right)
\;+\; 
\frac{\gamma}{4}\,\mathrm{Reg}_{\mathrm{Sq}}(T)\\
&\;+\; 
\sum_{t=1}^T \sum_{j=1}^K \left[
\frac{w_{t,j}}{n_{t,j}}
\sum_{\tau \in \mathcal T_t(j)} \Delta_{\tau,j}
\right]
\;+\; 
\sum_{t=1}^T \sum_{j=1}^K
\left|
\frac{w_{t,j}}{n_{t,j}}
\sum_{\tau \in \mathcal T_t(j)} \xi_{\tau,j}
\right|
\end{align*}

We select $\gamma = \sqrt{\alpha T / \mathrm{Reg_{Sq}}(T)}$ (for stability we must ensure that $\gamma \ge 4$) and combine the first two terms of the decomposition to get (absorbing polylog factors)
\begin{align*}
\mathbb{E}[\mathrm{Reg}_T]
\le
&\;\widetilde O\!\left(
\sqrt{\alpha T\,\mathrm{Reg}_{\mathrm{Sq}}(T)}
\right)
\;+\; 
\sum_{t=1}^T \sum_{j=1}^K \left[
\frac{w_{t,j}}{n_{t,j}}
\sum_{\tau \in \mathcal T_t(j)} \Delta_{\tau,j}
\right]
\;+\; 
\sum_{t=1}^T \sum_{j=1}^K
\left|
\frac{w_{t,j}}{n_{t,j}}
\sum_{\tau \in \mathcal T_t(j)} \xi_{\tau,j}
\right|
\end{align*}
Applying the results of Lemmas~\ref{lem:cabsc-reg-prop-bound} and \ref{lem:cabsc-obs-noise-bound} we get
\begin{align*}
\mathbb{E}[\mathrm{Reg}_T]
\le
&\;\widetilde O\!\left(
\sqrt{\alpha T\,\mathrm{Reg}_{\mathrm{Sq}}(T)}
\right)
\;+\; 
\sqrt{\alpha T \,\ln(4K/\alpha\epsilon)\,(1+\ln T)\,\mathrm{Reg}_{\mathrm{Sq}}(T)}\\
&\;+\; 
\varepsilon_n \sqrt{{2 \alpha KT \ln\!\left(\frac{2KT}{\delta}\right)}
\ln\!\left(\frac{4K}{\alpha\epsilon}\right)
(1+\ln T)}.
\end{align*}
combining the final two terms and once again absorbing polylog factors, we get
\begin{equation*}
\mathbb{E}[\mathrm{Reg}_T]
\le
\widetilde O\!\left(
\sqrt{\alpha T\,\mathrm{Reg}_{\mathrm{Sq}}(T)}
+
\varepsilon_n\,\sqrt{\alpha K T}\
\right).
\end{equation*}
\end{proof}

\subsubsection{Bounds for $m$-Transitive Feedback Graphs}
\label{app:m-trans-feedback}
For completeness, we now state a theorem relating the independence number $\alpha$ of the feedback graph $G_t$ with a constant number of additional arms $m$ observed at each round. This result enables us to tighten existing regret bounds for both standard and contextual bandits by a factor of roughly $1/\sqrt{m}$, quantifying the direct benefit of the auxiliary feedback.

\begin{theorem}[Independence Number for Transitive $m$-Regular Feedback]\label{thm:m-trans-feedback}
Consider a bandit setting with $K$ arms where the feedback graph $G_t$ is transitive and $m$-regular (i.e., each vertex has out-degree $m$). Under these conditions, $G_t$ decomposes into $K/(m+1)$ disjoint cliques of size $m+1$, and its independence number is:
\begin{equation*}
\alpha(G_t) = \frac{K}{m+1}
\end{equation*}
\end{theorem}
\begin{proof}
Let $G_t = (V, E)$ be the feedback graph. By the assumption of transitivity, if $(u, v) \in E$ and $(v, w) \in E$, then $(u, w) \in E$. In the context of an undirected or bi-directed feedback graph, transitivity implies that the graph is a disjoint union of cliques. Since $G_t$ is $m$-regular, each vertex belongs to a clique of size exactly $m+1$. Consequently, the vertex set $V$ is partitioned into $N$ disjoint cliques $\{C_1, C_2, \dots, C_N\}$, where:
\begin{equation*}
    N = \frac{|V|}{m+1} = \frac{K}{m+1}
\end{equation*}
To determine the independence number $\alpha(G_t)$, we observe the following:
\begin{enumerate}
    \item \textbf{Lower Bound:} 
        By the Caro-Wei Theorem \cite{caro1979new, wei1981lower}, for any graph with vertex degrees $d(k)$, the independence number is bounded by:
        \begin{equation}\alpha(G_t) \geq \sum_{k \in V} \frac{1}{d(k) + 1} = \sum_{k=1}^K \frac{1}{m+1} = \frac{K}{m+1} \label{eq:m-trans-alpha-lower}
        \end{equation}
    \item \textbf{Upper Bound:} 
        An independent set can contain at most one vertex from any clique. Since $V$ is partitioned into exactly $\frac{K}{m+1}$ disjoint cliques, the size of the maximum independent set is bounded by the number of cliques:
        \begin{equation}
            \alpha(G_t) \leq \frac{K}{m+1} \label{eq:m-trans-alpha-upper}
        \end{equation}
\end{enumerate}
Combining inequalities in Equations~\ref{eq:m-trans-alpha-lower} and \ref{eq:m-trans-alpha-upper}, we conclude that $\alpha(G_t) = \frac{K}{m+1}$.
\end{proof}

\paragraph{Deriving Corollary~\ref{cor:CABS-C_linear_mfeedback}.}
Corollary~\ref{cor:CABS-C_linear_mfeedback} follows by substituting $\alpha=\frac{K}{m+1}$ from Theorem~\ref{thm:m-trans-feedback} into the general regret bound of Theorem~\ref{theorem:CABS-C}, and applying the standard square-loss regret bound for linear regression oracles~\cite{foster2020beyond}, $\mathrm{Reg}_{\mathrm{Sq}}(T) = O(d\log \left(\frac{T}{d}\right))$.

\subsection{Correlation Aware Bandits with Surrogates Decoupled  (CABS-D)}
\label{appendix:CABS-D}

We now present a meta bandit algorithm that achieves the minimum regret of the above two sections, i.e, $\widetilde O\Big( \min \{\sqrt{dKT \log \left(\frac{T}{d}\right)}, \sqrt{d\frac{K}{m+1}T\log \left(\frac{T}{d}\right)}+ \varepsilon_n\sqrt{\frac{K^2}{m+1}T}\Big)$. In our main section, we present this as CABS-D (Algorithm~\ref{alg:CABS-D}). We start by presenting the Algorithm~\ref{alg:CABS-D} for Correlation-Aware Bandits with Surrogates Decoupled (CABS-D) and then prove the regret bound of it. 

\subsubsection{High Level Proof Sketch}
\label{app:CABS-D-sketch}
The proof of Theorem~\ref{alg:CABS-D} establishes a best-of-both-worlds guarantee by analyzing CABS-D as an aggregation of two base experts: a bandit-feedback expert, SquareCB (Algorithm~\ref{alg:squarecb-reward}) and a graph-feedback expert, CABS-C (Algorithm~\ref{alg:CABS-C}), each instantiated over a geometric grid of learning rates to form $M = 2L$ meta-copies. The analysis proceeds by first establishing the optimistic bias, uniform boundedness, and second-moment properties of the Implicit eXploration (IX) loss estimators utilized by both experts. Next, we decompose the total expected regret into the regret of the best-performing base expert ($R_m$), an additive bias term, and a meta-learning overhead term. By bounding the second-order variance terms independently for the regimes where either the bandit or graph expert dominates, and demonstrating that the geometric grid contains an optimal scaling parameter $\gamma^{(g)}$ for each respective regime, we show that the aggregation overhead is tightly controlled. Consequently, the total expected regret asymptotically matches the minimum regret of either the pure bandit or correlation-aware surrogate expert, yielding $\tilde{O}\left(\min\{\sqrt{KT \text{Reg}_{Sq}(T)}, \sqrt{\alpha T \text{Reg}_{Sq}(T)} + \epsilon_n\sqrt{\alpha KT}\}\right)$. \\

\noindent \textbf{Regret Bound for Algorithm~\ref{alg:CABS-D} ({\sffamily{CABS-D}})}\\

\noindent\textbf{Overview and setup.}  
We follow the same algorithmic structure as in Algorithm~\ref{alg:CABS-D} (geometric gridded \(\gamma\) per base expert). For clarity we restate the essential components and set the notation used throughout the proof. Note that we utilize losses in the following analysis. There are 2 \emph{base} experts (bandit-style ($m=1$) and graph-style ($m=2$)). Each base expert \(m\in\{1,2\}\) is instantiated with a geometric grid of $L$ parameter values $\Gamma=\{\gamma^{(1)},\dots,\gamma^{(L)}\}$, with dyadic spacing $\gamma^{(g+1)}=2\gamma^{(g)}$ (or any constant ratio \(>1\)). Each pair $(m,g)$ is a distinct \emph{meta-copy}. The total number of meta-copies is $M \;=\; 2L$.\\\\
To prove Theorem~\ref{theorem:CABS-D}, we first introduce the following auxiliary lemma regarding the properties of the Implicit eXploration (IX) estimators. We continue by introducing auxiliary lemmas regarding the regret decomposition and bounding second-order terms. We combine and utilize the results of these lemmas in the final proof of Theorem~\ref{theorem:CABS-D}.

\begin{lemma}[Properties of IX Estimators] \label{lem:ix-properties}
Let $\hat\ell_{t,(1,g)}$ and $\hat\ell_{t,(2,g)}$ be the IX estimators defined in Algorithm~\ref{alg:CABS-D}. Then, the following properties are true:
\begin{enumerate}
    \item (Optimistic Bias) For every meta expert $(m,g)$ in round $t$ and conditioned on $\mathcal{F}_{t-1}$, we have
    \begin{equation*}
        \mathbb{E}\left[\hat\ell_{t,(m,g)} \mid \mathcal{F}_{t-1}\right]
        \le \sum_i p_{t,m,i}\ell_{t,i}
    \end{equation*}
    \begin{proof} We begin by showing this is true for $\hat\ell_{t,(1,g)}$ (bandit experts):
        \begin{equation*}
        \mathbb{E}\left[\hat\ell_{t,(1,g)} \mid \mathcal{F}_{t-1}\right]
          \;=\;
          \sum_{i}\frac{q_{t,i}}{q_{t,i}+\gamma^{(g)}}\, p_{t,m,i}\,\ell_{t,i}
          \le \sum_i p_{t,1,i}\ell_{t,i}.
        \end{equation*}
        We proceed similarly for $\hat\ell_{t,(2,g)}$ (graph experts). We begin by rewriting the estimator as:
        \begin{align*}
            \mathbb{E}\left[\hat\ell_{t,(2,g)} \mid \mathcal{F}_{t-1}\right]
            & \;=\;
            \sum_{k=1}^K q_{t,k}\left[\sum_{i=1}^K \frac{\ell_{t,i}\, p_{t,2,i}}{\sum_{j=1}^K q_{t,j}\, \mathbb{I}\{{A}_{t,(j,i)}=1\} + \gamma^{(g)}} \cdot \mathbb{I}\!\left\{{A}_{t,(k,i)}=1\right\}\right]\\
            & \;=\;
            \sum_{i=1}^K \frac{\ell_{t,i}\, p_{t,2,i}}{\sum_{j=1}^K q_{t,j}\, \mathbb{I}\{{A}_{t,(j,i)}=1\} + \gamma^{(g)}} \cdot \sum_{k=1}^K q_{t,k}\, \mathbb{I}\!\left\{{A}_{t,(k,i)}=1\right\}\\
            & \;\le\;
            \sum_{i=1}^K \ell_{t,i}\, p_{t,2,i},
        \end{align*}
        where the last step is due to $\frac{\sum_{k=1}^K q_{t,k}\, \mathbb{I}\{{A}_{t,(k,i)}=1\}}{\sum_{j=1}^K q_{t,j}\, \mathbb{I}\{{A}_{t,(j,i)}=1\} + \gamma^{(g)}}\le1$.
    \end{proof}
    \item (Uniform Bound) For every meta expert $(m,g)$ in round $t$, we may bound the value of $\hat\ell_{t,(m,g)}$ as
    \begin{equation*}
        \hat\ell_{t,(m,g)}
        \le \frac{1}{\gamma^{(g)}}
    \end{equation*}
    \begin{proof} For both $\hat\ell_{t,(1,g)}$ and $\hat\ell_{t,(2,g)}$, the numerators are at most $1$ and denominators are at least $\gamma^{(g)}$.
    \end{proof}
    \item (Second Moment Bounds) For every bandit meta expert $(1,g)$ in round $t$ and conditioned on $\mathcal{F}_{t-1}$, we may say
    \begin{equation*}
      \mathbb{E}\left[\hat \ell_{t,(1,g)}^2 \mid \mathcal{F}_{t-1}\right]
      \le \sum_{i=1}^K \frac{p_{t,1,i}}{q_{t,i}+\gamma^{(g)}}.
    \end{equation*}
    Likewise, for every graph meta expert $(2,g)$ in round $t$ and conditioned on $\mathcal{F}_{t-1}$, we may say
    \begin{equation*}
        \mathbb{E}\left[\hat \ell_{t,(2,g)}^2 \mid \mathcal{F}_{t-1}\right]
        \le \sum_{i=1}^K \frac{p_{t,2,i}}{\sum_{j:{A}_{t,(j,i)}=1} q_{t,j} + \gamma^{(g)}}
    \end{equation*}
    \begin{proof}
    We begin by showing this is true for $\hat\ell_{t,(1,g)}$ (bandit experts):
    \begin{align*}
      \mathbb{E}\left[\hat \ell_{t,(1,g)}^2 \mid \mathcal{F}_{t-1}\right]
      = \sum_{i=1}^K q_{t,i}\left(\frac{\ell_{t,i\,}p_{t,1,i}}{q(a)+\gamma^{(g)}}\right)^2
      \le \sum_{i=1}^K \frac{(p_{t,1,i})^2}{q_{t,i}+\gamma^{(g)}}
      \le \sum_{i=1}^K \frac{p_{t,1,i}}{q_{t,i}+\gamma^{(g)}}
    \end{align*}
    We proceed similarly for $\hat\ell_{t,(2,g)}$ (graph experts). We once again rewrite the estimator as:
    \begin{align*}
      \mathbb{E}\left[\hat \ell_{t,(2,g)}^2 \mid \mathcal{F}_{t-1}\right] 
      &= \sum_{k=1}^K q_{t,k}\left(\sum_{i=1}^K \frac{\ell_{t,i}\, p_{t,2,i}}{\sum_{j=1}^K q_{t,j}\, \mathbb{I}\{{A}_{t,(j,i)}=1\} + \gamma^{(g)}} \cdot \mathbb{I}\{{A}_{t,(k,i)}=1\}\right)^2\\
      &\le \sum_{k=1}^K q_{t,k} \sum_{i=1}^K(p_{t,2,i})^2\left(\frac{\ell_{t,i}\, \mathbb{I}\{{A}_{t,(k,i)}=1\}}{\sum_{j=1}^K q_{t,j}\, \mathbb{I}\{{A}_{t,(j,i)}=1\} + \gamma^{(g)}}\right)^2\\
      % &= \sum_{i=1}^K q_{t,i}\left(\sum_{i=1}^K \frac{\ell_{t,i}\, p_{t,2,i}}{\sum_{j:{A}_{t,(i_t,j)}=1} q_{t,j} + \gamma^{(g)}} \cdot \mathbb{I}\{{A}_{t,(i_t,i)}=1\}\right)^2\\
      &\le \sum_{i=1}^K \frac{\ell_{t,i}^2 \, p_{t,2,i}^2}{(\sum_{j=1}^K q_{t,j}\, \mathbb{I}\{{A}_{t,(j,i)}=1\} + \gamma^{(g)})^2} \sum_{k=1}^K q_{t,k}\, \mathbb{I}\!\left\{{A}_{t,(k,i)}=1\right\}\\
      &\le \sum_{i=1}^K \frac{\ell_{t,i}^2 \, p_{t,2,i}^2}{\sum_{j=1}^K q_{t,j}\, \mathbb{I}\!\left\{{A}_{t,(j,i)}=1\right\} + \gamma^{(g)}} \\
      &\le \sum_{i=1}^K \frac{p_{t,2,i}}{\sum_{j:{A}_{t,(j,i)}=1} q_{t,j} + \gamma^{(g)}}
    \end{align*}
    \end{proof}
\end{enumerate}
\end{lemma}

\begin{lemma}[Regret Decomposition]\label{lem:cabsd-regret-decomp}
We may decompose the regret of Algorithm~\ref{alg:CABS-D} as the following:
\begin{align}
\mathbb{E}\left[\mathrm{Reg}_T\right]
&\le \min_{\substack{m \in \{1,2\}}}R_{m} + \mathbb{E}\left[\frac{\ln M}{\eta_T}\right] + \mathbb{E}\left[\frac{1}{2\eta_T}\sum_{t=1}^T \eta_{t}^2 V_t\right]
+ \mathbb{E}[\mathcal{B}_T], \notag
\end{align}
where $R_m$ indicates the expected regret of expert $m$, $\mathcal{B}_T$ is an additive term to account for estimation bias.
\begin{proof}
We begin by defining the weights $w_{t,(m,g)}$ for each meta expert $(m,g)$. We additionally recall:
\begin{align*}
w_{t+1,(m,g)} = w_{t,(m,g)}\, e^{-\eta_t \hat\ell_{t,(m,g)}},\qquad
  W_{t+1} = \sum_{m=1}^{2} \sum_{g=1}^{L} w_{t,(m,g)} e^{-\eta_t \hat\ell_{t,(m,g)}}.
\end{align*}
Thus
\begin{align*}
    \frac{W_{t+1}}{W_t} = \sum_{m=1}^2\sum_{g=1}^L \mu_{t,(m,g)}\, e^{-\eta_t \hat\ell_{t,(m,g)}}
\end{align*}
We take logs on both sides of the inequality and sum over all $t=1,\dots,T$. Simplifying the resulting telescoping series we get
\begin{align}
    \ln\frac{W_{T+1}}{W_1} &= \sum_{t=1}^T \ln\!\left(\sum_{m=1}^2\sum_{g=1}^L \mu_{t,(m,g)}e^{-\eta_t \hat\ell_{t,(m,g)}}\right)\notag\\
    \ln\frac{W_{1}}{W_{T+1}} &= \sum_{t=1}^T - \ln\!\left(\sum_{m=1}^2\sum_{g=1}^L \mu_{t,(m,g)}e^{-\eta_t \hat\ell_{t,(m,g)}}\right)\label{eq:cabsd-telescope-inv}
\end{align}
Applying $e^{-x} \le 1-x+{x^2}/{2}$, we may say
\begin{align*}
    \sum_{m=1}^2\sum_{g=1}^L \mu_{t,(m,g)}e^{-\eta_t \hat\ell_{t,(m,g)}}
    &\le
    1-\eta_{t} \sum_{m=1}^2\sum_{g=1}^L \mu_{t,(m,g)}\hat\ell_{t,(m,g)} + 
    \frac{\eta_{t}^2}{2} \sum_{m=1}^2\sum_{g=1}^L \mu_{t,(m,g)}\hat\ell_{t,(m,g)}^2.
\end{align*}
Applying the inequality $-\ln(1+x) \ge -x$, we get
\begin{align}
    - \ln\!\left(\sum_{m=1}^2\sum_{g=1}^L \mu_{t,(m,g)}e^{-\eta_t \hat\ell_{t,(m,g)}}\right)
    &\ge
    \eta_{t} \sum_{m=1}^2\sum_{g=1}^L \mu_{t,(m,g)}\hat\ell_{t,(m,g)} + 
    \frac{\eta_{t}^2}{2} V_t\label{eq:cabsd-W-lbound}.
\end{align}
We may upper bound the quantity $\ln \frac{W_1}{W_{T+1}}$ using the knowledge that $W_1 = M$ and \\$W_{T+1} \ge \exp\!\left(-\min_{m,g}\sum_{s=1}^T \eta_{s}\hat\ell_{s,(m,g)}\right)$:
\begin{align}
    \ln \frac{W_1}{W_{T+1}}
    &\le \ln M + \min_{\substack{m \in \{1,2\} \\ g \in [L]}} \sum_{s=1}^T \eta_{s}\hat\ell_{s,(m,g)}\label{eq:cabsd-W-ubound}
\end{align}
Plugging Equations~\ref{eq:cabsd-W-lbound} and \ref{eq:cabsd-W-ubound} into Equation~\ref{eq:cabsd-telescope-inv} yields
\begin{align}
    \sum_{t=1}^T\eta_{t} \sum_{m=1}^2\sum_{g=1}^L \mu_{t,(m,g)}\hat\ell_{t,(m,g)} + 
    \sum_{t=1}^T\frac{\eta_{t}^2}{2} V_t
    &\le \ln M + \min_{\substack{m \in \{1,2\} \\ g \in [L]}} \sum_{s=1}^T \eta_{s}\hat\ell_{s,(m,g)}\notag\\
    \sum_{t=1}^T\eta_{t} \sum_{m=1}^2\sum_{g=1}^L \mu_{t,(m,g)}\hat\ell_{t,(m,g)} -
    \min_{\substack{m \in \{1,2\} \\ g \in [L]}} \sum_{s=1}^T \eta_{s}\hat\ell_{s,(m,g)}
    &\le \ln M + \sum_{t=1}^T\frac{\eta_{t}^2}{2} V_t \label{eq:cabsd-init-decomp}
\end{align}
Since we choose $n_t$ to be non-increasing in $t$ (since $\Delta_t$ from Algorithm~\ref{alg:CABS-D} is non-decreasing), we have $\eta_t \ge \eta_T$. Therefore, we can say
\begin{align}
    &\sum_{t=1}^T \eta_t \sum_{m=1}^2\sum_{g=1}^L \mu_{t,(m,g)}\hat\ell_{t,(m,g)}  - \min_{\substack{m \in \{1,2\} \\ g \in [L]}}\sum_{s=1}^T \eta_s \hat\ell_{t,(m,g)}\notag \\
    &\le
    \eta_T\left(\sum_{t=1}^T\sum_{m=1}^2\sum_{g=1}^L \mu_{t,(m,g)}\hat\ell_{t,(m,g)}  - \min_{\substack{m \in \{1,2\} \\ g \in [L]}}\sum_{s=1}^T \hat\ell_{t,(m,g)}\right)\label{eq:cabsd-left-times-eta_T}
\end{align}
Combining Equation~\ref{eq:cabsd-left-times-eta_T} with Equation~\ref{eq:cabsd-init-decomp} and dividing by $\eta_T$ gives us the \textit{master inequality}:
\begin{align}
\sum_{t=1}^T\sum_{m=1}^2\sum_{g=1}^L \mu_{t,(m,g)}\hat\ell_{t,(m,g)}  - \min_{\substack{m \in \{1,2\} \\ g \in [L]}}\sum_{s=1}^T\hat\ell_{t,(m,g)}
&\le \frac{\ln M}{\eta_T} + \frac{1}{2\eta_T}\sum_{t=1}^T \eta_{t}^2 V_t
\end{align}
We may take expectation over the master inequality to obtain
\begin{align}
\mathbb{E}\left[\sum_{t=1}^T\sum_{m=1}^2\sum_{g=1}^L \mu_{t,(m,g)}\hat\ell_{t,(m,g)}\right]  - \mathbb{E}\left[\min_{\substack{m \in \{1,2\} \\ g \in [L]}}\sum_{s=1}^T\hat\ell_{t,(m,g)}\right]
&\le \mathbb{E}\left[\frac{\ln M}{\eta_T}\right] + \mathbb{E}\left[\frac{1}{2\eta_T}\sum_{t=1}^T \eta_{t}^2 V_t\right].\label{eq:cabsd-expectation-master-ineq}
\end{align}
The first term on the left-hand side of the inequality is the \textit{weighted sum of loss estimators}. We relate this weighted sum to the \textit{actual observed loss} by decomposing the expected cumulative observed loss $\mathbb{E}\left[\sum_{t=1}^T\ell_{t,i_t}\right]$ into $\mathbb{E}\left[\sum_{t=1}^T\sum_{m=1}^2\sum_{g=1}^L \mu_{t,(m,g)}\hat\ell_{t,(m,g)}\right]$ and an additive expected bias term $\mathbb{E}[\mathcal{B}_T]$:
\begin{align}
    \mathbb{E}\left[\sum_{t=1}^T\ell_{t,i_t}\right] = \mathbb{E}\left[\sum_{t=1}^T\sum_{m=1}^2\sum_{g=1}^L \mu_{t,(m,g)}\hat\ell_{t,(m,g)}\right] + \mathbb{E}[\mathcal{B}_T] \label{eq:cabsd-mixbias-decomp}
\end{align}
Substituting Equation~\ref{eq:cabsd-mixbias-decomp} into Equation~\ref{eq:cabsd-expectation-master-ineq},
\begin{align}
\mathbb{E}\left[\sum_{t=1}^T\ell_{t,i_t}\right] 
- \mathbb{E}[\mathcal{B}_T]  
- \mathbb{E}\left[\min_{\substack{m \in \{1,2\} \\ g \in [L]}}\sum_{s=1}^T\hat\ell_{t,(m,g)}\right]
&\le \mathbb{E}\left[\frac{\ln M}{\eta_T}\right] + \mathbb{E}\left[\frac{1}{2\eta_T}\sum_{t=1}^T \eta_{t}^2 V_t\right]\notag\\
\mathbb{E}\left[\sum_{t=1}^T\ell_{t,i_t}\right] - L^*
- \mathbb{E}[\mathcal{B}_T]  
- \left(\mathbb{E}\left[\min_{\substack{m \in \{1,2\} \\ g \in [L]}}\sum_{s=1}^T\hat\ell_{t,(m,g)}\right] - L^*\right)
&\le \mathbb{E}\left[\frac{\ln M}{\eta_T}\right] + \mathbb{E}\left[\frac{1}{2\eta_T}\sum_{t=1}^T \eta_{t}^2 V_t\right]\notag\\
\mathbb{E}\left[\sum_{t=1}^T\ell_{t,i_t}\right] - L^*
- \mathbb{E}[\mathcal{B}_T]  
- \left(\mathbb{E}\left[\min_{\substack{m \in \{1,2\}}}\sum_{s=1}^T\sum_{i=1}^K p_{t,m,i}\ell_{t,i}\right] - L^*\right)
&\le \mathbb{E}\left[\frac{\ln M}{\eta_T}\right] + \mathbb{E}\left[\frac{1}{2\eta_T}\sum_{t=1}^T \eta_{t}^2 V_t\right] \label{eq:cabsd-regret-decomp-prop1}\\
\mathbb{E}\left[\mathrm{Reg}_T\right]
- \mathbb{E}[\mathcal{B}_T]  
- \min_{\substack{m \in \{1,2\}}}R_{m}
&\le \mathbb{E}\left[\frac{\ln M}{\eta_T}\right] + \mathbb{E}\left[\frac{1}{2\eta_T}\sum_{t=1}^T \eta_{t}^2 V_t\right].\notag
\end{align}
Where $L^*$ is the loss from following the optimal policy. Equation~\ref{eq:cabsd-regret-decomp-prop1} is due to applying the \textit{optimistic bias} property from Lemma~\ref{lem:ix-properties}. Rearranging terms gives us the final regret decomposition
\begin{align}
\mathbb{E}\left[\mathrm{Reg}_T\right]
&\le \min_{\substack{m \in \{1,2\}}}R_{m} + \mathbb{E}\left[\frac{\ln M}{\eta_T}\right] + \mathbb{E}\left[\frac{1}{2\eta_T}\sum_{t=1}^T \eta_{t}^2 V_t\right]
+ \mathbb{E}[\mathcal{B}_T]. \notag
\end{align}
\end{proof}
\end{lemma}

\begin{lemma}[Bounding Second-Order Terms]\label{lem:cabsd-second-order}
Suppose that the feedback graph $G_t$ associated with the graph feedback expert ($m=1$) is strongly observable, and has independence number no greater than $\alpha$. If we assume the geometric grid of $\gamma$-values is of size $L=\ln T$ and that $q_{t,i} \ge \epsilon$, we may bound $V_t$ independently for each expert such that 
\begin{align}
    \mathbb{E}\left[V_t\right] \le \sum_{m=1}^2
    \sum_{g=1}^L
\sum_{i=1}^K \left(\frac{q_{t,i}}{Q_{t,m,i} + \gamma^{(g)}}\right) \le
\begin{cases}
    2K \ln T, & \mathrm{if}\; m=1\; \mathrm{dominates} \\
    2\alpha \ln\!\left(\frac{4T}{\alpha \epsilon} \right)\ln\!\left(T \right),  & \mathrm{if}\; m=2\; \mathrm{dominates}
\end{cases},
\end{align}
where 
\begin{align}
Q_{t,m,i} = 
\begin{cases}
q_{t,i},  &\textrm{if} \; m=1\\
\sum_{j=1}^K q_{t,j} \cdot \mathbb{I}\{A_{t,(j,i)}=1\}, & \textrm{if} \; m=2
\end{cases}\notag.
\end{align}
\begin{proof}
We begin by proving the statement for when $m=1$ (bandit feedback expert) dominates. Specifically, in this case we assume the regret of expert $m=1$ consistently out-performs that of expert $m=2$, thus the $\mathbb{E}[V_t]$ term associated with $m=1$ dominates. We begin with the definition of $V_t$ with fixed $m=1$
\begin{align}
\mathbb{E}\left[V_t\right] &= \sum_{m=1}^2\sum_{g=1}^L \mu_{t,(1,g)} \mathbb{E}\left[\hat\ell_{t,(1,g)}^2\right]\notag\\
&\le  \sum_{m=1}^2\sum_{g=1}^L \mu_{t,(1,g)} \sum_{i=1}^K \frac{p_{t,1,i}}{q_{t,i}+\gamma^{(g)}} \notag\\
&\le  \sum_{m=1}^2\sum_{g=1}^L \sum_{i=1}^K \frac{q_{t,i}}{q_{t,i}+\gamma^{(g)}}\label{cabsd-m1-p2q}\\
&\le  \sum_{m=1}^2\sum_{g=1}^L \sum_{i=1}^K 1\notag\\
&\le  2K\ln T\notag,
\end{align}
where Equation~\ref{cabsd-m1-p2q} is a result of applying the definition of $q_{t,i}$. We continue by proving the statement for when $m=2$ (graph feedback dominates). Using a similar line of reasoning to the case when $m=1$ dominates, we begin with the definition of $V_t$ with fixed $m=2$
\begin{align}
\mathbb{E}\left[V_t\right] &= \sum_{m=1}^2\sum_{g=1}^L \mu_{t,(2,g)} \mathbb{E}\left[\hat\ell_{t,(2,g)}^2\right]\notag\\
&\le  \sum_{m=1}^2\sum_{g=1}^L \mu_{t,(2,g)} \sum_{i=1}^K \frac{p_{t,2,i}}{\sum_{j:{A}_{t,(j,i)}=1} q_{t,j}+\gamma^{(g)}} \notag\\
&\le  \sum_{m=1}^2\sum_{g=1}^L \sum_{i=1}^K \frac{q_{t,i}}{\sum_{j:{A}_{t,(j,i)}=1} q_{t,j}+\gamma^{(g)}}\label{cabsd-m2-p2q}\\
&\le  \sum_{m=1}^2\sum_{g=1}^L \sum_{i=1}^K \frac{q_{t,i}}{\sum_{j:{A}_{t,(j,i)}=1} q_{t,j}}\notag\\
&\le  \sum_{m=1}^2\sum_{g=1}^L \sum_{i=1}^K \alpha \ln\!\left(\frac{4T}{\alpha \epsilon}\right)\label{cabsd-m2-alpha}\\
&\le  2\alpha \ln\!\left(\frac{4T}{\alpha \epsilon} \right) \ln\!\left(T \right)\notag,
\end{align}
where Equation~\ref{cabsd-m2-p2q} is a result of applying the definition of $q_{t,i}$, and Equation~\ref{cabsd-m2-alpha} is due to Lemma~\ref{lem:alon-graph-alpha}.
\end{proof}
\end{lemma}
\paragraph{Proof of Theorem~\ref{theorem:CABS-D}.}
We now use the preceding lemmas to prove that if there are $M=2$ experts: one bandit-feedback expert and one graph-feedback expert (tuned for known independence number $\alpha$), then Algorithm~\ref{alg:CABS-D} satisfies
\begin{align*}
&\mathbb{E}\left[\mathrm{Reg_{Sq}}(T)\right] \\
&\le O\!\left(\min\left\{R_1 + \ln\!\left(KT\ln T\right)\sqrt{KT\ln M \ln T}, \right.\right.
\left.\left. R_2 + \ln\!\left( \alpha T \ln(T/\alpha\epsilon)\ln T\right)\sqrt{\alpha T \ln(T/\alpha\epsilon) \ln M \ln T}\right\}\right)
\end{align*}

Using $R_1 = \widetilde O(\sqrt{KT \, \mathrm{Reg_{Sq}}(T)})$ from Theorem~\ref{theorem:squarecb} and $R_2 = \widetilde O(\sqrt{\alpha T \, \mathrm{Reg_{Sq}}(T)} + \sqrt{\alpha KT}\,\varepsilon_n)$ from Theorem~\ref{theorem:CABS-C}, we obtain the regret bound
\[
\widetilde O\!\left( \min \{\sqrt{KT \, \mathrm{Reg_{Sq}}(T)}, \sqrt{\alpha T \, \mathrm{Reg_{Sq}}(T)} + \sqrt{\alpha KT}\,\varepsilon_n\}\right).
\]
\begin{proof}
We restate the decomposition derived in Lemma~\ref{lem:cabsd-regret-decomp},
\begin{align}
\mathbb{E}\left[\mathrm{Reg}_T\right]
&\le \min_{\substack{m \in \{1,2\}}}R_{m} + \mathbb{E}\left[\frac{\ln M}{\eta_T}\right] + \mathbb{E}\left[\frac{1}{2\eta_T}\sum_{t=1}^T \eta_{t}^2 V_t\right]
+ \mathbb{E}[\mathcal{B}_T]. \label{eq:cabsd-regret-decomp}
\end{align}
We begin by bounding $\mathbb{E}[\mathcal{B}_T]$. Rearranging Equation~\ref{eq:cabsd-mixbias-decomp} we get
\begin{align}
\mathbb{E}[\mathcal{B}_T]
&\;=\; \mathbb{E}\left[\sum_{t=1}^T\ell_{t,i_t}\right] - \mathbb{E}\left[\sum_{t=1}^T\sum_{m=1}^2\sum_{g=1}^L \mu_{t,(m,g)}\hat\ell_{t,(m,g)}\right]\notag\\
&\;=\; \sum_{t=1}^T\sum_{i=1}^K q_{t,i}\ell_{t,i} - \sum_{t=1}^T\sum_{m=1}^2\sum_{g=1}^L \mu_{t,(m,g)}\sum_{i=1}^K p_{t,m,i}\ell_{t,i} \frac{Q_{t,m,i}}{Q_{t,m,i} + \gamma^{(g)}} \label{eq:cabsd-expanded-bias}
\end{align}
We may further expand Equation~\ref{eq:cabsd-expanded-bias}:
\begin{align}
&\sum_{t=1}^T\sum_{i=1}^K q_{t,i}\ell_{t,i} - \sum_{t=1}^T\sum_{m=1}^2\sum_{g=1}^L \mu_{t,(m,g)}\sum_{i=1}^K p_{t,m,i}\ell_{t,i} \frac{Q_{t,m,i}}{Q_{t,m,i} + \gamma^{(g)}} \notag\\
&\;=\; \sum_{t=1}^T\sum_{m=1}^2\sum_{g=1}^L \mu_{t,(m,g)}\sum_{i=1}^K p_{t,m,i}\ell_{t,i}  - \sum_{t=1}^T\sum_{m=1}^2\sum_{g=1}^L \mu_{t,(m,g)}\sum_{i=1}^K p_{t,m,i}\ell_{t,i} \frac{Q_{t,m,i}}{Q_{t,m,i} + \gamma^{(g)}} \notag\\
&\;=\; \sum_{t=1}^T\sum_{m=1}^2\sum_{g=1}^L \mu_{t,(m,g)}\sum_{i=1}^K p_{t,m,i}\ell_{t,i} \left(1-\frac{Q_{t,m,i}}{Q_{t,m,i} + \gamma^{(g)}}\right) \notag\\
&\;=\; \sum_{t=1}^T\sum_{m=1}^2\sum_{g=1}^L \mu_{t,(m,g)}\sum_{i=1}^K p_{t,m,i}\ell_{t,i} \left(\frac{\gamma^{(g)}}{Q_{t,m,i} + \gamma^{(g)}}\right) \notag\\
&\;\le\; \sum_{t=1}^T\sum_{m=1}^2\sum_{g=1}^L \mu_{t,(m,g)}\, \gamma^{(g)}
\sum_{i=1}^K \left(\frac{p_{t,m,i}}{Q_{t,m,i} + \gamma^{(g)}}\right)\notag\\
&\;\le\; \sum_{t=1}^T\sum_{m=1}^2\sum_{g=1}^L \, \gamma^{(g)}
\sum_{i=1}^K \left(\frac{q_{t,i}}{Q_{t,m,i} + \gamma^{(g)}}\right) \label{eq:cabsd-bias-decomp}
\end{align}
Next, we bound $\mathbb{E}\left[\frac{\ln M}{\eta_T}\right] + \mathbb{E}\left[\frac{1}{2\eta_T}\sum_{t=1}^T \eta_{t}^2 V_t\right]$.
\begin{align}
    \frac{\ln M}{\eta_T} + \frac{1}{2\eta_T}\sum_{t=1}^T \eta_{t}^2 V_t
    &= \frac{\ln M}{\eta_T} + \frac{1}{2\eta_T}\sum_{t=1}^T \frac{\ln M}{1+\Delta_{t-1}} V_t\notag\\
    &= \frac{\ln M}{\eta_T} \left(1 + \frac{1}{2}\sum_{t=1}^T \frac{V_t}{1+\Delta_{t-1}}\right)\notag\\
    &= \sqrt{\ln M}\sqrt{1+\Delta_{T-1}} \left(1 + \frac{1}{2}\sum_{t=1}^T \frac{V_t}{1+\Delta_{t-1}}\right)\notag\\
    &\le \sqrt{\ln M}\sqrt{1+\Delta_{T-1}} \left(1 + \sum_{t=1}^T \frac{V_t}{\Delta_{t}}\right)\label{eq:cabsd-Delta_t-bound}\\
    &\le \sqrt{\ln M}\sqrt{1+\Delta_{T-1}} \left(1 + \sum_{t=1}^T \frac{V_t}{\sum_{s=1}^t V_s}\right)\notag\\
    &\le \sqrt{\ln M}\sqrt{1+\Delta_{T-1}} \left(2 + \ln\!\left(\sum_{t=1}^T V_t\right) \right)\label{eq:cabsd-ln-bound}\\
    &\le \sqrt{\ln M}\sqrt{1+\Delta_{T-1}} \left(2 + \ln\!\left(\Delta_{T}\right) \right)\notag\\
    &\le \ln(\Delta_T)\sqrt{\ln M} + 2\sqrt{\ln M} + \ln(\Delta_T)\sqrt{\Delta_{T-1}\ln M} + 2\sqrt{\Delta_{T-1}\ln M}\notag\\
    &\le \ln(\Delta_T)\sqrt{\Delta_{T-1}\ln M} + c_1 \label{eq:cabsd-nonbias-bound},
\end{align}
where Equation~\ref{eq:cabsd-Delta_t-bound} is due to $\Delta_t = V_t + \Delta_{t-1} \le 1 + \Delta_{t-1}$, and Equation~\ref{eq:cabsd-ln-bound} is due to the following fact: $\sum_{t=1}^T \frac{x_t}{\sum_{s=1}^t x_s} \le 1 + \ln\!\left(\sum_{t=1}^T x_t\right)$. For cleaner notation, in the Equation~\ref{eq:cabsd-nonbias-bound} we absorb all dominated terms into the constant $c_1$. We continue by plugging Equation~\ref{eq:cabsd-nonbias-bound} and Equation~\ref{eq:cabsd-bias-decomp} into Equation~\ref{eq:cabsd-regret-decomp} to get
\begin{align}
&\mathbb{E}\left[\mathrm{Reg}_T\right]\notag\\
&\le \min_{\substack{m \in \{1,2\}}}R_{m} + 
\mathbb{E}[\ln(\Delta_T)]\sqrt{\mathbb{E}[\Delta_{T-1}]\ln M}
+ \sum_{t=1}^T\sum_{m=1}^2\sum_{g=1}^L \, \gamma^{(g)}
\sum_{i=1}^K \left(\frac{q_{t,i}}{Q_{t,m,i} + \gamma^{(g)}}\right) + c_1\notag\\
&\le \min_{\substack{m \in \{1,2\}}}R_{m} + \underbrace{\ln\!\left(\sum_{t=1}^T V_t\right)\sqrt{\left(\sum_{t=1}^{T-1} V_t\right)\ln M}
+ \sum_{t=1}^T\sum_{m=1}^2\sum_{g=1}^L\, \gamma^{(g)}
\sum_{i=1}^K \left(\frac{q_{t,i}}{Q_{t,m,i} + \gamma^{(g)}}\right) + c_1}_{\text{CABS-D overhead}}\label{eq:cabsd-bias-decomp-plugged}
\end{align}
Since the regret contributions of both experts cannot be simultaneously large, according to the first term if $R_1 \le R_2$ then the analysis ensures that terms associated with $R_2$ do not accumulate (the same applies if $R_1 \ge R_2$). This allows us to divide the analysis of the \textit{CABS-D overhead} into two cases: 1) where $m=1$ (bandit feedback) is the best expert, and 2) where $m=2$ (graph feedback) is the best expert.  We begin by examining the first case:\\\\
When $m=1$:
\begin{align}
&\ln\!\left(\sum_{t=1}^T V_t\right)\sqrt{\left(\sum_{t=1}^{T-1} V_t\right)\ln M}
+ \sum_{t=1}^T\sum_{m=1}^2\sum_{g=1}^L\, \gamma^{(g)}
\sum_{i=1}^K \left(\frac{q_{t,i}}{Q_{t,1,i} + \gamma^{(g)}}\right) + c_1\notag\\
&\le \ln\!\left(\sum_{t=1}^T 2K\ln T\right)\sqrt{\left(\sum_{t=1}^{T-1} 2K\ln T\right)\ln M}
+ \sum_{t=1}^T2K\sum_{g=1}^L\, \gamma^{(g)} + c_1\label{eq:cabsd-m1-v2k}\\
&\le \ln\!\left(2KT\ln T\right)\sqrt{2K(T-1)\ln M \ln T}
+ 2KT\sum_{g=1}^L\, \gamma^{(g)} + c_1,\notag
\end{align}
where Equation~\ref{eq:cabsd-m1-v2k} is due to Lemma~\ref{lem:cabsd-second-order}. Recall that the algorithm aggregates a set of experts parameterized by a geometric grid $\Gamma$. By the properties of CABS-D (Lemma~\ref{lem:cabsd-regret-decomp}), the total regret is bounded by the regret of the best fixed expert in the grid, plus an additive expert-dependent overhead. Therefore, we bound the regret by choosing the specific $\gamma^{(g)} \in \Gamma$ that minimizes the bound for the specific feedback structure. Because $\Gamma$ is a dyadic grid of size $L = \ln T$, for any optimal tuning parameter $\gamma^*$, there exists a grid point $\gamma^{(g)}$ such that $\gamma^{(g)} \approx \gamma^*$ up to a small constant factor. When $\gamma^{(g)} = \frac{1}{\sqrt{KT}}$, we get
\begin{align}
&\ln\!\left(\sum_{t=1}^T V_t\right)\sqrt{\left(\sum_{t=1}^{T-1} V_t\right)\ln M}
+ \sum_{t=1}^T\sum_{m=1}^2\sum_{g=1}^L\, \gamma^{(g)}
\sum_{i=1}^K \left(\frac{q_{t,i}}{Q_{t,1,i} + \gamma^{(g)}}\right) + c_1\notag\\
&\le \ln\!\left(2KT\ln T\right)\sqrt{2K(T-1)\ln M \ln T}
+ 2\ln T \sqrt{KT} + c_1\notag\\
&\le O\!\left(\ln\!\left(KT\ln T\right)\sqrt{KT\ln M \ln T}\right), \label{eq:cabsd-m1-overhead}
\end{align}
Similarly, when $m=2$:
\begin{align}
&\ln\!\left(\sum_{t=1}^T V_t\right)\sqrt{\left(\sum_{t=1}^{T-1} V_t\right)\ln M}
+ \sum_{t=1}^T\sum_{m=1}^2\sum_{g=1}^L\, \gamma^{(g)}
\sum_{i=1}^K \left(\frac{q_{t,i}}{Q_{t,2,i} + \gamma^{(g)}}\right) + c_1\notag\\
&\le \ln\!\left(\sum_{t=1}^T 2\alpha \ln\!\left(\frac{4T}{\alpha \epsilon} \right) \ln T\right)\sqrt{\left(\sum_{t=1}^{T-1} 2\alpha \ln\!\left(\frac{4T}{\alpha \epsilon} \right)\ln T\right)\ln M}
+ \sum_{t=1}^T2\alpha \ln\!\left(\frac{4T}{\alpha \epsilon} \right)\sum_{g=1}^L\, \gamma^{(g)} + c_1\notag\\
&\le \ln\!\left( 2\alpha T \ln\!\left(\frac{4T}{\alpha \epsilon} \right)\ln T\right)\sqrt{\left(2\alpha (T-1) \ln\!\left(\frac{4T}{\alpha \epsilon} \right)\ln T \ln M\right)}
+ 2\alpha T \ln\!\left(\frac{4T}{\alpha \epsilon} \right)\sum_{g=1}^L\, \gamma^{(g)} + c_1\notag.
\end{align}
By setting  $\gamma^{(g)} = \frac{1}{\sqrt{\alpha T}}$, we get
\begin{align}
&\ln\!\left(\sum_{t=1}^T V_t\right)\sqrt{\left(\sum_{t=1}^{T-1} V_t\right)\ln M}
+ \sum_{t=1}^T\sum_{m=1}^2\sum_{g=1}^L\, \gamma^{(g)}
\sum_{i=1}^K \left(\frac{q_{t,i}}{Q_{t,2,i} + \gamma^{(g)}}\right) + c_1\notag\\
&\le \ln\!\left( 2\alpha T \ln\!\left(\frac{4T}{\alpha \epsilon} \right)\ln T\right)\sqrt{\left(2\alpha (T-1) \ln\!\left(\frac{4T}{\alpha \epsilon} \right)\ln M \ln T\right)}
+ 2\sqrt{\alpha T} \ln\!\left(\frac{4T}{\alpha \epsilon} \right) \ln T + c_1\notag\\
&\le O\!\left(\ln\!\left( \alpha T \ln(T/\alpha\epsilon)\ln T\right)\sqrt{\alpha T \ln(T/\alpha\epsilon) \ln M \ln T}\right). \label{eq:cabsd-m2-overhead}
\end{align}
Combining the two cases (from Equations~\ref{eq:cabsd-m1-overhead} and \ref{eq:cabsd-m2-overhead}) and plugging them into Equation~\ref{eq:cabsd-bias-decomp-plugged} yields

\begin{align*}
&\mathbb{E}\left[\mathrm{Reg_{Sq}}(T)\right] \\
&\le O\!\left(\min\left\{R_1 + \ln\!\left(KT\ln T\right)\sqrt{KT\ln M \ln T}, \right.\right.
\\&\quad\left.\left. R_2 + \ln\!\left( \alpha T \ln(T/\alpha\epsilon)\ln T\right)\sqrt{\alpha T \ln(T/\alpha\epsilon) \ln M \ln T}\right\}\right)
\end{align*}
If we use Theorem~\ref{theorem:squarecb} to obtain $R_1$ (the regret of the bandit feedback expert) and we use Theorem~\ref{theorem:CABS-C} to obtain $R_2$ (the regret of the graph feedback expert), we may plug them in to obtain (after absorbing polylogs)
\begin{align}
\mathbb{E}[\mathrm{Reg}_T] &\le \widetilde O\!\left(\min\left\{\sqrt{KT \, \mathrm{Reg_{Sq}}(T)} + \sqrt{KT}, \sqrt{\alpha T \, \mathrm{Reg_{Sq}}(T)} + \sqrt{\alpha KT}\,\varepsilon_n + \sqrt{\alpha T}\right\}\right)\notag\\
&\le \widetilde O\!\left( \min \{\sqrt{KT \, \mathrm{Reg_{Sq}}(T)}, \sqrt{\alpha T \, \mathrm{Reg_{Sq}}(T)} + \sqrt{\alpha KT}\,\varepsilon_n\}\right).
\end{align}
\end{proof}

\paragraph{Deriving Corollary~\ref{cor:CABS-D_linear_mfeedback}.} \label{app:CABS-D_linear_mfeedback}
Corollary~\ref{cor:CABS-D_linear_mfeedback} follows by substituting $\alpha=\frac{K}{m+1}$ from Theorem~\ref{thm:m-trans-feedback} into the general regret bound of Theorem~\ref{theorem:CABS-D}, and applying the standard square-loss regret bound for linear regression oracles~\cite{foster2020beyond}, $\mathrm{Reg}_{\mathrm{Sq}}(T) = O(d\log \left(\frac{T}{d}\right))$.

\section{Discussions \& Limitations}\label{sec:discussion}
We discuss the trade-offs between coupled and decoupled use of surrogate rewards, the assumptions and limitations of our surrogate-feedback model, the computational overhead of the proposed algorithms, and practical extensions of the utility formulation.

\paragraph{Comparison of coupled and decoupled approaches.} CABS-C and CABS-D differ primarily in how tightly they integrate surrogate feedback with true bandit feedback. CABS-C follows a coupled strategy: after debiasing surrogate rewards and applying importance weighting, it feeds both true and surrogate observations into a common regression oracle. This coupling can improve sample efficiency when the surrogate rewards are informative, since each round provides feedback for multiple correlated arms. However, because the two feedback sources update the same predictor, persistent surrogate bias or high surrogate noise can still affect the learned model, even when importance weighting partially controls the contribution of each observation. CABS-D addresses this limitation by maintaining separate learners for the standard bandit-feedback signal and the surrogate-aware feedback signal, and then adaptively mixing their induced policies. As a result, when surrogate rewards are useful, CABS-D can exploit them for faster learning; when they are uninformative or too noisy, it can fall back on the standard contextual bandit expert and retain the corresponding regret guarantee $\widetilde O(\sqrt{KT\mathrm{Reg}_{\mathrm{Sq}}(T)})$.

\paragraph{Why assume a fixed arm-specific bias?} Assumption~\ref{asm:ml-side-obs-model} models surrogate error as an arm-specific, context-independent bias $b_i$ plus conditionally zero-mean bounded noise. While offline predictors may exhibit domain-specific behavior, this assumption is sufficient to capture context-varying errors while preserving online learnability. First, if the bias were an arbitrary context-dependent function $b_i(\boldsymbol{x}_t)$, the observation model would become $s_{t,i}=f^\star(\boldsymbol{x}_t,i)+b_i(\boldsymbol{x}_t)+\xi_{t,i}$. In this case, $f^\star(\boldsymbol{x}_t,i)$ and $b_i(\boldsymbol{x}_t)$ become non-identifiable and collapse into a single unobservable function $g(\boldsymbol{x}_t,i)$. Under arbitrary context-dependent bias, the offline surrogate observations $s_{t,i}$ provide no identifiable signal about $f^\star(\boldsymbol{x}_t,i)$, rendering de-biasing impossible and forcing any online algorithm that uses these surrogates to suffer linear $\mathcal{O}(T)$ regret. Second, the context dependence of the reward is already captured by $f^\star(\boldsymbol{x}_t,i)$. The parameter $b_i$ is designed to isolate a global, context-free calibration offset, such as a mean shift, while context-specific variation is captured by $f^\star(\boldsymbol{x}_t,i)$ and the bounded residual noise $\xi_{t,i}$.
A natural extension is to use category-specific offsets $b_{i,c}$ when query categories, such as coding or mathematics, are observed or predictable. CABS-C could then de-bias separately within each category, although extending the guarantee would require estimating these additional offsets.

\paragraph{Surrogate-error assumptions.}Our theoretical guarantees assume a fixed arm-specific surrogate bias and conditionally zero-mean bounded noise (Assumption~\ref{asm:ml-side-obs-model}). To empirically assess this model, we tested whether query embeddings help predict the trained reward oracle's ($\roracle$) errors beyond a constant per-arm average. We evaluated ridge regression on held-out queries across varying $\rho$, averaging MSE within each cost regime. Relative to the constant-bias baseline, ridge reduced residual-prediction MSE by 10.4\%, 13.3\%, and 12.9\% on RouterBench; 5.2\%, 7.1\%, and 13.2\% on Open LLM LB v2; and 10.3\%, 11.9\%, and 11.5\% on SPROUT, in the low-, medium-, and high-cost regimes, respectively. These results indicate predictable query-dependent structure in observed-reward errors beyond a fixed arm-specific bias, suggesting departures from the idealized noise model. The remaining error cannot be interpreted as purely random noise, and extending the guarantees to context-dependent bias remains future work.

\paragraph{Computational complexity.} CABS-C has complexity \(O(K^3 + \text{\sffamily RS})\), while CABS-D has \(O(K^3 + K\ln T + \text{\sffamily RS})\). The \(K^3\) term comes from the optimization in line 8 of Algorithm~\ref{alg:CABS-C}, and the extra \(K\ln T\) term in CABS-D comes from the dyadic grid in Algorithm~\ref{alg:CABS-D}. For a linear oracle, \(O(\text{\sffamily RS})=O(Kd+d^2)\). In contrast, standard linear bandit methods such as LinUCB/LinTS are typically dominated by matrix operations cubic in the feature dimension \(d\), $O(Kd^3)$. Therefore, CABS-C and CABS-D may be preferable in settings with large \(d\) or more general nonlinear function classes, whereas their stronger dependence on \(K\) can make them less attractive when the number of actions is very large.

\paragraph{Surrogate reward prediction overheads.} Beyond the learner update, the additional cost of producing surrogate rewards is small in our implementation. The reward oracle is a 2-layer MLP that reuses the context embeddings already computed for the bandit, together with the chosen arm embedding and observed reward. In batch-size-1 measurements, embedding generation takes 12.32 ms on average, while surrogate prediction itself takes only 0.22 ms on average. Thus, the incremental overhead from surrogate prediction beyond the encoder is negligible.

\paragraph{Extending utility.} Our current utility objective can already capture latency when latency is embedded in the cost term, for example through a deployment-specific cost that combines monetary price, compute, and response time. More explicitly, one can extend the utility as $ r_{t,a} = s(\mathbf{x}_t,a) - \rho\, c(\mathbf{x}_t,a) - \rho_{1}\, l(\mathbf{x}_t,a),$
where $s(\mathbf{x}_t,a)$ is the accuracy score, $c(\mathbf{x}_t,a)$ is the monetary cost, and $l(\mathbf{x}_t,a)$ is the latency of arm $a$ for context $\mathbf{x}_t$. Another promising direction is to combine our framework with delayed-feedback bandits~\cite{zhang2025contextuallinearbanditsdelay}, which could model latency more directly.

%%%%%%%%%%%%%%%%%%%%%%%%%%%%%%%%%%%%%%%%%%%%%%%%%%%%%%%%%%%%

\newpage
\section*{NeurIPS Paper Checklist}

\begin{enumerate}

\item {\bf Claims}
    \item[] Question: Do the main claims made in the abstract and introduction accurately reflect the paper's contributions and scope?
    \item[] Answer: \answerYes{} % Replace by \answerYes{}, \answerNo{}, or \answerNA{}.
    \item[] Justification: In the abstract and introduction, we list our contributions

\item {\bf Limitations}
    \item[] Question: Does the paper discuss the limitations of the work performed by the authors?
    \item[] Answer: \answerYes{} % Replace by \answerYes{}, \answerNo{}, or \answerNA{}.
    \item[] Justification: We have listed the assumptions/limitations in the main section as well as have a discussion section in appendix for limitations/future work.

\item {\bf Theory assumptions and proofs}
    \item[] Question: For each theoretical result, does the paper provide the full set of assumptions and a complete (and correct) proof?
    \item[] Answer: \answerYes{} % Replace by \answerYes{}, \answerNo{}, or \answerNA{}.
    \item[] Justification: We state assumptions, lemmas and theorems in the main section, with the proof presented in appendix.

    \item {\bf Experimental result reproducibility}
    \item[] Question: Does the paper fully disclose all the information needed to reproduce the main experimental results of the paper to the extent that it affects the main claims and/or conclusions of the paper (regardless of whether the code and data are provided or not)?
    \item[] Answer: \answerYes{} % Replace by \answerYes{}, \answerNo{}, or \answerNA{}.
    \item[] Justification: We present datasets, models and implementation details in the main experiments section as well as in appendix.  

\item {\bf Open access to data and code}
    \item[] Question: Does the paper provide open access to the data and code, with sufficient instructions to faithfully reproduce the main experimental results, as described in supplemental material?
    \item[] Answer: \answerNo{} % Replace by \answerYes{}, \answerNo{}, or \answerNA{}.
    \item[] Justification: We use open-source datasets and models, but do not attach the code.

\item {\bf Experimental setting/details}
    \item[] Question: Does the paper specify all the training and test details (e.g., data splits, hyperparameters, how they were chosen, type of optimizer) necessary to understand the results?
    \item[] Answer: \answerYes{} % Replace by \answerYes{}, \answerNo{}, or \answerNA{}.
    \item[] Justification: We present datasets, models and implementation details in the main experiments section as well as in appendix.

\item {\bf Experiment statistical significance}
    \item[] Question: Does the paper report error bars suitably and correctly defined or other appropriate information about the statistical significance of the experiments?
    \item[] Answer: \answerYes{} % Replace by \answerYes{}, \answerNo{}, or \answerNA{}.
    \item[] Justification:  The experimental results are presented with 95\% confidence intervals. 

\item {\bf Experiments compute resources}
    \item[] Question: For each experiment, does the paper provide sufficient information on the computer resources (type of compute workers, memory, time of execution) needed to reproduce the experiments?
    \item[] Answer: \answerYes{} % Replace by \answerYes{}, \answerNo{}, or \answerNA{}.
    \item[] Justification: We report the details in appendix. 
    
\item {\bf Code of ethics}
    \item[] Question: Does the research conducted in the paper conform, in every respect, with the NeurIPS Code of Ethics \url{https://neurips.cc/public/EthicsGuidelines}?
    \item[] Answer: \answerYes{} % Replace by \answerYes{}, \answerNo{}, or \answerNA{}.
    \item[] Justification: We follow the ethics guidelines. 

\item {\bf Broader impacts}
    \item[] Question: Does the paper discuss both potential positive societal impacts and negative societal impacts of the work performed?
    \item[] Answer: \answerNA{} % Replace by \answerYes{}, \answerNo{}, or \answerNA{}.
    \item[] Justification: Our paper focuses on LLM Routing techniques. 
    
\item {\bf Safeguards}
    \item[] Question: Does the paper describe safeguards that have been put in place for responsible release of data or models that have a high risk for misuse (e.g., pre-trained language models, image generators, or scraped datasets)?
    \item[] Answer: \answerNA{} % Replace by \answerYes{}, \answerNo{}, or \answerNA{}.
    \item[] Justification: Our paper poses no such risks.

\item {\bf Licenses for existing assets}
    \item[] Question: Are the creators or original owners of assets (e.g., code, data, models), used in the paper, properly credited and are the license and terms of use explicitly mentioned and properly respected?
    \item[] Answer: \answerYes{} % Replace by \answerYes{}, \answerNo{}, or \answerNA{}.
    \item[] Justification: We cite datasets and baseline models in our paper.  

\item {\bf New assets}
    \item[] Question: Are new assets introduced in the paper well documented and is the documentation provided alongside the assets?
    \item[] Answer: \answerNA{} % Replace by \answerYes{}, \answerNo{}, or \answerNA{}.
    \item[] Justification: The paper does not release new assets.

\item {\bf Crowdsourcing and research with human subjects}
    \item[] Question: For crowdsourcing experiments and research with human subjects, does the paper include the full text of instructions given to participants and screenshots, if applicable, as well as details about compensation (if any)? 
    \item[] Answer: \answerNA{} % Replace by \answerYes{}, \answerNo{}, or \answerNA{}.
    \item[] Justification:  Our paper does not involve crowdsourcing nor research with human subjects.

\item {\bf Institutional review board (IRB) approvals or equivalent for research with human subjects}
    \item[] Question: Does the paper describe potential risks incurred by study participants, whether such risks were disclosed to the subjects, and whether Institutional Review Board (IRB) approvals (or an equivalent approval/review based on the requirements of your country or institution) were obtained?
    \item[] Answer: \answerNA{} % Replace by \answerYes{}, \answerNo{}, or \answerNA{}.
    \item[] Justification: Our paper does not involve crowdsourcing nor research with human subjects.

\item {\bf Declaration of LLM usage}
    \item[] Question: Does the paper describe the usage of LLMs if it is an important, original, or non-standard component of the core methods in this research? Note that if the LLM is used only for writing, editing, or formatting purposes and does \emph{not} impact the core methodology, scientific rigor, or originality of the research, declaration is not required.
    %this research? 
    \item[] Answer: \answerNA{} % Replace by \answerYes{}, \answerNo{}, or \answerNA{}.
    \item[] Justification: Our paper's core method development does not involve LLMs as any important, original, or non-standard components.

\end{enumerate}

\end{document}